\documentclass[twoside,11pt]{article}
\usepackage{jmlr2e}

\usepackage[numbers]{natbib}
\newcommand{\newadd}[1]{{\color{black}#1}}

\usepackage[utf8]{inputenc} 
\usepackage[T1]{fontenc}
\usepackage{url}
\usepackage{hyperref}       
\usepackage{booktabs}       
\usepackage{amsfonts}       
\usepackage{nicefrac}       
\usepackage{microtype}      
\usepackage{ifthen}
\usepackage[linesnumbered,ruled,algo2e]{algorithm2e}
\usepackage{amsthm, amssymb, amsmath}
\usepackage{graphicx}
\usepackage{cleveref}
\usepackage[colorinlistoftodos,prependcaption]{todonotes}

\usepackage{epstopdf}
\usepackage{algorithm}
 \usepackage{algorithmic}
\usepackage{wrapfig}
\usepackage{subcaption}
\usepackage{bbm}

\newtheorem{lem}{Lemma}
\newtheorem{thm}{Theorem}
\newtheorem{defn}{Definition}
\newtheorem{coro}{Corollary}

\newtheorem{prop}{Proposition}

\newtheorem{rem}{Remark}
\newtheorem{fact}{Fact}

\newcommand{\E}[1]{\mathbb{E}\left[{#1}\right]}

\DeclareMathOperator*{\argmax}{arg\,max}

\graphicspath{{Figures/}}

\crefname{equation}{}{}
\Crefname{equation}{}{}
\crefname{thm}{theorem}{theorems}
\Crefname{thm}{Theorem}{Theorems}
\crefname{clm}{claim}{claims}
\Crefname{clm}{Claim}{Claims}
\Crefname{coro}{Corollary}{Corollaries}
\Crefname{lem}{Lemma}{Lemmas}
\Crefname{sec}{Section}{Sections}
\crefname{app}{appendix}{appendices}
\Crefname{app}{Appendix}{Appendices}
\crefname{prop}{proposition}{propositions}
\Crefname{prop}{Proposition}{Propositions}
\Crefname{propty}{Property}{Properties}
\crefname{figure}{figure}{figures}
\Crefname{figure}{Figure}{Figures}
\crefname{fig}{figure}{figures}
\Crefname{fig}{Figure}{Figures}
\crefname{defn}{definition}{definitions}
\Crefname{defn}{Definition}{Definitions}
\crefname{fact}{fact}{facts}
\Crefname{fact}{Fact}{Facts}
\crefname{appendix}{appendix}{appendices}
\Crefname{appendix}{Appendix}{Appendices}
\crefname{algo}{algorithm}{algorithms}
\Crefname{algo}{Algorithm}{Algorithms}
\crefname{algorithm}{algorithm}{algorithms}
\Crefname{algorithm}{Algorithm}{Algorithms}
\crefname{conj}{conjecture}{conjectures}
\Crefname{conj}{Conjecture}{Conjectures}
\crefname{obs}{observation}{observations}
\Crefname{obs}{Observation}{Observations}
\newcommand{\totalPulls}{T}
\newcommand{\setofArmsC}{\mathcal{C}}
\newcommand{\setofArmsK}{\mathcal{K}}

\newcommand{\numArms}{K}

\newcommand{\arm}{k}
\newcommand{\meanReward}{\mu}
\newcommand{\Index}{I}
\newcommand{\pulls}{n}
\newcommand{\slot}{t}
\newcommand{\gap}{\Delta}
\newcommand{\estimateMean}{\hat{\phi}}
\newcommand{\expectedPseudoReward}{\phi}

\newcommand{\regret}{Reg}
\newcommand{\reward}{r}
\newcommand{\estimateReward}{s}
\newcommand{\optimistGap}{\tilde{\Delta}}
\newcommand{\indicator}{\mathbbm{1}}

\newcommand{\bestarm}{{k^*}}

\def \OO {\mathrm{O}}

\begin{document}

\title{Multi-Armed Bandits with Correlated Arms}

\author{\name Samarth Gupta \email samarthg@andrew.cmu.edu \\
 \addr Carnegie Mellon University\\
 Pittsburgh, PA 15213 
 \AND
 \name Shreyas Chaudhari \email schaudh2@andrew.cmu.edu \\
 \addr Carnegie Mellon University\\
 Pittsburgh, PA 15213 
 \AND
 \name Gauri Joshi \email gaurij@andrew.cmu.edu \\
 \addr Carnegie Mellon University\\
 Pittsburgh, PA 15213 
 \AND
 \name Osman Ya\u{g}an \email oyagan@andrew.cmu.edu\\
 \addr Carnegie Mellon University\\
 Pittsburgh, PA 15213}

\editor{No editors}
\maketitle

\begin{abstract}
We consider a multi-armed bandit framework where the rewards obtained by pulling different arms are correlated. We develop a unified approach to leverage these reward correlations and present fundamental generalizations of classic bandit algorithms to the correlated setting. We present a unified proof technique to analyze the proposed algorithms. Rigorous analysis of C-UCB (the correlated bandit versions of Upper-confidence-bound) reveals that the algorithm end up pulling certain sub-optimal arms, termed as \textit{non-competitive}, only $\OO(1)$ times, as opposed to the $\OO(\log T)$ pulls required by classic bandit algorithms such as UCB, TS etc. We present regret-lower bound and show that when arms are correlated through a latent random source, our algorithms obtain {\em order-optimal} regret. We validate the proposed algorithms via experiments on the MovieLens and Goodreads datasets, and show significant improvement over classical bandit algorithms.
\end{abstract}

\maketitle

\section{Introduction}
\label{sec:introduction}

\subsection{Background and Motivation}

\textbf{Classical Multi-armed Bandits.} The \emph{multi-armed bandit} (MAB) problem falls under the class of sequential decision making problems. In the classical multi-armed bandit problem, there are $K$ arms, with each arm having an {\em unknown} reward distribution. At each round $t$, we need to decide an arm $k_{t} \in \mathcal{K}$ and we receive a random reward $R_{k_t}$ drawn from the reward distribution of arm $k_{t}$. The goal in the classical multi-armed bandit  is to maximize the {\em long-term} cumulative reward. In order to maximize cumulative reward, it is important to balance the {\em exploration-exploitation} trade-off, i.e., pulling each arm enough number of times to identify the one with the highest mean reward, while trying to make sure that the arm with the highest mean reward is played as many times as possible. This problem has been well studied starting with the work of Lai and Robbins \cite{lai1985asymptotically} that proposed the upper confidence bound (UCB) arm-selection algorithm and studied its fundamental limits in terms of bounds on \emph{regret}. Subsequently, several other algorithms including Thompson Sampling (TS) \cite{agrawal2012analysis} and KL-UCB \cite{garivier2011kl}, have been proposed for this setting. The generality of the classical multi-armed bandit model allows it to be useful in numerous applications. For example, MAB algorithms are useful in medical diagnosis \cite{villar2015multi}, where the arms correspond to the different treatment mechanisms/drugs and are widely used for the problem of ad optimization \cite{agarwal2009explore} by viewing different version of ads as the arms in the MAB problem. The MAB framework is also useful in system testing \cite{tekin2017multi}, scheduling in computing systems \cite{mora2009stochastic, krishnasamy2016regret, joshi2016efficient}, and web optimization \cite{white2012bandit, agarwal2009explore}. 

\noindent
\textbf{Correlated Multi-Armed Bandits.} The classical MAB setting implicitly assumes that the rewards are independent across arms, i.e., pulling an arm $k$ does not provide any information about the reward we would have received from arm $\ell$. However, this may not be true in practice as the reward corresponding to different treatment/drugs/ad-versions are likely to be {\em correlated} with each other. For instance, similar ads/drugs may generate similar reward for the user/patient. These correlations, when modeled and accounted for, can allow us to significantly improve the cumulative reward by reducing the amount of \emph{exploration} in bandit algorithms.

\begin{figure}[t]
    \centering
    \includegraphics[width = 0.6\textwidth]{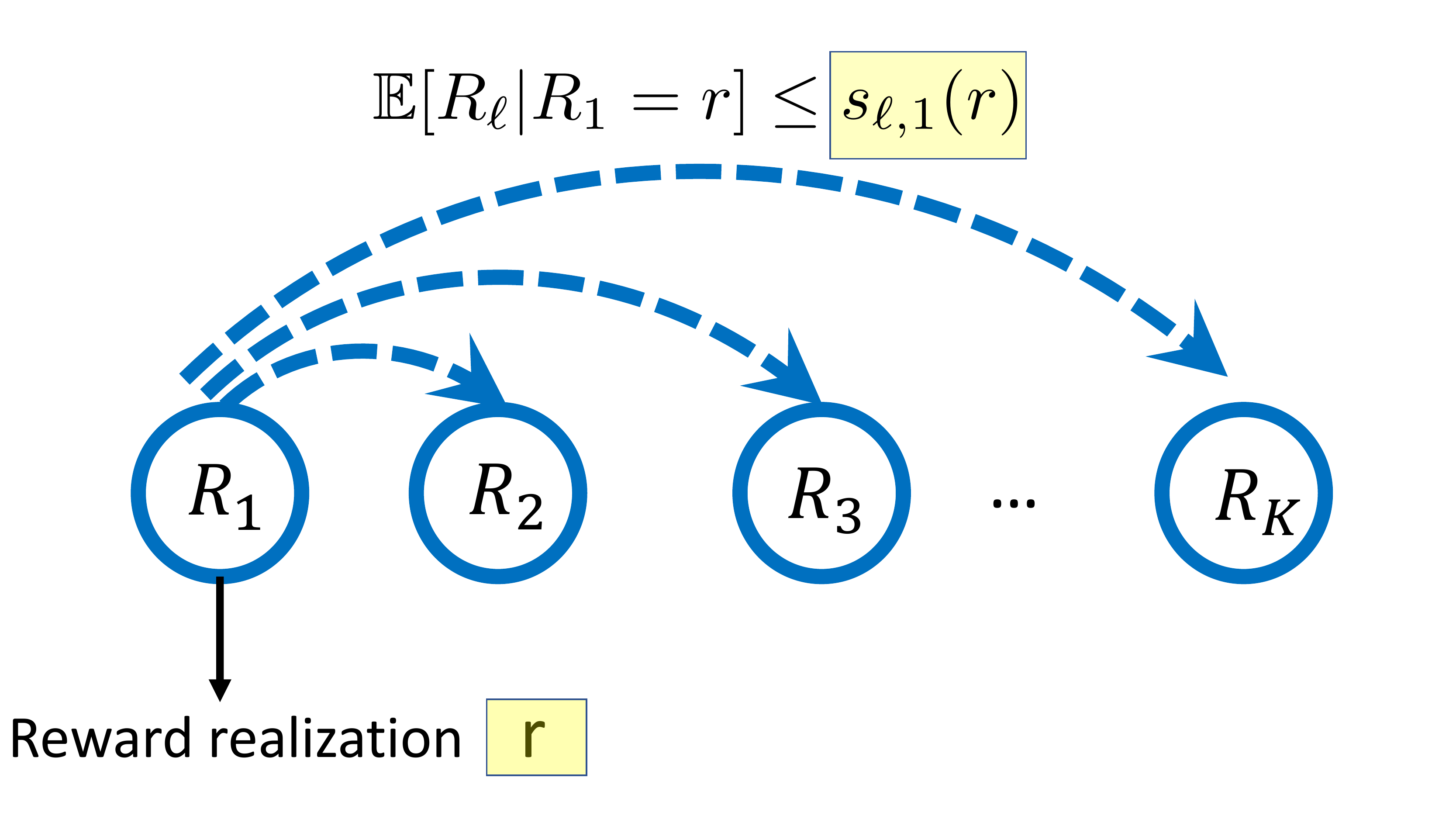}
    \caption{ Upon observing a reward $r$ from an arm $k$, pseudo-rewards $s_{\ell,k}(r),$ give us an upper bound on the conditional expectation of the reward from arm $\ell$ given that we observed reward $r$ from arm $k$. These pseudo-rewards models the correlation in rewards corresponding to different arms. 
    }
    \label{fig:pseudoModel}
\end{figure}

Motivated by this, we study a variant of the classical multi-armed bandit problem in which rewards corresponding to different arms are correlated to each other, i.e., the conditional reward distribution satisfies $f_{R_\ell | R_{k}}(r_{\ell} | r_k) \neq f_{R_\ell}(r_{\ell})$, whence $\E{R_{\ell} | R_k} \neq \E{R_{\ell}}$. Such correlations can only be learned upon obtaining samples from different arms simultaneously, i.e., by pulling multiple arms at a time. As that is not allowed in the classical Multi-Armed Bandit formulation, we assume the knowledge of such correlations in the form of prior knowledge that might be obtained through domain expertise or from controlled surveys. One way of capturing correlations is through the knowledge of the joint reward distribution. However, if the complete joint reward distribution is known, then the best-arm is known trivially. Instead, in our work, we only assume restrictive information about correlations in the form of \textit{pseudo-rewards} that constitute an upper bound on conditional expected rewards. This makes our model more general and suitable for practical applications. Fig.~\ref{fig:pseudoModel} presents an illustration of our model, where the pseudo-rewards, denoted by $s_{\ell,k}(r)$, provide an upper bound on the reward that we could have received from arm $\ell$ given
that pulling arm $k$ led to a reward of $r$; i.e., 
\vspace{-1mm}
\begin{equation}
 \mathbb{E}[R_\ell | R_k = r] \leq s_{\ell,k}(r).
 \label{eq:pseudo_reward_defn}
\end{equation}

We show that the knowledge of such bounds, even when they are not all tight, can lead to significant improvement in the cumulative reward obtained by reducing the amount of {\em exploration}  compared to classical MAB algorithms. Our proposed MAB model and algorithm can be applied in all real-world applications of the classical Multi-Armed bandit problem, where it is possible to know pseudo-rewards from domain knowledge or through surveyed data. In the next section, we illustrate the applicability of our novel correlated Multi-Armed Bandit model and its differences with the existing contextual and structured bandit works through the example of optimal {\em ad-selection}.

\subsection{An Illustrative Example}

Suppose that a company is to run a display advertising campaign for one of their products, and its creative team have designed several different versions that can be displayed. It is expected that the user engagement (in terms of click probability and time spent looking at the ad) depends the version of the ad that is displayed. In order to maximize the total user engagement over the course of the ad campaign, multi-armed bandit algorithms can be used; different versions of the ad correspond to the {\em arms} and the reward from selecting an arm is given by the clicks or time spent looking at the ad version corresponding to that arm.

\begin{figure}[th]
    \centering
    \includegraphics[width = 0.7\textwidth]{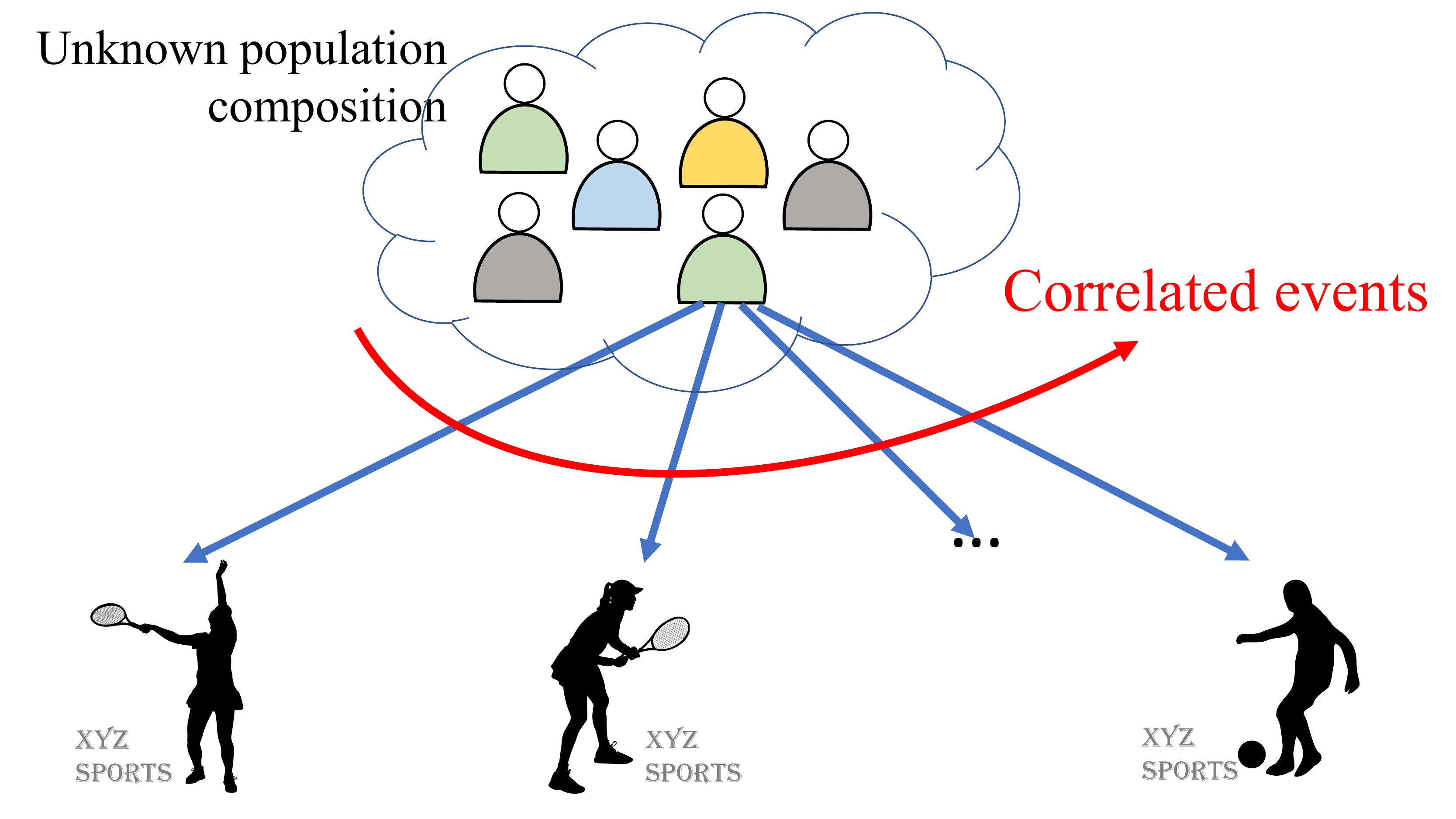}
    \caption{The ratings of a user corresponding to different versions of the same ad are likely to be correlated. For example, if a person likes first version, there is a good chance that they will also like the 2nd one as it also related to tennis. However, the population composition is unknown, i.e., the fraction of people liking the first/second or the last version is unknown.} 
    \label{fig:clooneyEx}
    \vspace{-0.2cm}
\end{figure}

\vspace{0.1cm}
\noindent
\textbf{Personalized recommendations using Contextual and Structured bandits.} Although the ad-selection problem can be solved by standard MAB algorithms, there are several specialized MAB variants that are designed to give better performance. For instance, the {\em contextual} bandit problem \cite{zhou15survey, agarwal2014taming} has been studied to provide {\em personalized} displays of the ads to the users. Here, before making a choice at each time step (i.e., deciding which version to show to a user),  we observe the {\em context}  associated with that user (e.g., age/occupation/income features). Contextual bandit algorithms learn the mappings from the context $\theta$ to the most favored version of ad $k^*(\theta)$ in an online manner and thus are useful for personalized recommendations. A closely related problem is the structured bandit problem \cite{combes2017minimal, lattimore2014bounded, abbasi2011improved, dani2008stochastic}, in which the context $\theta$ (age/ income/ occupational features) is {\em hidden} but the mean rewards for different versions of ad (arms) as a function of hidden context $\theta$ are known. Such models prove useful for personalized recommendation in which the context of the user is unknown, but the reward mappings $\mu_k(\theta)$ are known through surveyed data. 

\vspace{0.15cm}
\noindent
\textbf{Global Recommendations using Correlated-Reward Bandits.}
In this work we study a variant of the classical multi-armed bandit problem in which rewards corresponding to different arms are correlated to each other. 
In many practical settings, the reward we get from different arms
at any given step are likely to be correlated. In the ad-selection example given in \Cref{fig:clooneyEx}, a user reacting positively (by clicking, ordering, etc.) to the first version of the ad with a girl playing tennis might also be more likely to click the second version as it is also related to tennis; of course one can construct examples where there is negative correlation between click events to different ads. The model we study in this paper explicitly captures these correlations through the knowledge of pseudo-rewards $s_{\ell,k}(r)$ (See \Cref{fig:pseudoModel}). Similar to the classical MAB setting, the goal here is to display versions of the ad to maximize user engagement. In addition, unlike contextual bandits, we do not observe the context (age/occupational/income) features of the user and do not focus on providing personalized recommendation. Instead our goal is to provide global recommendations to a population whose demographics is unknown. Unlike {\em structured bandits}, we do not assume that the mean rewards are functions of a hidden context parameter $\theta$. In structured bandits, although the {\em mean} rewards depend on $\theta$ the reward realizations can still be independent. See \Cref{subsec:strucBandit} for more details.

\subsection{Main Contributions and Organization}

\vspace{0.1cm}
\textbf{i) A General and Previously Unexplored Correlated Multi-Armed Bandit Model.} In \Cref{sec:problem} we describe our novel correlated multi-armed bandit model, in which rewards of a user corresponding to different arms are correlated with each other. This correlation is captured by the knowledge of \textit{pseudo-rewards}, which are upper bounds on the conditional mean reward of arm $\ell$ given reward of arm $k$. In practice, pseudo-rewards can be obtained via expert/domain knowledge (for example, common ingredients in two drugs that are being considered to treat an ailment) or controlled surveys (for example, beta-testing users who are asked to rate different versions of an ad). A key advantage of our framework is that pseudo-rewards are just upper bounds on the conditional expected rewards and can be arbitrarily loose. This also makes the proposed framework and algorithm directly usable in practice -- if some pseudo-rewards are unknown due to lack of domain knowledge/data, they can simply be replaced by the maximum possible reward entries, which serves a natural upper bound.

\vspace{0.1cm}
\noindent
\textbf{ii) An approach to generalize algorithms to the Correlated MAB setting.}
We propose a novel approach in \Cref{sec:algorithm} that extends any classical bandit (such as UCB, TS, KL-UCB etc.) algorithm to the correlated MAB setting studied in this paper. This is done by making use of the pseudo-rewards to reduce exploration in standard bandit algorithms. We refer to this algorithm as \textsc{C-Bandit} where \textsc{Bandit} refers to the classical bandit algorithm used in the last step of the algorithm (i.e., UCB/TS/KL-UCB). 

\vspace{0.1cm}
\noindent
\textbf{iii) Unified Regret Analysis} We study the performance of our proposed algorithms by analyzing their expected \emph{regret}, $\E{\text{Reg}(T)}$. The regret of an algorithm is defined as the difference between the cumulative reward of a \emph{genie} policy, that always pulls the optimal arm $k^*$, and the cumulative reward obtained by the algorithm over $T$ rounds. By doing regret analysis of C-UCB, we obtain the following upper bound on the expected regret of C-UCB. 

\begin{prop}[Upper Bound on Expected Regret]
The expected cumulative regret of the C-UCB algorithm is upper bounded as
\begin{equation}
    \E{Reg(T)} \leq (C-1) \cdot \OO(\log T) + \OO(1),
\end{equation}
\label{coro:teaser}
\end{prop}

Here $C$ denotes the number of \emph{competitive} arms. An arm $k$ is said to be \emph{competitive} if expected pseudo-reward of arm $k$ with respect to the optimal arm $k^*$ is larger than the mean reward of arm $k^*$, that is, if $\E{s_{k,k^*}(r)} \geq \mu_{k^*}$, otherwise, the arm is said to be non-competitive. The result in \Cref{coro:teaser} arises from the fact that the C-UCB algorithm ends up pulling the non-competitive arms only $\OO(1)$ times and only the competitive sub-optimal arms are pulled $\OO(\log T)$ times. In contrast to UCB, that pulls all $K-1$ sub-optimal arms $\OO(\log T)$ times, our proposed C-UCB  algorithm pulls only $C-1 \leq K-1$ arms $\OO(\log T)$ times. In fact, when $C = 1$, our proposed algorithm achieves {\em bounded} regret meaning that after some finite step, no arm but the optimal one will be selected. In this sense, we reduce a $K$-armed bandit problem to a $C$-armed bandit problem. We emphasize that $k^*$, $\mu^*$ and $C$ are {\em all} unknown to the algorithm at the beginning.

We present our detailed regret bounds and analysis in \Cref{sec:regret}. A rigorous  analysis of the regret achieved under C-UCB is given through a unified technique. This technique can be of broad  interest as we also provide a recipe to obtain regret analysis for any \textit{C-Bandit} algorithm. For instance, the analysis of C-KL-UCB can be easily done through our provided outline.

\vspace{0.1cm}
\noindent
\textbf{iv) Evaluation using real-world datasets.} 

We perform simulations to validate our theoretical results in \Cref{sec:simulation}. In \Cref{sec:experiments}, we do extensive validation of our results by performing experiments on two real-world datasets, namely \textsc {Movielens} and \textsc{Goodreads}, which show that the proposed approach yields drastically smaller regret than classical Multi-Armed Bandit strategies.

\section{Problem Formulation}
\label{sec:problem}

\subsection{Correlated Multi-Armed Bandit Model}
\label{subsec:generalModel}

\begin{table}[t]
\centering
\begin{tabular}{|l|l|l|l|l|}
\cline{1-2} \cline{4-5}
\textbf{r} & \textbf{$s_{2,1}(r)$} &  & \textbf{r} & \textbf{$s_{1,2}(r)$} \\ \cline{1-2} \cline{4-5} 
\textbf{0} & 0.7                   &  & \textbf{0} & 0.8                     \\ \cline{1-2} \cline{4-5} 
\textbf{1} & 0.4                   &  & \textbf{1} & 0.5                     \\ \cline{1-2} \cline{4-5} 
\end{tabular}
\\ \vspace{2mm}
\parbox{.45\linewidth}{
\centering
\begin{tabular}{|l|l|l|}
\hline
    \textbf{(a)}    & $R_1 = 0$ & $R_1 = 1$ \\ \hline
$R_2 = 0$ & 0.2       & 0.4       \\ \hline
$R_2 = 1$ & 0.2       & 0.2       \\ \hline
\end{tabular}
}
\hfill
\parbox{.45\linewidth}{
\centering
\begin{tabular}{|l|l|l|}
\hline
    \textbf{(b)}    & $R_1 = 0$ & $R_1 = 1$ \\ \hline
$R_2 = 0$ & 0.2       & 0.3       \\ \hline
$R_2 = 1$ & 0.4       & 0.1       \\ \hline
\end{tabular}
}
\caption{The top row shows the pseudo-rewards of arms 1 and 2, i.e., upper bounds on the conditional expected rewards (which are known to the player). The bottom row depicts two possible joint probability distribution (unknown to the player). Under distribution (a), Arm 1 is optimal whereas Arm 2 is optimal under distribution (b).}
\label{tab:pseudoBin}
\vspace{-0.2cm}
\end{table}

Consider a Multi-Armed Bandit setting with $K$ arms $\{1,2, \ldots K\}$. At each round $t$, a user enters the system and we need to decide an arm $k_t$ to display to the user. Upon pulling arm $k_t$, we receive a random reward $R_{k_t} \in [0,B]$. Our goal is to maximize the cumulative reward over time. The expected reward of arm $k$, is denoted by $\mu_k$. If we knew the arm with highest mean, i.e., $k^* = \arg \max_{k \in \mathcal{K}} \mu_k$ beforehand, then we would always pull arm $k^*$ to maximize expected cumulative reward. We now define the cumulative regret, minimizing which is equivalent to maximizing cumulative reward:
\begin{equation}
    Reg(T) = \sum_{t = 1}^{T} \mu_{k_t} - \mu_{k^*} = \sum_{k \neq k^*} n_k(T) \Delta_k.
\label{eqn:regretdefinition}
\end{equation}
Here, $n_k(T)$ denotes the number of times a sub-optimal arm is pulled till round $T$ and $\Delta_k$ denotes the {\em sub-optimality gap} of arm $k$, i.e., $\Delta_k = \mu_{k^*} -\mu_k$.

The classical multi-Armed bandit setting implicitly assumes the rewards $R_1, R_2 \ldots R_K$ are independent, that is, $\Pr(R_{\ell} = r_\ell | R_k = r) = \Pr(R_{\ell} = r_\ell) \quad \forall{r_{\ell},r} \& \forall{\ell,k},$ which implies that, $\E{R_{\ell} | R_k = r} = \E{R_{\ell}} \quad \forall{r,\ell,k}$. However, in most practical scenarios this assumption is unlikely to be true. In fact, rewards of a user corresponding to different arms are likely to be correlated. Motivated by this we consider a setup where the conditional distribution of the reward from arm $\ell$ given reward from arm $k$ is not equal to the probability distribution of the reward from arm $\ell$, i.e., $f_{R_\ell | R_{k}}(r_{\ell} | r_k) \neq f_{R_\ell}(r_{\ell})$, with $f_{R_\ell}(r_{\ell})$ denoting the probability distribution function of the reward from arm $\ell$. Consequently, due to such correlations, we have $\E{R_{\ell} | R_k} \neq \E{R_{\ell}}$.

In our problem setting, we consider that the player has partial knowledge about the joint distribution of correlated arms in the form of \emph{pseudo-rewards}, as defined below:

\begin{defn}[Pseudo-Reward]
Suppose we pull arm $k$ and observe reward $r$, then the pseudo-reward of arm $\ell$ with respect to arm $k$, denoted by $s_{\ell,k}(r)$, is an upper bound on the conditional expected reward of arm $\ell$, i.e.,
\begin{equation}
 \mathbb{E}[R_\ell | R_k = r] \leq s_{\ell,k}(r),
\end{equation}
without loss of generality, we define $s_{\ell,\ell}(r) = r$.
\end{defn}

The pseudo-rewards information consists of a set of $K \times K$ functions $s_{\ell,k}(r)$ over $[0,B]$. This information can be obtained in practice through either domain/expert knowledge or from controlled surveys. For instance, in the context of medical testing, where the goal is to identify the best drug to treat an ailment from among a set of $K$ possible options, the effectiveness of two drugs is correlated when the drugs share some common ingredients. Through domain knowledge of doctors, it is possible answer questions such as ``what are the chances that drug $B$ would be effective given drug $A$ was not effective?", through which we can infer the pseudo-rewards. 

\begin{table}[t]
\parbox{.3\linewidth}{
\centering
\begin{tabular}{|l|l|l|}
\hline
\textbf{r} & \textbf{$s_{2,1}(r)$} & \textbf{$s_{3,1}(r)$} \\ \hline
\textbf{0} & 0.7                   &  \textcolor{red}{2}                  \\ \hline
\textbf{1} & 0.8                   & 1.2                   \\ \hline
\textbf{2} & \textcolor{red}{2}                   & 1                   \\ \hline
\end{tabular}
}
\hfill
\parbox{.3\linewidth}{
\centering
\begin{tabular}{|l|l|l|}
\hline
\textbf{r} & \textbf{$s_{1,2}(r)$} & \textbf{$s_{3,2}(r)$} \\ \hline
\textbf{0} & 0.5                   & 1.5                   \\ \hline
\textbf{1} & 1.3                   &       \textcolor{red}{2}              \\ \hline
\textbf{2} &    \textcolor{red}{2}                 &   0.8                 \\ \hline
\end{tabular}
}
\hfill
\parbox{.3\linewidth}{
\centering
\begin{tabular}{|l|l|l|}
\hline
\textbf{r} & \textbf{$s_{1,3}(r)$} & \textbf{$s_{2,3}(r)$} \\ \hline
\textbf{0} & 1.5                    &        \textcolor{red}{2}             \\ \hline
\textbf{1} &   \textcolor{red}{2}                  & 1.3                   \\ \hline
\textbf{2} & 0.7                   & 0.75                   \\ \hline
\end{tabular}
}
\caption{If some pseudo-reward entries are unknown (due to lack of prior-knowledge/data), those entries can be replaced with the maximum possible reward and then used in the \textsc{C-BANDIT} algorithm. We do that here by entering $2$ for the entries where pseudo-rewards are unknown.}
\label{tab:paddedEntries}
\vspace{-0.2cm}
\end{table}

\subsection{Computing Pseudo-Rewards from prior-data/surveys}
The pseudo-rewards can also be learned from prior-available data, or through {\em offline} surveys in which users are presented with {\em all} $K$ arms allowing us to sample $R_1, \ldots, R_K$ jointly. Through such data, we can evaluate an estimate on the conditional expected rewards. For example in \Cref{tab:pseudoBin}, we can look at all users who obtained $0$ reward for Arm 1 and calculate their average reward for Arm 2, say $\hat{\mu}_{2,1}(0)$. This average provides an estimate on the conditional expected reward. Since we only need an upper bound on $\E{R_2 | R_1 = 0}$, we can use several  approaches to construct the pseudo-rewards.
\begin{enumerate}
    \item If the training data is \textit{large}, one can use the empirical estimate $\hat{\mu}_{2,1}(0)$ directly as $s_{2,1}(0)$, because through law of large numbers, the empirical average equals the $\E{R_2 | R_1 = 0}$.
    \item Alternatively, we can set $s_{2,1}(0) = \hat{\mu}_{2,1}(0) + \hat{\sigma}_{2,1}(0)$, with $\hat{\sigma}_{2,1}(0)$ denoting the empirical standard deviation on the conditional reward of arm 2, to ensure that pseudo-reward is an upper bound on the conditional expected reward.
    \item In addition, pseudo-rewards for any unknown conditional mean reward could be filled with the maximum possible reward for the corresponding arm.  \Cref{tab:paddedEntries} shows an example of a 3-armed bandit problem where some pseudo-reward entries are unknown, e.g., due to lack of data. We can fill these missing entries with  maximum possible reward $(i.e., 2)$ as shown in \Cref{tab:paddedEntries} to complete the pseudo-reward entries. 
    \item  If through the training data, we obtain a soft upper bound $u$ on $\E{R_2|R_1 = 0}$ that holds with probability $1-\delta$, then we can translate it to the pseudo-reward $s_{2,1}(0) = u \times (1 - \delta) + 2 \times \delta$, (assuming maximum possible reward is 2). 
\end{enumerate}

\begin{rem}
Note that the pseudo-rewards are upper bounds on the expected conditional reward and not hard bounds on the conditional reward itself. This makes our problem setup practical as upper bounds on expected conditional reward are easier to obtain, as illustrated in the previous paragraph.
\end{rem}

\begin{rem}[Reduction to Classical Multi-Armed Bandits]
When all pseudo-reward entries are unknown, then all pseudo-reward entries can be filled with maximum possible reward for each arm, that is, $s_{\ell, k}(r) = B$  $\forall{r,\ell,k}$. In such a case, the problem framework studied in this paper reduces to the setting of the classical Multi-Armed Bandit problem and our proposed $\textsc{C-Bandit}$ algorithm performs exactly as standard \textsc{bandit} (for e.g., UCB, TS etc.) algorithms.  
\end{rem}

While the pseudo-rewards are known in our setup, the underlying joint probability distribution of rewards is unknown. For instance, \Cref{tab:pseudoBin} (a) and \Cref{tab:pseudoBin} (b) show two joint probability distributions of the rewards that are both possible given the pseudo-rewards at the top of \Cref{tab:pseudoBin}.
If the joint distribution is as given in \Cref{tab:pseudoBin} (a), then Arm 1 is optimal, while Arm 2 is optimal if the joint  distribution is as given in \Cref{tab:pseudoBin}(b).

\newadd{
\begin{rem} For a setting where reward domain is \emph{large} or there are a large number of arms, it may be difficult to learn the pseudo-reward entries from prior-data. In such scenarios, the knowledge of additional correlation structure may be helpful to know the value of pseudo-rewards. We describe one such structure in the next section where rewards are correlated through a latent random source and show how to evaluate pseudo-rewards in such a scenario.
\end{rem}
}

\subsection{Special Case: Correlated Bandits with a Latent Random Source}
\label{subsec:specialCase}

\begin{figure}[t]
    \centering
    \includegraphics[width = 0.6\textwidth]{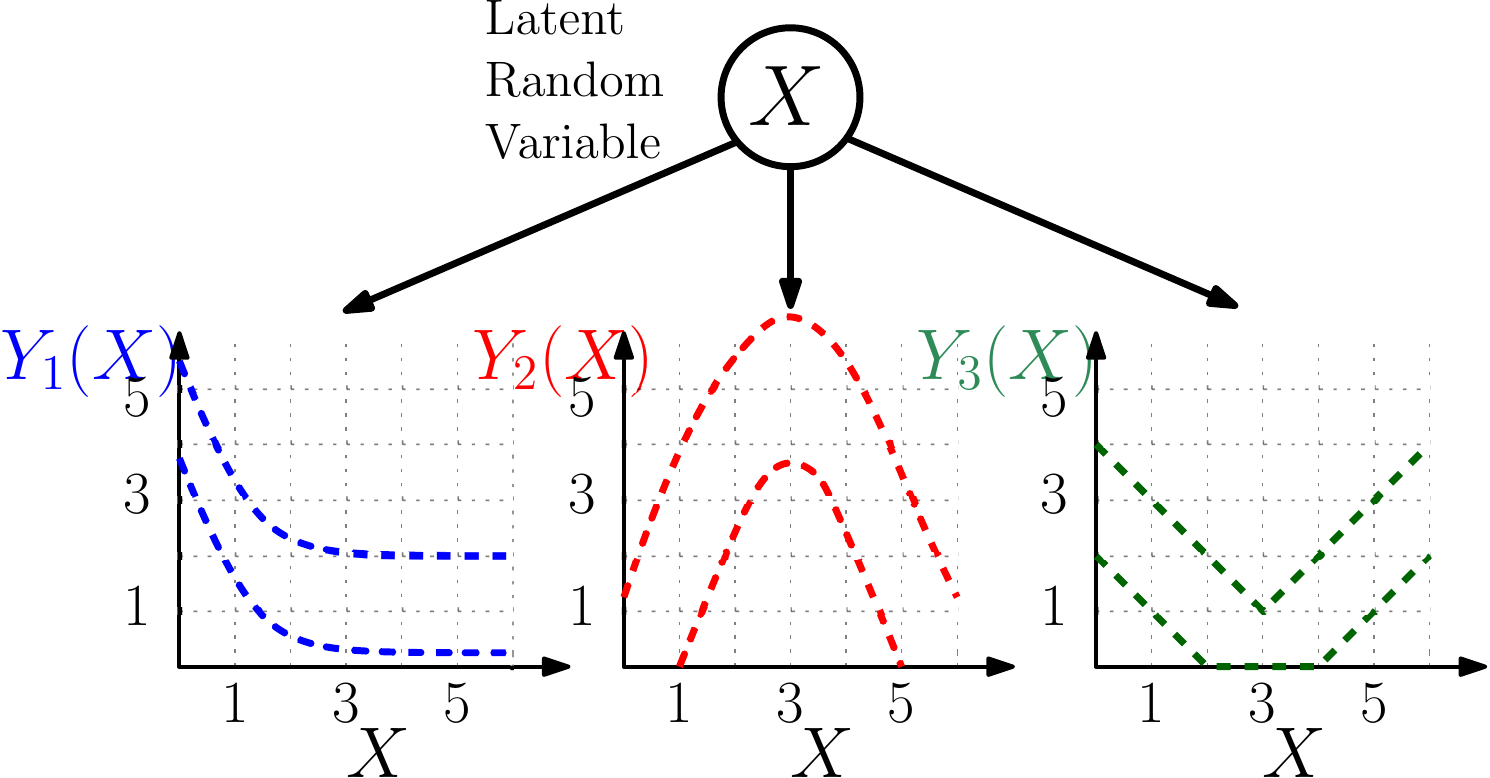}
    \caption{A special case of our proposed problem framework is a setting in which rewards for different arms are correlated through a hidden random variable X. At each round $X$ takes a realization in $\mathcal{X}$. The reward obtained from an arm $k$ is $Y_k(X)$. The figure illustrates lower bounds and upper bounds on $Y_k(X)$ (through dotted lines). For instance, when $X$ takes the realization $1$, reward of arm 3 is a random variable bounded between $1$ and $3$. }
    \label{fig:latentExample}
    \vspace{-0.2cm}
\end{figure}
\begin{figure}
    \centering
    \includegraphics[width = 0.6\textwidth]{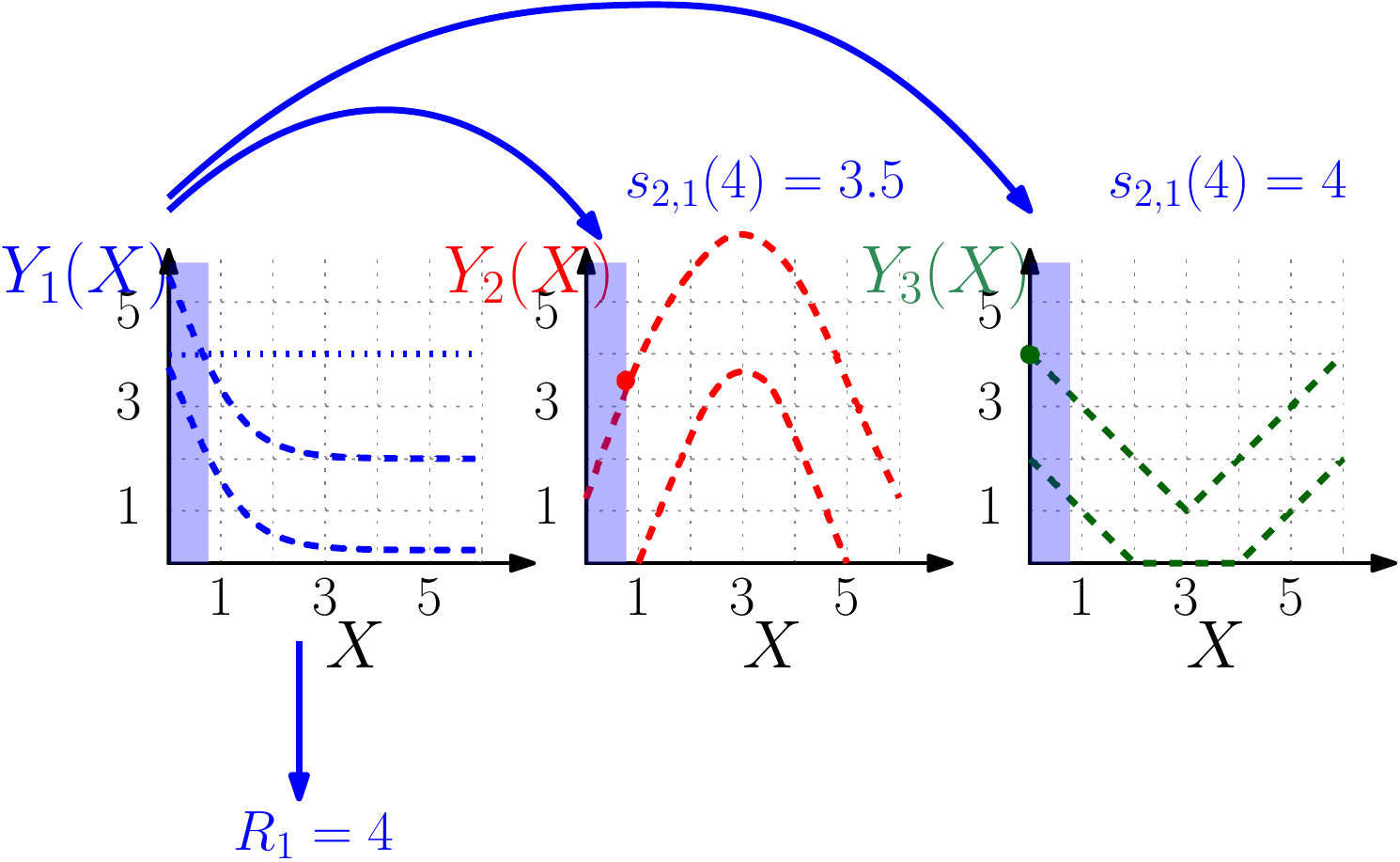}
    \caption{An illustration on how to calculate pseudo-rewards in CMAB with latent random source. Upon observing a reward of 4 from arm 1, we can see that the maximum possible reward for arms 2 and 3 is 3.5 and 4 respectively. Therefore, $s_{2,1}(4) = 3.5$ and $s_{3,1}(4) = 4$. }
    \label{fig:latentReward}
    \vspace{-0.2cm}
\end{figure}

Our proposed correlated multi-armed bandit framework subsumes many interesting and previously unexplored multi-armed bandit settings. One such special case is the correlated multi-armed bandit model where the rewards depend on a common latent source of randomness \cite{gupta2020correlated}. More concretely, the rewards of different arms are correlated through a hidden random variable $X$ (see \Cref{fig:latentExample}). At each round $t$, $X$ takes a an i.i.d. realization $X_t \in \mathcal{X}$ (unobserved to the player) and upon pulling arm $k$, we observe a random reward $Y_k(X_t)$. The latent random variable $X$ here could represent the \textit{features} (i.e., age/occupation etc.) of the user arriving to the system, to whom we show one of the $K$ arms. These \emph{features} of the user are hidden in the problem due to  privacy concerns. The random reward $Y_k(X_t)$ represents the preference of user with context $X_t$ for the $k^{th}$ version of the ad, for the application of ad-selection.

In this problem setup, upper and lower bounds on $Y_k(X)$, namely $\bar{g}_k(X)$ and $\underline{g}_k(X)$ are known. For instance, the information on upper and lower bounds of $Y_k(X_t)$ could represent knowledge of the form that \textit{children of age 5-10 rate documentaries only in the range 1-3 out of 5}. Such information can be known or learned through prior available data. While the bounds on $Y_k(X)$ are known, the distribution of $X$ and reward distribution within the bounds is unknown, due to which the optimal arm is not known beforehand. Thus, an online approach is needed to minimize the regret. 

It is possible to translate this setting to the general framework described in the problem by transforming the mappings $Y_k(X)$ to pseudo-rewards $s_{\ell,k}(r)$. Recall the pseudo-rewards represent an upper bound on the conditional expectation of the rewards. In this framework, $s_{\ell,k}(r)$ can be calculated as:
$$s_{\ell,k}(r) = \max_{\{x: \underline{g}_k(x) \leq r \leq \bar{g}_k(x)\}}  \bar{g}_{\ell}(x),$$
where $\underline{g}_k(x)$ and $\bar{g}_k(x)$ represent upper and lower bounds on $Y_k(x)$. Upon observing a realization from arm $k$, it is possible to estimate the maximum possible reward that would have been obtained from arm $\ell$ through the knowledge of bounds on $Y_k(X)$.

\Cref{fig:latentReward} illustrates how pseudo-reward is evaluated when we obtain a reward $r = 4$ by pulling arm 1. We first infer that $X$ lies in $[0, 0.8]$ if $r = 4$ and then find the maximum possible reward for arm 2 and arm 3 in $[0,0.8]$. Once these pseudo-rewards are constructed, the problem fits in the general framework described in this paper and we can use the algorithms proposed for this setting directly.

\begin{rem}
In the scenario where $\underline{g}_k(x)$ and $\bar{g}_k(x)$ are soft lower and upper bounds, i.e., $\underline{g}_k(x) \leq Y_k(x) \leq \bar{g}_k(x)$ w.p. $1 - \delta$, we can still construct pseudo-reward as follows: 
$$s_{\ell,k}(r) = (1 - \delta)^2 \times \left( \max_{\{x: \underline{g}_k(x) \leq r \leq \bar{g}_k(x)\}}  \bar{g}_{\ell}(x)  \right) + (1 - (1 - \delta)^2) \times M,$$ where $M$ is the maximum possible reward an arm can provide. Thus our proposed framework and algorithms work under this setting as well.\footnote{We evaluate a range of values within which $x$ lies based on the reward with probability $1-\delta$. The maximum possible reward of arm $\ell$ for values of $x$ is then identified with probability $1-\delta$. Due to this, with probability $(1 - \delta)^2$, conditional reward of arm $\ell$ is at-most $\max_{\{x: \underline{g}_k(x) \leq r \leq \bar{g}_k(x)\}}  \bar{g}_{\ell}(x)$.}
\end{rem}

\subsection{Comparison with parametric (structured) models}
\label{subsec:strucBandit}

As mentioned in \Cref{sec:introduction}, a seemingly related model is the structured bandits model \cite{combes2017minimal, lattimore2014bounded, gupta2018unified}. Structured bandits is a class of problems that cover linear bandits \cite{abbasi2011improved}, generalized linear bandits \cite{filippi2010parametric}, Lipschitz bandits \cite{magureanu2014lipschitz}, global bandits \cite{ata2015global}, regional bandits \cite{wang2018regional} etc. In the structured bandits setup, mean rewards corresponding to different arms are related to one another through a hidden parameter $\theta$. The underlying value of $\theta$ is fixed and the mean reward mappings $\theta \rightarrow \mu_k(\theta)$ are known. Similarly, \cite{pandey2007multi-armed} studies a dependent armed bandit problem, that also has mean rewards corresponding to different arms related to one another. It considers a parametric model, where mean rewards of different arms are drawn from one of the $K$ clusters, each having an unknown parameter $\pi_{i}$. All of these  models are fundamentally different from the problem setting considered in this paper. We list some of the differences with the structured  bandits (and the model in \cite{pandey2007multi-armed}) below.

\begin{enumerate}
    \item In this work we explicitly model the correlations in the rewards of a user corresponding to different arms. While mean rewards are related to each other in structured bandits and \cite{pandey2007multi-armed}, the reward realizations are not necessarily correlated.
    \item The model studied here is non-parametric in the sense that there is no hidden feature space as is the case in structured bandits and the work of Pandey et al. \cite{pandey2007multi-armed}. 
    \item In structured bandits, the reward mappings from $\theta$ to $\mu_k(\theta)$ need to be {\em exact}. If they happen to be incorrect, then the algorithms for structured bandit cannot be used as they rely on the correctness of $\mu_k(\theta)$ to construct confidence intervals on the unknown parameter $\theta$. In contrast, the model studied here relies on the pseudo-rewards being upper bounds on conditional expectations. These bounds need not be tight and the proposed C-Bandit algorithms adjust accordingly and perform at least as well as the corresponding classical bandit algorithm. 
    \item Similar to the structured bandits, the unimodal bandit framework \cite{combes2014unimodal, trinh2020solving} also assumes a structure on the mean rewards and does not capture the reward correlations explicitly. Under the unimodal framework, it is assumed that the mean reward $\mu_k$ as a function of the arms $k$ has a single mode. Instead of assuming that mean rewards are related to one another, our framework explicitly captures the inherent correlations in the form of pseudo-reward. Unimodal bandits have often been used to model the problem of link-rate adaptation in wireless networks, where the mean-reward corresponding to different choices of arms is a unimodal function \cite{qureshi2019fast, qureshi2020online, gupta2019link}. The same problem can also be dealt by modeling the correlations explicitly through the pseudo-reward framework described in this paper.
\end{enumerate}

\section{The Proposed \textsc{C-BANDIT} Algorithms}
\label{sec:algorithm}

We now propose an approach that extends the classical multi-armed bandit algorithms (such as UCB, Thompson Sampling, KL-UCB) to the correlated MAB setting. At each round $\slot+1$, the UCB algorithm \cite{auer2002finite} selects the arm with the highest UCB index $I_{k,t}$, i.e.,
\begin{align}
\arm_{\slot+1} = \arg \max_{\arm \in \setofArmsK} I_{k,t},  \quad I_{k,t} =  \hat{\mu}_\arm(t) + B\sqrt{\frac{2 \log (\slot)}{\pulls_\arm(\slot)}}, \label{eqn:UCB1_index}
\end{align} 
where $\hat{\meanReward}_\arm(t)$ is the empirical mean of the rewards received from arm $\arm$ until round $t$, and $\pulls_\arm(\slot)$ is the number of times arm $\arm$ is pulled till round $\slot$. The second term in the UCB index causes the algorithm to explore arms that have been pulled only a few times (i.e., those with small $\pulls_\arm(\slot)$). Recall that we assume all rewards to be bounded within an interval of size $B$. When the index $t$ is implied by context, we abbreviate $\hat{\meanReward}_\arm(t)$ and $ \Index_\arm(t)$ to $\hat{\meanReward}_\arm$ and $\Index_\arm$ respectively in the rest of the paper.

Under Thompson sampling \cite{agrawal2013further}, the arm $\arm_{\slot+1} =\arg$ $\max_{\arm \in \setofArmsK} S_{\arm,\slot}$ is selected at time step $t+1$. Here, $S_{\arm,\slot}$ is the sample obtained from the posterior distribution of $\mu_\arm$, That is,
\begin{align}
\arm_{\slot+1} = \arg\max_{\arm \in \setofArmsK} S_{\arm,\slot}, \quad S_{\arm,\slot} \sim \mathcal{N}\left(\hat{\mu}_\arm(\slot), \frac{\beta B}{\pulls_\arm(\slot) + 1}\right),
\label{eqn:ts_index}
\end{align}
here $\beta$ is a hyperparameter for the Thompson Sampling algorithm

In the correlated MAB framework, the rewards observed from one arm can help estimate the rewards from other arms. Our key idea is to use this information to reduce the amount of exploration required. We do so by evaluating the \emph{empirical pseudo-reward} of every other arm $\ell$ with respect to an arm $\arm$ at each round $t$. Using this additional information, we identify some arms as \emph{empirically non-competitive} at round $t$, and only for this round, do not consider them as a candidate in the UCB/Thompson Sampling/(any other bandit algorithm). 

\subsection{Empirical Pseudo-Rewards}
\label{sec:pseudo_reward}

In our correlated MAB framework, pseudo-reward of arm $\ell$ with respect to arm $\arm$ provides us an estimate on the reward of arm $\ell$ through the reward sample obtained from arm $\arm$. We now define the notion of empirical pseudo-reward  which can be used to obtain an \textit{optimistic estimate} of $\mu_\ell$ through just reward samples of arm $\arm$.

\begin{defn}[Empirical and Expected Pseudo-Reward]
\label{defn:empirical_pseudo_reward}
After $\slot$ rounds, arm $\arm$ is pulled $\pulls_\arm(\slot)$ times. Using these $\pulls_\arm(\slot)$ reward realizations, we can construct the empirical pseudo-reward $\estimateMean_{\ell, \arm}(\slot)$ for each arm $\ell$ with respect to arm $\arm$ as follows. 
\begin{align}
\estimateMean_{\ell, \arm}(\slot) \triangleq \frac{\sum_{\tau=1}^{\slot} \indicator_{k_\tau = k} \ \estimateReward_{\ell, \arm}(\reward_{k_\tau})}{\pulls_\arm(\slot)}, \qquad \ell \in \{1,\ldots, K\} \setminus \{k\}, 
\end{align}
The expected pseudo-reward of arm $\ell$ with respect to arm $\arm$ is defined as
\begin{align}
\expectedPseudoReward_{\ell, \arm} \triangleq \E{\estimateReward_{\ell, \arm}(R_k)}.
\end{align}
For convenience, we set $\hat{\phi}_{k,k}(t) = \hat{\mu}_k(t)$ and $\phi_{k,k} = \mu_k$.
\end{defn}

Observe that $\E{s_{\ell,k}(R_k)} \geq \E{\E{R_\ell | R_k=r}} = \mu_\ell$. Due to this, empirical pseudo-reward $\estimateMean_{\ell, \arm}(\slot)$ can be used to obtain an estimated upper bound on $\mu_{\ell}$. Note that the empirical pseudo-reward $\estimateMean_{\ell, \arm}(\slot)$ is defined with respect to arm $\arm$ and it is only a function of the rewards observed by pulling arm $\arm$.

\subsection{The \textsc{C-Bandit} Algorithm}
\label{sec:modified_ucb}
Using the notion of empirical pseudo-rewards, we now describe a 3-step procedure to fundamentally generalize classical bandit algorithms for the correlated MAB setting.

\noindent
\textbf{Step 1: Identify the set $\mathcal{S}_t$ of significant arms:} At each round $t$, define $\mathcal{S}_t$ to be the set of arms that have at least $t/K$ samples, i.e., $\mathcal{S}_t = \{k \in \mathcal{K}: n_k(t) > \frac{t}{K}\}$. As $\mathcal{S}_t$ is the set of arms that have relatively \emph{large} number of samples, we use these arms for the purpose of identifying \emph{empirically competitive} and \emph{empirically non-competitive} arms. Furthermore, define $k^{\text{emp}}(t)$ to be the arm that has the highest empirical mean in set $\mathcal{S}_t$, i.e., $k^{\text{emp}}(t) = \argmax_{k \in \mathcal{S}_t} \hat{\mu}_k(t)$. 
\footnote{If one were to use all arms (even those that have few samples) to identify empirically non-competitive arms, it can lead to incorrect inference, as pseudo-rewards with few samples will have larger noise, which can in-turn lead to elimination of the optimal arm. Using only the arms that have been pulled $\frac{t}{K}$ times in $\mathcal{S}_t$, allows us to ensure that the non-competitive arms are pulled only $\OO(1)$ times as we show in \Cref{sec:regret}. }

\noindent
\textbf{Step 2: Identify the set of \emph{empirically competitive} arms $\mathcal{A}_t$ :} 

Using the empirical mean, $\hat{\mu}_{k^{\text{emp}}}(t)$, of the arm with highest empirical reward in the set $\mathcal{S}_t$, we define the notions of empirically non-competitive and empirically competitive arms below.

\begin{defn}[Empirically Non-Competitive arm at round $t$]
An arm $k$ is said to be Empirically Non-Competitive at round $t$, if $\min_{\ell \in \mathcal{S}_t} \hat{\phi}_{k, \ell}(t) < \hat{\mu}_{k^{\text{emp}}}(t)$.
\end{defn}

\begin{defn}[Empirically Competitive arm at round $t$]
An arm $k$ is said to be Empirically Competitive at round $t$ if $\min_{\ell \in \mathcal{S}_t} \hat{\phi}_{k, \ell}(t) \geq \hat{\mu}_{k^{\text{emp}}}(t)$. The set of all empirically competitive arms at round $t$ is denoted by $\mathcal{A}_t$.
\end{defn}

The expression $\min_{\ell \in \mathcal{S}_t} \hat{\phi}_{k, \ell}(t)$ provides the tightest estimated upper bound on mean of arm $k$, through the samples of arms in $\mathcal{S}_t$. If this estimated upper bound is smaller than the estimated mean of $k^{\text{emp}}(t)$, then we call arm $k$ as \emph{empirically non-competitive} as it seems unlikely to be optimal through the samples of arms in $\mathcal{S}_t$. If the estimated upper bound of arm $k$ is greater than $\hat{\mu}_{k^{\text{emp}}}(t)$ , i.e., $\min_{\ell \in \mathcal{S}_t} \hat{\phi}_{k, \ell}(t) \geq \hat{\mu}_{k^{\text{emp}}}(t)$, we call arm $k$ as empirically competitive at round $t$, as it cannot be inferred as sub-optimal through samples of arms in $\mathcal{S}_t$. Note that the set of empirically competitive and empirically non-competitive arms is evaluated at each round $t$ and hence an arm that is empirically non-competitive at round $t$ may be empirically competitive in subsequent rounds.

\noindent
\textbf{Step 3: Play BANDIT algorithm in $\{\mathcal{A}_t \cup \{k^{\text{emp}}(t)\}\}$} As empirically non-competitive arm seem sub-optimal to be selected at round $t$, we only consider the set of empirically competitive arms along with $k^{\text{emp}}(t)$ in this step of the algorithm. At round $t$, we play a BANDIT algorithm from the set $\mathcal{A}_t \cup \{k^{\text{emp}}(t)\}$. For instance, the C-UCB pulls the arm $$k_t = \arg \max_{\arm \in \{\mathcal{A}_t \cup \arm^{\textit{emp}}(t) \}} I_{k,t-1},$$ 
where $\Index_{\arm,t-1}$ is the UCB index defined in \eqref{eqn:UCB1_index}. 
    
Similarly, C-TS pulls the arm $$k_t = \arg \max_{\arm \in \{\mathcal{A}_t \cup \arm^{\textit{emp}}(t) \}} S_{k,t-1},$$ where $S_{k,t}$ is the Thompson sample defined in \eqref{eqn:ts_index}). At the end of each round we update the empirical pseudo-rewards $\estimateMean_{\ell,\arm_\slot}(t)$ for all $\ell$, the empirical reward for arm $k_t$.

Note that our \textsc{C-BANDIT} approach allows using any classical Multi-Armed Bandit algorithm in the correlated Multi-Armed Bandit setting. This is important because some algorithms such as Thompson Sampling and KL-UCB are known to obtain much better empirical performance over UCB. Extending those to the correlated MAB setting allows us to have the superior empirical performance over UCB even in the correlated setting. This benefit is demonstrated in our simulations and experiments described in \Cref{sec:simulation} and \Cref{sec:experiments}.
\begin{rem}[Pseudo-lower bounds]
If suppose one had the information about pseudo-lower bounds (which are lower bounds on conditional expected rewards), then it is possible to use this in our correlated bandit framework. In step 2 of our algorithm, we identify an arm $k$ as empirically non-competitive if $\min_{\ell \in \mathcal{S}_t} \hat{\phi}_{k,\ell}(t) < \hat{\mu}_k^{\text{emp}}(t)$. We can maintain empirical pseudo-lower bound $\hat{w}_{i,j}(t)$ of each arm $i$ with respect to every other arm $j$. Then, we can replace the step 2 of our algorithm by calling an arm empirically non-competitive if $\min_{\ell \in \mathcal{S}_t} \hat{\phi}_{k,\ell}(t) < \max_{i \in \mathcal{S}_t} \max_{j \in \mathcal{S}_t} \hat{w}_{i,j}(t)$. In the situation where pseudo-lower bounds are unknown, they can be set to $-\infty$ and the algorithm reduces to the C-Bandit algorithm proposed in the paper. We can expect the empirical performance of this algorithm (which is aware of pseudo-lower bounds) to be slightly better than the C-Bandit algorithm. However, its regret guarantees will be the same as that of the C-Bandit algorithm. This is because pseudo-upper bounds are crucial to deciding whether an arm is competitive/non-competitive (defined in the next section), and pseudo-lower bounds are not. Put differently, even in the presence of pseudo-lower bound, the definition of non-competitive and competitive arms (Definition 5) remains the same.  
\end{rem}

\begin{algorithm}[t]
\hrule 
\vspace{0.1in}
\begin{algorithmic}[1]
\STATE \textbf{Input:} Pseudo-rewards $s_{\ell,k}(r)$ 
\STATE \textbf{Initialize:} $\pulls_\arm = 0, \Index_\arm = \infty$ for all $\arm \in \{1, 2, \dots \numArms\}$ 
\FOR{ each round $\slot$}
\STATE Find $\mathcal{S}_t = \{k: n_k(t) \geq \frac{t}{K}\}$, the arm that have been pulled significant number of times till $t-1$. Define $k^{\text{emp}}(t) = \argmax_{k \in \mathcal{S}_t} \hat{\mu}_k(t)$.
\STATE Initialize the empirically competitive set $\mathcal{A}_t$ as an empty set $\{ \}$.
\FOR{$\arm \in \mathcal{K}$}
\IF {$ \min_{\ell \in \mathcal{S}_t} \hat{\phi}_{k,\ell}(t) \geq \hat{\mu}_{k^{\text{emp}}}(t)$} 
\STATE Add arm $\arm$ to the empirically competitive set: $\mathcal{A}_t = \mathcal{A}_t \cup \{\arm\}$
\ENDIF
\ENDFOR
\STATE Apply UCB1 over arms in $\mathcal{A}_t \cup \{\arm^{\text{emp}}(t)\} $ by pulling arm $\arm_\slot = \arg \max_{\arm \in \mathcal{A}_t \cup \{\arm^{\text{emp}}(t)\}} \Index_\arm(t-1)$
\STATE Receive reward $\reward_{\slot}$, and update $\pulls_{\arm_\slot}(t) = \pulls_{\arm_\slot}(t) + 1$
\STATE Update Empirical reward: 
$\hat{\meanReward}_{\arm_\slot}(\slot) = \frac{\hat{\meanReward}_{\arm_\slot}(\slot-1)(\pulls_{\arm_\slot}(t)-1) + r_t }{\pulls_{\arm_\slot}(t)}$
\STATE Update the UCB Index: $\Index_{\arm_\slot}(t) = \hat{\meanReward}_{\arm_\slot}(t) + B\sqrt{\frac{2 \log \slot}{\pulls_{\arm_\slot}(t)}}$
\STATE Update empirical pseudo-rewards for all $ \arm \neq \arm_\slot$:  $\estimateMean_{\arm, \arm_\slot}(\slot) = \sum_{\tau: \arm_\tau = \arm_\slot} \estimateReward_{\arm,\arm_\tau}( \reward_\tau)/\pulls_{\arm_\slot}(t)$
\ENDFOR
\end{algorithmic}
\vspace{0.1in}
\hrule
\caption{C-UCB Correlated UCB Algorithm}
\label{alg:formalAlgo}
\vspace{-0.2cm}
\end{algorithm}

\begin{algorithm}[t]
\hrule 
\vspace{0.1in}
\begin{algorithmic}[1]
\STATE Steps 1 - 10 as in C-UCB
 \STATE \textbf{Apply TS over arms in $\mathcal{A}_t \cup \{\arm^{\text{emp}}(t)\} $} by pulling arm $\arm_\slot = \arg \max_{\arm \in \mathcal{A}_t \cup \{\arm^{\text{emp}}(t)\}} S_{k,t}$, where $S_{k,t} \sim \mathcal{N}\left(\hat{\mu}_k(t), \frac{\beta B}{n_k(t) + 1}\right)$.
\STATE Receive reward $\reward_{\slot}$, and update $\pulls_{\arm_\slot}(t)$, $\hat{\mu}_{k_t}(t)$ and empirical pseudo-rewards $\hat{\phi}_{k,k_t}(t)$. 
\end{algorithmic}
\vspace{0.1in}
\hrule
\caption{C-TS Correlated TS Algorithm}
\label{alg:formalAlgoTS}
\vspace{-0.2cm}
\end{algorithm}
\section{Regret Analysis and Bounds}
\label{sec:regret}
We now characterize the performance of the C-UCB algorithm by analyzing the expected value of the cumulative regret \eqref{eqn:regretdefinition}. The expected regret can be expressed as
\begin{align}
    \E{\regret(\totalPulls)} &=
   \sum_{\arm = 1}^{\numArms} \E{\pulls_\arm(T)} \Delta_\arm, 
    \label{eqn:exp_regret}
\end{align}
where $\Delta_\arm = \meanReward_{\arm^*} - \meanReward_\arm $ is the sub-optimality gap of arm $\arm$ with respect to the optimal arm $\arm^*$, and $\pulls_\arm(T)$ is the number of times arm $\arm$ is pulled in $\totalPulls$ slots.

For the regret analysis, we assume without loss of generality that the rewards are between 0 and 1 for all $\arm \in \{ 1, 2, \dots \numArms \}$. Note that the \textsc{C-Bandit} algorithms do not require this condition, and the regret analysis can also be generalized to any bounded rewards.

\subsection{Competitive and Non-competitive arms with respect to Arm k}
\label{sec:competitive}
For the purpose of regret analysis in \Cref{sec:regret}, we need to understand which arms are empirically competitive as $t \rightarrow \infty$. We do so by defining the notions of Competitive and Non-Competitive arms.
\begin{defn}[Non-Competitive and Competitive arms]
An arm $\ell$ is said to be non-competitive if the expected reward of optimal arm $k^*$ is larger than the expected pseudo-reward of arm $\ell$ with respect to the optimal arm $\arm^*$, i.e, if, $\optimistGap_{\ell,\arm^*} \triangleq \mu_{k^*} - \phi_{\ell,k^*} > 0$. Similarly, an arm $\ell$ is said to be competitive if $\optimistGap_{\ell,\arm^*} = \mu_{k^*} - \phi_{\ell, k^*}  <= 0$. The unique best arm $k^*$ has $\optimistGap_{\arm^*,\arm^*} = \mu_{k^*} - \phi_{k^*, k^*} = 0$ and is counted in the set of competitive arms.\footnote{As $t \rightarrow \infty$, only the optimal arm will remain in $\mathcal{S}_t$, and hence the definition of competitive arms only compares the expected mean of arm $k^*$ and expected pseudo-reward of arm $k$ with respect to arm $k^*$}
\end{defn}

We refer to $\optimistGap_{\ell,\arm^*}$ as the pseudo-gap of arm $\ell$ in the rest of the paper. These notions of competitiveness are used in the regret analysis in \Cref{sec:regret}. The central idea behind our correlated \textsc{C-BANDIT} approach is that after pulling the optimal arm $k^*$ sufficiently large number of times, the non-competitive (and thus sub-optimal) arms can be classified as     empirically non-competitive with increasing confidence, and thus need not be explored. As a result, the non-competitive arms will  be pulled only $\OO(1)$ times. However, the competitive arms cannot be discerned as sub-optimal by just using the rewards observed from the optimal arm, and have to be explored $\OO(\log \totalPulls)$ times each. Thus, we are able to reduce a $K$-armed bandit to a $C$-armed bandit problem, where $C$ is the number of competitive arms. \footnote{Observe that $k^*$ and subsequently $C$ are both unknown to the algorithm. Before the start of the algorithm, it is not known which arm is optimal/competitive/non-competitive. Algorithm works in an online manner by evaluating the noisy notions of competitiveness, i.e., empirically competitive arms, and ensures that only $C - 1$ of the arms are pulled $\OO(\log T)$ times.} We show this by bounding the regret of \textsc{C-BANDIT} approach. 

\subsection{Regret Bounds} 
In order to bound $\E{\regret(\totalPulls)}$ in \eqref{eqn:exp_regret}, we can analyze the expected number of times sub-optimal arms are pulled, that is, $\E{\pulls_\arm(\totalPulls)}$, for all $\arm \neq \arm^*$. \Cref{thm:NonCompetitiveBound} and \Cref{thm:CompetitiveBound} below show that $\E{\pulls_\arm(\totalPulls)}$ scales as $O(1)$ and $O(\log T)$ for non-competitive and competitive arms respectively. Recall that a sub-optimal arm is said to be non-competitive if its pseudo-gap $\optimistGap_{\arm,\arm^*}>0$, and competitive otherwise.

\begin{thm}[Expected Pulls of a Non-competitive Arm]
\label{thm:NonCompetitiveBound}

The expected number of times a non-competitive arm with pseudo-gap $\optimistGap_{\arm,\arm^*}$ is pulled by C-UCB is upper bounded as
\begin{align}
\E{\pulls_\arm(\totalPulls)} &\leq \numArms \slot_0 + K^3 \sum_{\slot= \numArms \slot_0}^{\totalPulls} 2 \left(\frac{\slot}{\numArms}\right)^{-2} + \sum_{\slot = 1}^{\totalPulls} 3\slot^{-3}, \label{eqn:upper_bnd_comp}\\
&= \OO(1), 
\end{align}
where, 
{
\begin{align*}
&\slot_0  = \inf \bigg\{\tau \geq 2: \Delta_{\text{min}} , \optimistGap_{\arm,\arm^*} \geq 4 \sqrt{\frac{2K\log \tau}{\tau}} \bigg\}. 
\end{align*}
}
\end{thm}
\begin{thm}[Expected Pulls of a Competitive Arm]
\label{thm:CompetitiveBound}
The expected number of times a competitive arm is pulled by C-UCB algorithm is upper bounded as 
\begin{align}
\E{\pulls_\arm(\totalPulls)} 
&\leq 8 \frac{\log (\totalPulls)}{\gap_\arm^2} + \left(1 + \frac{\pi^2}{3}\right) + \sum_{\slot = 1}^{\totalPulls} 2K\slot \exp\left(- \frac{\slot \gap_{\text{min}}^2}{2 \numArms}\right), \label{eqn:upper_bnd_non_comp}\\
&= \OO(\log \totalPulls) \quad \text{ where  } \gap_{\text{min}} = \min_k \Delta_\arm  > 0.
\end{align}
\end{thm}

Substituting the bounds on $\E{\pulls_\arm(\totalPulls)}$ derived in \Cref{thm:NonCompetitiveBound} and \Cref{thm:CompetitiveBound} into \eqref{eqn:exp_regret}, we get the following upper bound on expected regret.

\begin{coro}[Upper Bound on Expected Regret]
\label{thm:upper_bnd_exp_regret}
The expected cumulative regret of the C-UCB and C-TS algorithms is upper bounded as
\begin{align}
\E{\regret(\totalPulls)} &\leq \sum_{\arm \in \setofArmsC \setminus \{k^*\}} \Delta_\arm  U^{(c)}_\arm(\totalPulls) + \sum_{\arm' \in \{ 1, \ldots , \numArms \} \setminus \{ \setofArmsC \}
}\Delta_{\arm'} U^{(nc)}_{\arm'}(\totalPulls) , \label{eqn:upper_bnd_exp_regret}\\
&= (C-1) \cdot \OO(\log \totalPulls) + \OO(1), \label{eqn:upper_bnd_exp_regret_order}
\end{align} 
where $\mathcal{C} \subseteq \{ 1, \ldots , \numArms \}$ is set of competitive arms with cardinality $C$, $U^{(c)}_\arm (\totalPulls)$ is the upper bound on $\E{\pulls_\arm(\totalPulls)}$ for competitive arms given in \eqref{thm:CompetitiveBound}, and $U^{(nc)}_\arm(\totalPulls)$ is the upper bound for non-competitive arms given in \eqref{thm:NonCompetitiveBound}.
\end{coro}
\vspace{0.05cm}

\subsection{Proof Sketch}

We now present an outline of our regret analysis of C-UCB. A key strength of our analysis is that it can be extended very easily to any \textsc{C-BANDIT} algorithm. The results independent of last step in the algorithm are presented in Appendix B, while the rigorous regret upper bounds for C-UCB is presented in Appendix D. We also present a regret analysis for C-TS in a scenario where $K = 2$, and TS is employed with Beta priors in Appendix E.

There are three key components to prove the result in \Cref{thm:NonCompetitiveBound} and \Cref{thm:CompetitiveBound}. The first two components hold independent of which bandit algorithm (UCB/TS/KL-UCB) is used for selecting the arm from the set of competitive arms, which makes our analysis easy to extend to any \textsc{C-BANDIT} algorithm. The third step is specific to the last step in C-BANDIT algorithm. We analyse the third component for C-UCB to provide its rigorous regret results.

\noindent
\textbf{i) Probability of optimal arm being identified as empirically non-competitive at round $t$ (denoted by $\Pr(E_1(t))$) is small.} In particular, we show that  $$\Pr(E_1(t)) \leq 2Kt \exp\left(-\frac{t\Delta_{\text{min}}^2}{2K}\right).$$ This ensures that the optimal arm is identified as empirically non-competitive only $\OO(1)$ times. We show that the number of times a competitive arm is pulled is bounded as 
\begin{equation}
\E{n_k(T)} \leq \sum_{t = 1}^T \Pr(E_1(t)) + \Pr(E^c_1(t), k_t = k, I_{k,t-1} > I_{k^*,t-1}).
\end{equation}
The first term sums to a constant, while the second term is upper bounded by the number of times UCB pulls the sub-optimal arm $k$. Due to this the upper bound on the number of pulls of competitive arm by C-UCB / C-TS is only an additive constant more than the upper bound on the number of pulls for an arm by UCB / TS algorithms and hence we have same pre-log constants for the upper bound on the pulls of competitive arms.   

\noindent
\textbf{ii) Probability of identifying a non-competitive arm as empirically competitive jointly with optimal arm being pulled more than $\frac{t}{K}$ times is small.} Notice that the first two steps of our algorithm involve identifying the set of arms $\mathcal{S}_t$ that have been pulled at least $\frac{t}{K}$ times, and eliminating arms which are empirically non-competitive with respect to the set $\mathcal{S}_t$ for round $t$. We show that the joint event that arm $k^* \in \mathcal{S}_t$ and a non-competitive arm $k$ is identified as empirically non-competitive is small. Formally, 
\begin{equation}
\Pr\left(k_{t+1} = k, n_{k^*}(t) \geq \frac{t}{K}\right) \leq t\exp\left(-\frac{t \tilde{\Delta}_{k,k^*}}{2K}\right). \label{eqn:eq2}
\end{equation}
This occurs because upon obtaining a \textit{large} number of samples of arm $k^*$, expected reward of arm $k^*$ (i.e., $\mu_{k^*}$) and expected pseudo-reward of arm $k$ with respect to arm $k^*$ (i.e., $\phi_{k,k^*}$) can be estimated \textit{fairly accurately}. Since the pseudo-gap of arm $k$ is positive (i.e., $\mu_{k^*} > \phi_{k,k^*}$), the probability that arm $k$ is identified as empirically competitive is small. 
An implication of \eqref{eqn:eq2} is that the expected number of times a non-competitive arm is identified as empirically competitive jointly with the optimal arm having at least $\frac{t}{K}$ pulls at round $t$ is bounded above by a constant. 

\noindent
iii) \textbf{Probability that a sub-optimal arm is pulled more than $t/K$ times at round $t$ is small.} Formally, we show that for C-UCB, we have
\begin{equation}
    \Pr\left(n_k(t) \geq \frac{t}{K}\right) \leq (2K + 2) \left(\frac{t}{K}\right)^{-2} \quad \forall t > Kt_0, k \neq k^*
\label{eq:eq3}
\end{equation}
This component of our analysis is specific to the classical bandit algorithm used in \textsc{C-BANDIT}. Intuitively, a result of this kind should hold for any \textit{good performing} classical multi-armed bandit algorithm.  
We reach the result of \eqref{eq:eq3} in C-UCB by showing that 
\begin{equation}
    \Pr\left(k_{t+1} = k, n_k(t) > \frac{t}{2K}\right) \leq t^{-3} \quad \forall t > t_0, k \neq k^*
\label{eq:eq4}
\end{equation}
The probability of selecting a sub-optimal arm $k$ after it has been pulled \textit{significantly} many times is small as with more number of pulls, the exploration component in UCB index of arm $k$ becomes small, and consequently it is likely to be smaller than the UCB index of optimal arm $k^*$ (as it has larger empirical mean reward or has been pulled fewer number of times). Our analysis in \Cref{lem:suboptimalNotPulled} shows how the result in \eqref{eq:eq4} can be translated to obtain \eqref{eq:eq3} (this translation is again not dependent on which bandit algorithm is used in \textsc{C-BANDIT}). 

We show that the expected number of pulls of a non-competitive arm $k$ can be bounded as
\begin{equation}
    \E{n_k(T)} \leq \sum_{t = 1}^{T} \Pr\left(k_{t+1} = k, k^* = \argmax_{k} n_k(t)\right) + \Pr\left(k^* \neq \argmax_{k} n_k(t) \right)
\label{eqn:noncompNum}
\end{equation}
The first term in \eqref{eqn:noncompNum} is $\OO(1)$ due to \eqref{eqn:eq2} and the second term is $\OO(1)$ due to \eqref{eq:eq3}. Refer to Appendix D for rigorous regret analysis of C-UCB.

\subsection{Discussion on Regret Bounds}
\normalfont

\noindent
\textbf{Competitive Arms.} Recall than an arm is said to be competitive if $\mu_{k^*}$ (i.e., expected reward from arm $k^*$) $> \E{\phi_{k,k^*}} = \E{\tilde{\mathbb{E}}[R_{k'} | R_k]}$. Since the distribution of reward of each arm is unknown, initially the Algorithm does not know which arm is \textit{competitive} and which arm is \textit{non-competitive}.

\vspace{0.1cm}
\noindent
\textbf{Reduction in effective number of arms.} Interestingly, our result from \Cref{thm:NonCompetitiveBound} shows that the C-UCB algorithm, that operates in a sequential fashion, makes sure that \textit{non-competitive} arms are pulled only $\OO(1)$ times. Due to this, only the competitive arms are pulled $\OO(\log T)$ times. Moreover, the pre-log terms in the upper bound of UCB and C-UCB for these arms is the same. In this sense, our \textsc{C-BANDIT} approach reduces a $K$-armed bandit problem to a $C$-armed bandit problem. Effectively only $C-1 \leq K-1$ arms are pulled $\OO(\log T)$ times, while other arms are stopped being pulled after a finite time.

\begin{table}[]
\centering
\begin{tabular}{|l|l|l|l|}
\hline
\textbf{$p_1(r)$} & \textbf{r} & \textbf{$s_{2,1}(r)$} & \textbf{$s_{3,1}(r)$} \\ \hline
0.2               & \textbf{0} & 0.7                     & 2                     \\ \hline
0.2               & \textbf{1} & 0.8                     & 1.2                     \\ \hline
0.6               & \textbf{2} & 2                     & 1                     \\ \hline
\end{tabular}
\caption{Suppose Arm 1 is optimal and its unknown probability distribution is $(0.2,0.2,0.6)$, then $\mu_1 = 1.4$, while $\phi_{2,1} = 1.5$ and $\phi_{3,1} = 1.2$. Due to this Arm 2 is Competitive while Arm 3 is non-competitive}
\label{tab:comp}
\vspace{-0.2cm}
\end{table}

\noindent
Depending on the joint probability distribution, different arms can be optimal, competitive or non-competitive. \Cref{tab:comp} shows a case where arm 1 is optimal and the reward distribution of arm $1$ is $(0.2,0.2,0.6)$, which leads to $\mu_1 = 1.4 > \phi_{3,1} = 1.2$ and $\mu_1 = 1.4 < \phi_{2,1} = 1.5$. Due to this Arm 2 is competitive while Arm 3 is non-competitive.

\vspace{0.1cm}
\noindent
\textbf{Achieving Bounded Regret.}
If the set of competitive arms $\setofArmsC$ is a singleton set containing only the optimal arm (i.e., the number of competitive arms $C =1$), then our algorithm will lead to (see \eqref{eqn:upper_bnd_exp_regret_order}) an expected regret of $\OO(1)$, instead of the typical $\OO(\log T)$ regret scaling in classic multi-armed bandits. One such scenarion in which this can happen is if pseudo-rewards $s_{k,k^*}$ of all arms with respect to optimal arm $k^*$ match the conditional expectation of arm $k$. Formally, if $s_{k,k^*} = \E{R_k | R_{k^*}} \forall{k}$, then $\E{s_{k,k^*}} = \E{R_k} = \mu_k < \mu_{k^*}$. Due to this, all sub-optimal arms are non-competitive and our algorithms achieve only $\OO(1)$ regret. We now evaluate a lower bound result for a special case of our model, where rewards are correlated through a latent random variable $X$ as described in \Cref{subsec:specialCase}. 

We present a lower bound on the expected regret for the model described in \Cref{subsec:specialCase}. Intuitively, if an arm $\ell$ is \textit{competitive}, it can not be deemed sub-optimal by only pulling the optimal arm $\arm^*$ infinitely many times. This indicates that exploration is necessary for competitive sub-optimal arms. The proof of this bound closely follows that of the 2-armed classical bandit problem \cite{lai1985asymptotically}; i.e., we construct a new bandit instance under which a previously sub-optimal arm becomes optimal without affecting reward distribution of any other arm.

\begin{thm}[Lower Bound for Correlated MAB with latent random source]
\label{thm:lower_bnd_exp_regret}

For any algorithm that achieves a sub-polynomial regret, the expected cumulative regret for the model described in \Cref{subsec:specialCase} is lower bounded as

\begin{equation}
    \lim_{\totalPulls \rightarrow \infty}\inf \frac{\E{\regret (\totalPulls)}}{\log (\totalPulls)} \geq 
    \begin{cases} 
    \max_{\arm \in \setofArmsC}\frac{\Delta_k}{D(f_{R_k} || f_{\tilde{R}_k})} \quad &\text{if } C > 1\\
    0 \quad &\text{if } C = 1. 
    \end{cases}
\vspace{-0.2cm}
\end{equation}
\label{thm:lowerBound}
 \end{thm}
\vspace{-0.2cm} 
Here $f_{R_k}$ is the reward distribution of arm $k$, which is linked with $f_X$ since $R_k = Y_k(X)$. The term $f_{\tilde{R}_{k}}$ represents the reward distribution of arm $k$ in the new bandit instance where arm $k$ becomes optimal and distribution $f_{R_{k^{*}}}$ is unaffected. The divergence term represents "the amount of distortion needed in reward distribution of arm $k$ to make it better than arm $k^*$", and hence captures the problem difficulty in the lower bound expression. 

\vspace{0.1cm}
\noindent
\textbf{Bounded regret whenever possible for the special case of \Cref{subsec:specialCase}.}
From \Cref{thm:upper_bnd_exp_regret}, we see that whenever $C > 1$, our proposed algorithm achieves $\OO(\log \totalPulls)$ regret matching the lower bound given in Theorem \ref{thm:lowerBound} order-wise. Also, when $C = 1$, our algorithm achieves $\OO(1)$ regret. Thus, our algorithm achieves bounded regret whenever possible, i.e., when $C = 1$ for the model described in \Cref{subsec:specialCase}. In the general problem setting, a lower bound $\Omega(\log T)$ exists whenever it is possible to change the joint distribution of rewards such that the marginal distribution of optimal arm $k^*$ is unaffected and pseudo-rewards $s_{\ell,k}(r)$ still remain an upper bound on $\E{R_{\ell} | R_k = r}$ under the new joint probability distribution. In general, this can happen even if $C = 1$, we discuss one such scenario in the Appendix F and explain the challenges that need to come from the algorithmic side to meet the lower bound.
\section{Simulations}
\label{sec:simulation}
We now present the empirical performance of proposed algorithms. For all the results presented in this section, we compare the performance of all algorithms on the same reward realizations and plot the cumulative regret averaged over 100 independent trials. The shaded area represents error bars with one standard deviation. We set $\beta = 1$ for all TS and C-TS plots.

\subsection{Simulations with known pseudo-rewards}
\begin{table}[t]
\centering
\begin{tabular}{|l|l|l|l|l|}
\cline{1-2} \cline{4-5}
\textbf{r} & \textbf{$s_{2,1}(r)$} &  & \textbf{r} & \textbf{$s_{1,2}(r)$} \\ \cline{1-2} \cline{4-5} 
\textbf{0} & 0.7                   &  & \textbf{0} & 0.8                     \\ \cline{1-2} \cline{4-5} 
\textbf{1} & 0.4                   &  & \textbf{1} & 0.5                     \\ \cline{1-2} \cline{4-5} 
\end{tabular}
\\ \vspace{2mm}
\parbox{.45\linewidth}{
\centering
\begin{tabular}{|l|l|l|}
\hline
    \textbf{(a)}    & $R_1 = 0$ & $R_1 = 1$ \\ \hline
$R_2 = 0$ & 0.2       & 0.4       \\ \hline
$R_2 = 1$ & 0.2       & 0.2       \\ \hline
\end{tabular}
}
\hfill
\parbox{.45\linewidth}{
\centering
\begin{tabular}{|l|l|l|}
\hline
    \textbf{(b)}    & $R_1 = 0$ & $R_1 = 1$ \\ \hline
$R_2 = 0$ & 0.2       & 0.3       \\ \hline
$R_2 = 1$ & 0.4       & 0.1       \\ \hline
\end{tabular}
}
\caption{The top row shows the pseudo-rewards of arms 1 and 2, i.e., upper bounds on the conditional expected rewards (which are known to the player). The bottom row depicts two possible joint probability distribution (unknown to the player). Under distribution (a), Arm 1 is optimal whereas Arm 2 is optimal under distribution (b).}
\label{tab:pseudoBin2}
\vspace{-0.2cm}
\end{table}
\begin{figure}[t]
    \centering
    \includegraphics[width = 0.8\textwidth]{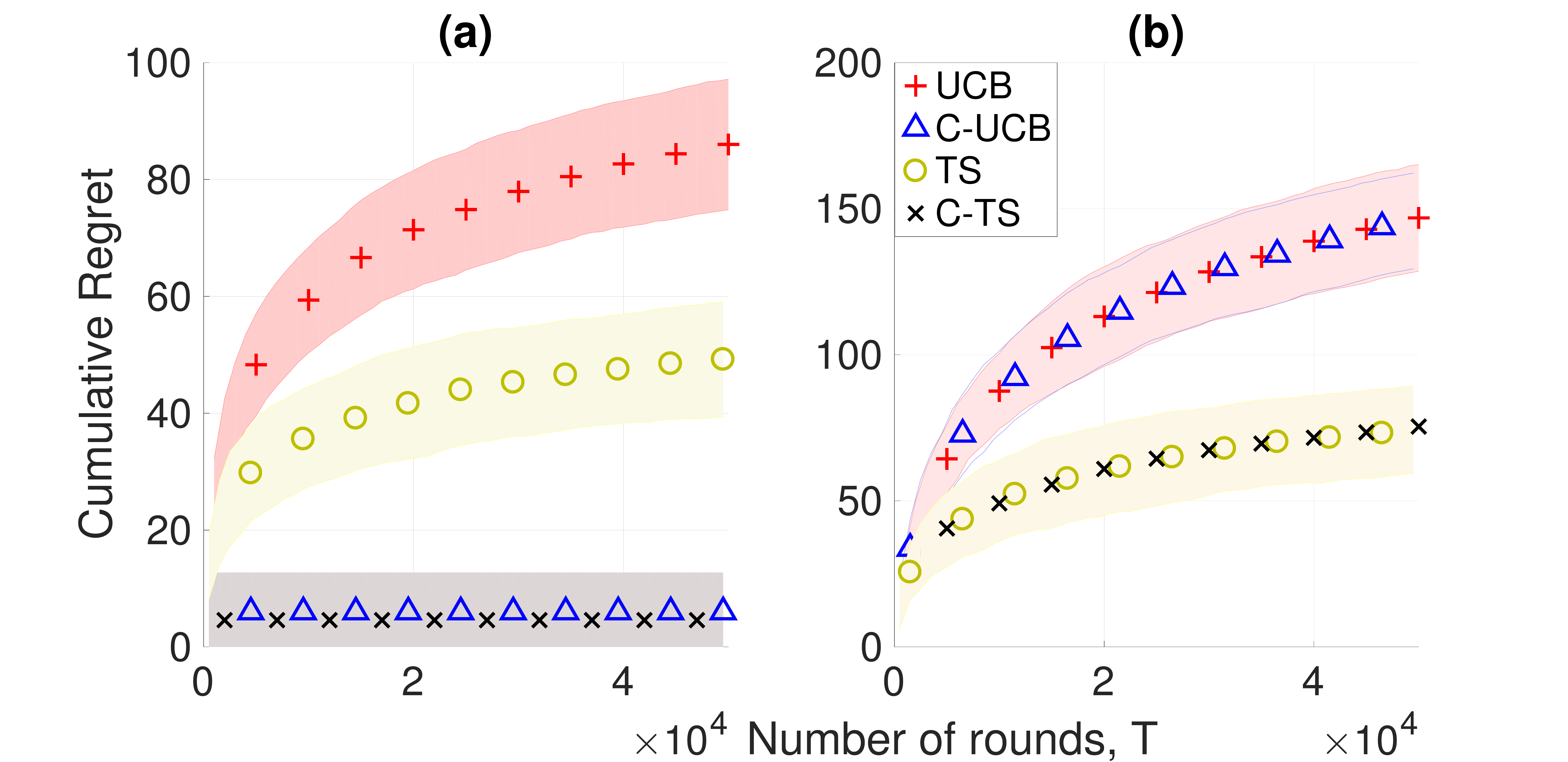}
    \caption{Cumulative regret for UCB, C-UCB, TS and C-TS corresponding to the problem shown in \Cref{tab:pseudoBin2}. For the setting (a) in \Cref{tab:pseudoBin2}, Arm 1 is optimal and Arm 2 is non-competitive, in setting (b) of \Cref{tab:pseudoBin2} Arm 2 is optimal while Arm 1 is competitive.}
    \label{fig:simulationWPseudoReward}
    \vspace{-0.3cm}
\end{figure}

Consider the example shown in \Cref{tab:pseudoBin}, with the top row showing the pseudo-rewards, which are known to the player, and the bottom row showing two possible joint probability distributions (a) and (b), which are unknown to the player. We show the simulation result of our proposed algorithms C-UCB, C-TS against UCB, TS in \Cref{fig:simulationWPseudoReward} for the setting considered in \Cref{tab:pseudoBin}.

\noindent
\textbf{Case (a): Bounded regret}. For the probability distribution (a), notice that Arm 1 is optimal with $\mu_1 = 0.6, \mu_2 = 0.4$. Moreover, $\phi_{2,1} = 0.4 \times 0.7 + 0.6 \times 0.4 = 0.52$. Since $\phi_{2,1} < \mu_1$, Arm 2 is non-competitive. Hence, in \Cref{fig:simulationWPseudoReward}(a), we see that our proposed C-UCB and C-TS Algorithms achieve bounded regret, whereas UCB, TS show logarithmic regret.

\noindent
\textbf{Case (b): All competitive arms}. For the probability distribution (b), Arm 2 is optimal with $\mu_2 = 0.5$ and $\mu_1 = 0.4$. The expected pseudo-reward of arm 1 w.r.t to arm 2 in this case is $\phi_{1,2} = 0.8 \times 0.5 + 0.5 \times 0.5 = 0.65$. Since $\phi_{1,2} > \mu_2$, the sub-optimal arm (i.e., Arm 1) is competitive and hence C-UCB and C-TS also end up exploring Arm 1. Due to this we see that C-UCB, C-TS achieve a regret similar to UCB, TS in \Cref{fig:simulationWPseudoReward}(b). C-TS has empirically smaller regret than C-UCB as Thompson Sampling performs better empirically than the UCB algorithm. The design of our C-Bandit approach allows the use of any other bandit algorithm in the last step, e.g., KL-UCB.

\subsection{Simulations for the latent random source model in \Cref{subsec:specialCase}}
\begin{figure}[t]
    \centering
    \includegraphics[width=0.6\textwidth]{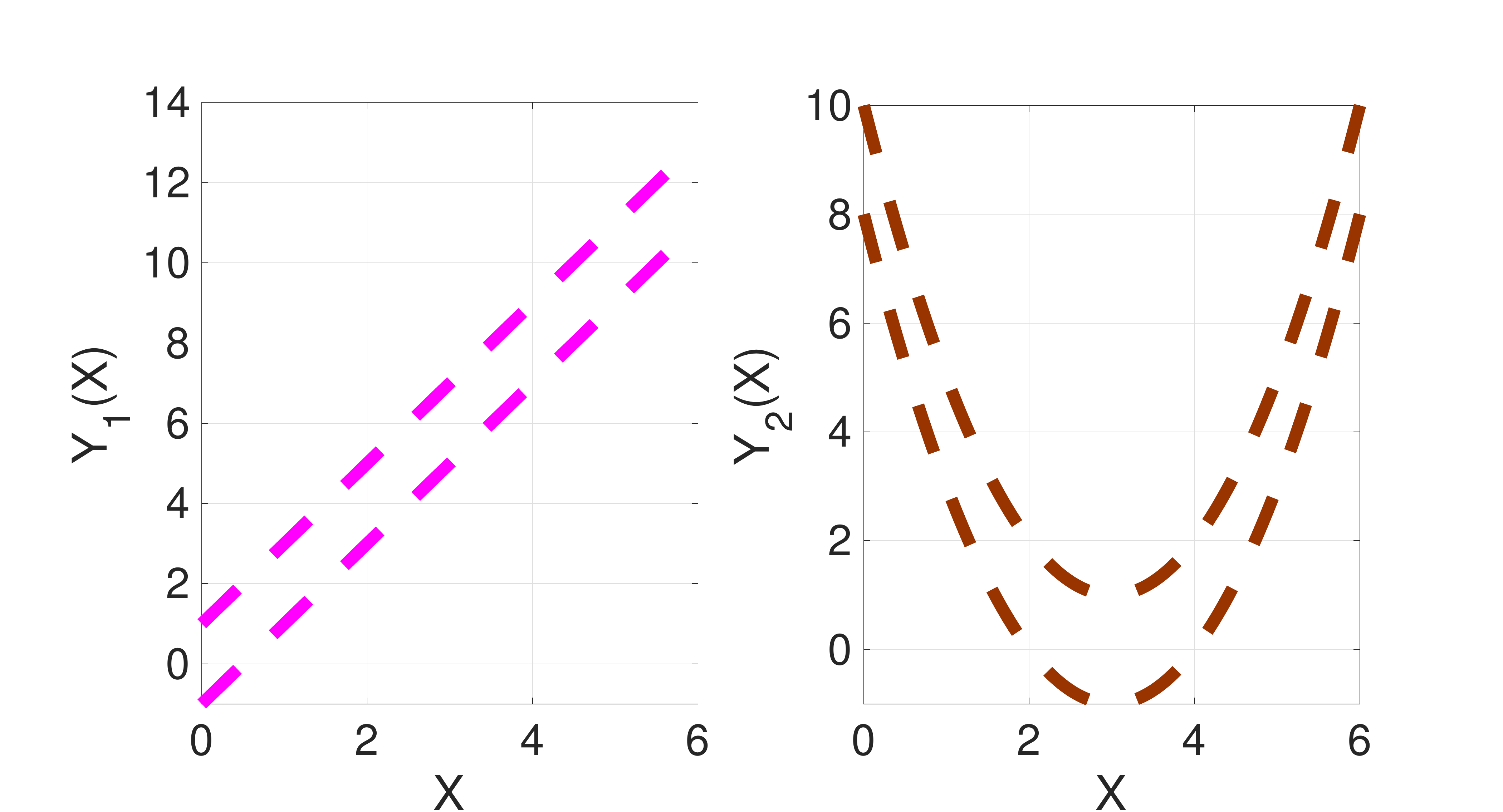}
    \caption{Rewards corresponding to the two arms are correlated through a random variable $X$ lying in $(0,6)$. The lines represent lower and upper bounds on reward of Arms 1 ,$Y_1(X)$, and 2, $Y_2(X)$, given the realization of random variable $X$.}
    \label{fig:illustrationLBUB}
    \vspace{-0.2cm}
\end{figure}

\begin{figure}[t]
    \centering
    \includegraphics[width = 0.8\textwidth]{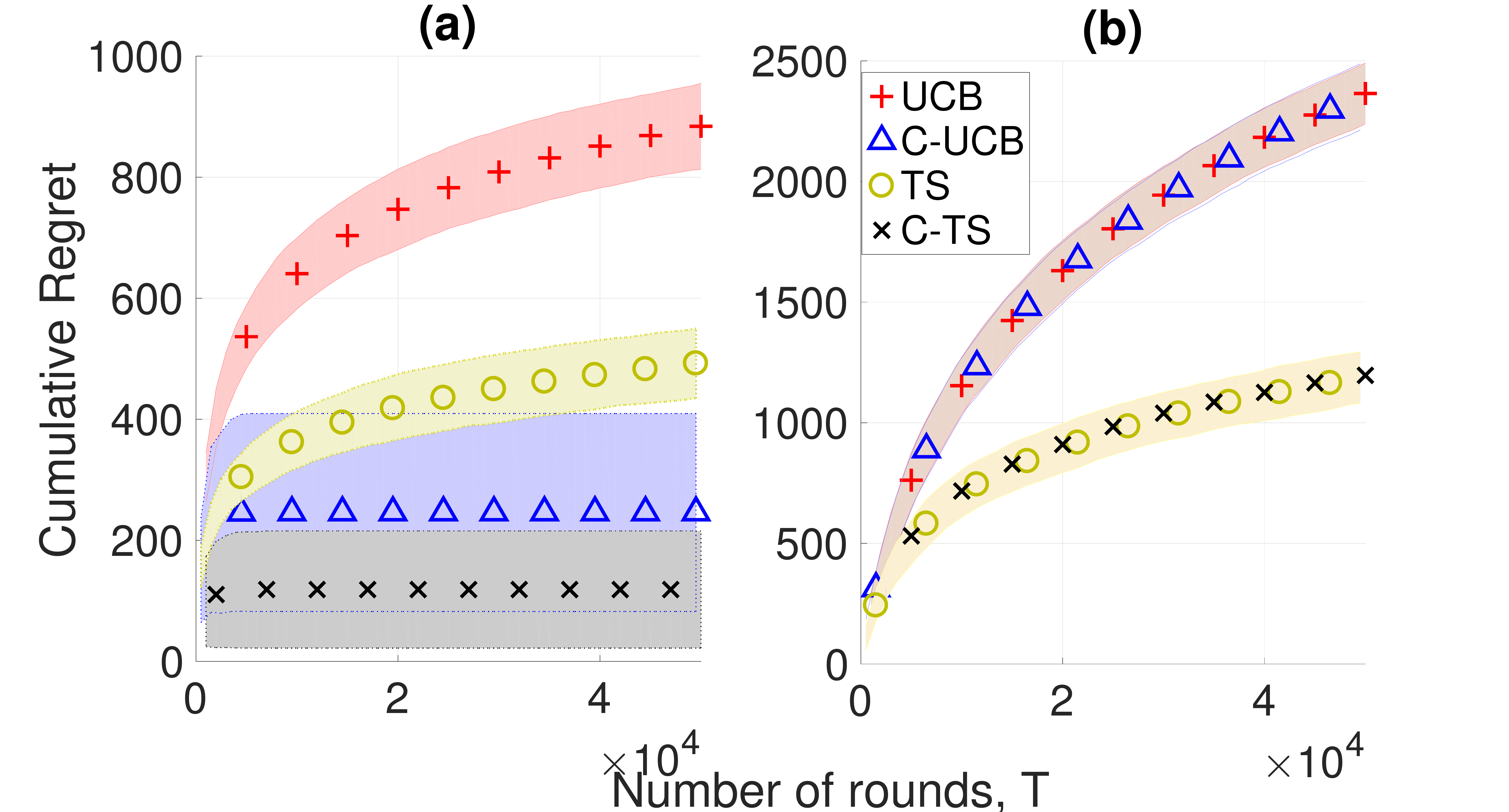}
    \caption{Simulation results for the example shown in \Cref{fig:illustrationLBUB}. In (a), $X \sim \text{Beta}(1,1)$ and in (b) $X \sim \text{Beta}(1.5,5)$. In case (a), Arm 1 is optimal while Arm 2 is non-competitive (C = 1), due to which we see that C-UCB and C-TS obtain bounded regret. Arm 2 is optimal for the distribution in (b) and Arm 1 is competitive, due to which $C=2$ and we see that C-UCB and C-TS attain a performance similar to UCB and TS.}
    \label{fig:upperlowerbdSim}
    \vspace{-0.3cm}
\end{figure}
We now show the performance of C-UCB and C-TS against UCB, TS for the model considered in \Cref{subsec:specialCase}, where rewards corresponding to different arms are correlated through a latent random variable $X$. We consider a setting where reward obtained from Arm 1, given a realization $x$ of $X$, is bounded between $2x - 1$ and $2x + 1$, i.e., $2X - 1 \leq Y_1(X) \leq 2X + 1$. Similarly, conditional reward of Arm 2 is, $(3-X)^2 - 1 \leq Y_2(X) \leq (3 - X)^2 + 1$. \Cref{fig:illustrationLBUB} demonstrates these upper and lower bounds on $Y_k(X)$. We run C-UCB, C-TS, TS and UCB for this setting for two different distributions of $X$. For the simulations, we set the conditional reward of both the arms to be distributed uniformly between the upper and lower bounds, however this information is not known to the Algorithms. 

\noindent
\textbf{Case (a): $X \sim \text{Beta}(1,1)$}. When $X$ is distributed as $X \sim \text{Beta}(1,1)$, Arm 1 is optimal while Arm 2 is non-competitive. Due to this, we observe that C-UCB and C-TS achieve bounded regret in \Cref{fig:upperlowerbdSim}(a).

\noindent
\textbf{Case (b): $X \sim \text{Beta}(1.5,5)$}. In the scenario where $X$ has the distribution $\text{Beta}(1.5,5)$, Arm 2 is optimal while Arm 1 is competitive. Due to this, C-UCB and C-TS do not stop exploring Arm 1 in finite time and we see the cumulative regret similar to UCB, TS in \Cref{fig:upperlowerbdSim}(b).

Our next simulation result considers a setting where the known upper and lower bounds on $Y_k(X)$ match and the reward $Y_k$ corresponding to a realization of $X$ is deterministic, i.e., $Y_k(X) = g_k(X)$. We show our simulation results for the reward functions described in \Cref{fig:sim_reward_funcs_cont} with three different distributions of $X$. Corresponding to $X \sim \text{Beta}(4,4)$, Arm 1 is optimal and Arms 2,3 are non-competitive leading to bounded regret for C-UCB, C-TS in \Cref{fig:teaserSim}(a). In setting (b), we consider $X \sim \text{Beta}(2,5)$ in which Arm 1 is optimal, Arm 2 is competitive and Arm 3 is non-competitive. Due to this, our proposed C-UCB and C-TS Algorithms stop pulling Arm 3 after some time and hence achieve significantly reduced regret relative to UCB in \Cref{fig:teaserSim}(b). For third scenario (c), we set $X \sim \text{Beta}(1,5)$, which makes Arm 3 optimal while Arms 1 and 2 are competitive. Hence, our algorithms explore both the sub-optimal arms and have a regret comparable to that of UCB, TS in \Cref{fig:teaserSim}(c). 

\begin{figure}[t]
    \centering
    \includegraphics[width=0.45\textwidth]{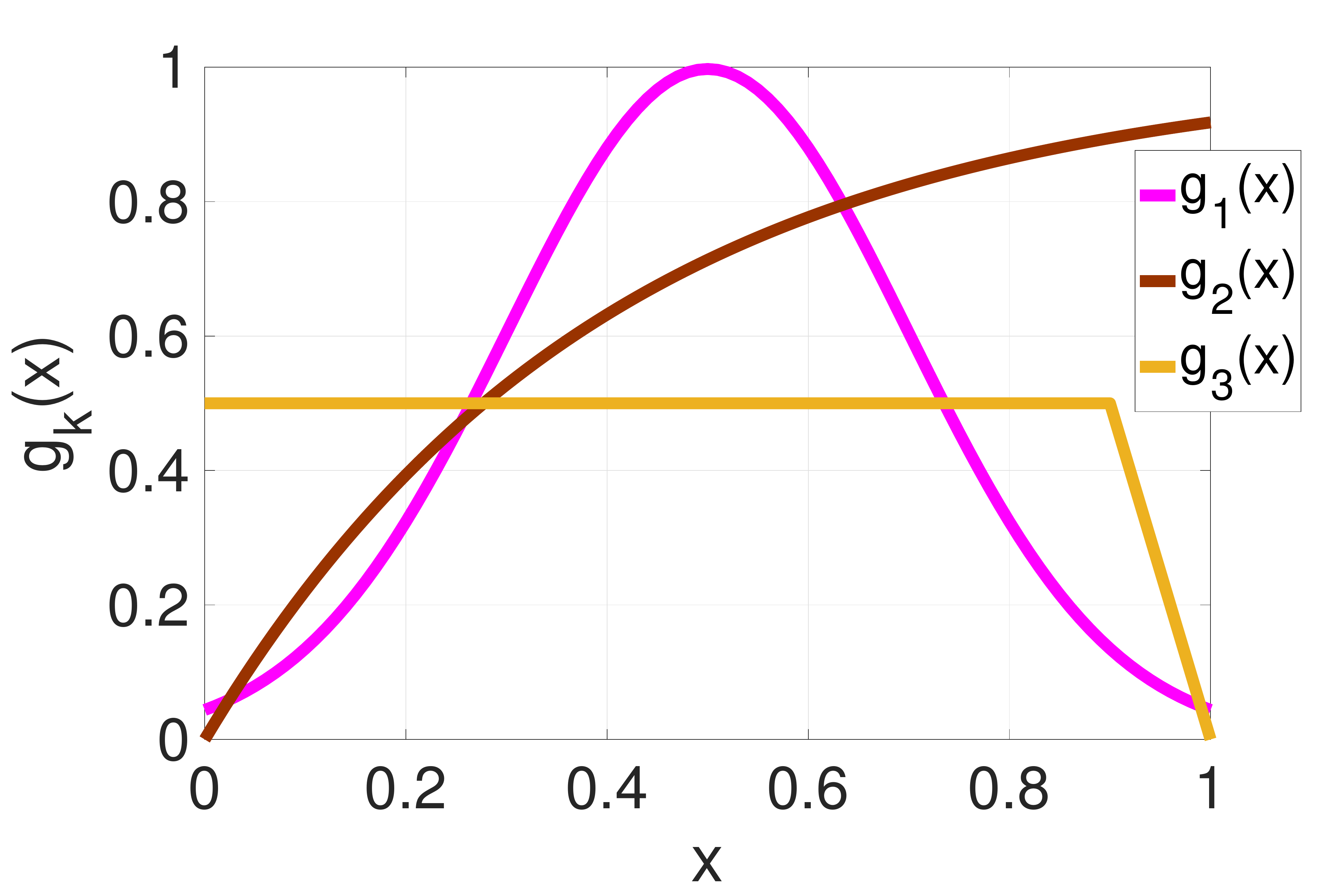}
    \caption{Reward Functions used for the simulation results presented in \Cref{fig:teaserSim}. The reward $g_k(X)$ is a function of a latent random variable $X$. For instance, when $X = 0.5$, reward from Arms 1,2 and 3 are $g_1(X) = 1$, $g_2(X) = 0.7135$ and $g_3(X) = 0.5$. }
    \label{fig:sim_reward_funcs_cont}
    \vspace{-0.2cm}
\end{figure}

\begin{figure}[t]
    \centering
    \includegraphics[width = 0.8\textwidth]{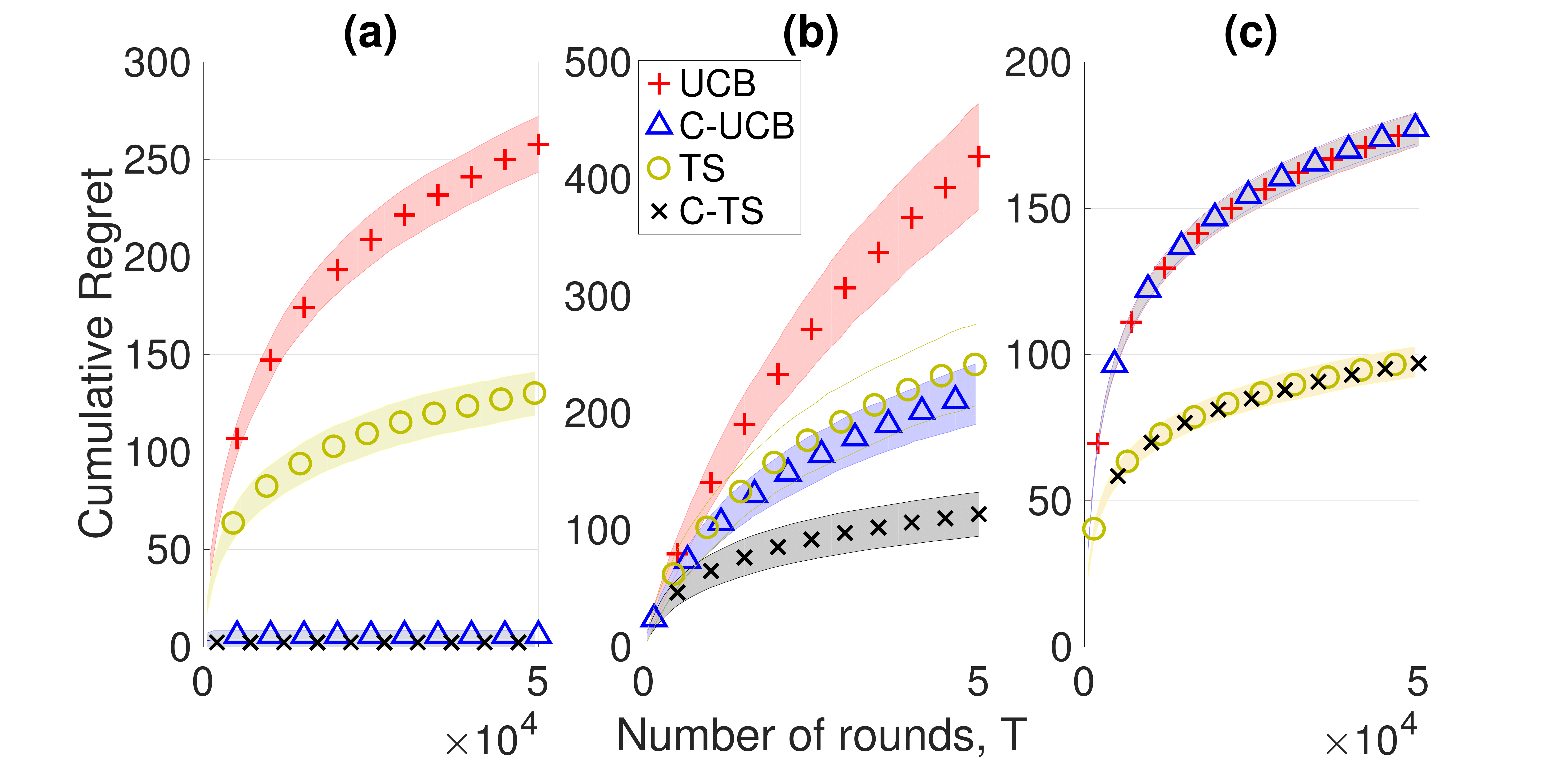}
    \caption{ \sl The cumulative regret of C-UCB and C-TS depend on $C$, the number of \emph{competitive} arms. The value of $C$ depends on the {\em unknown} joint probability distribution of rewards and is not known beforehand. We consider a setup where $C = 1$ in (a), $C = 2$ in (b) and $C = 3$ in (c). Our proposed algorithm pull only the $C-1$ competitive sub-optimal arms $\OO(\log T)$ times, as opposed to UCB, TS that pull all $K-1$ sub-optimal arms $\OO(\log T)$ times. Due to this, we see that our proposed algorithms achieve bounded regret when $C = 1$. When $C = 3$, our proposed algorithms perform as well as the UCB, TS algorithms.}
    \label{fig:teaserSim} 
    \vspace{-0.3cm}
\end{figure}

\section{Experiments}
\label{sec:experiments}

We now show the performance of our proposed algorithms in real-world settings. Through the use of \textsc{MovieLens} and \textsc{Goodreads} datasets, we demonstrate how the correlated MAB framework can be used in practical settings for  recommendation system applications. In such systems, it is possible to use the prior available data (from a certain population) to learn the correlation structure in the form of pseudo-rewards. When trying to design a campaign to maximize user engagement in a new unknown demographic, the learned correlation information in the form of pseudo-rewards can help significantly reduce the regret as we show from our results described below.

\subsection{Experiments on the \textsc{MovieLens} dataset}
The \textsc{MovieLens} dataset \cite{movielenspaper} contains a total of 1M ratings for a total of 3883 Movies rated by 6040 Users. Each movie is rated on a scale of 1-5 by the users. Moreover, each movie is associated with one (and in some cases, multiple) genres. For our experiments, of the possibly several genres associated with each movie, one is picked uniformly at random. To perform our experiments, we split the data into two parts, with the first half containing ratings of the users who  provided the most number of ratings. This half is used to learn the pseudo-reward entries, the other half is the test set which is used to evaluate the performance of the proposed algorithms. Doing such a split ensures that the rating distribution is different in the training and test data.

\begin{figure}[t]
    \centering
    \includegraphics[width = 0.6\textwidth]{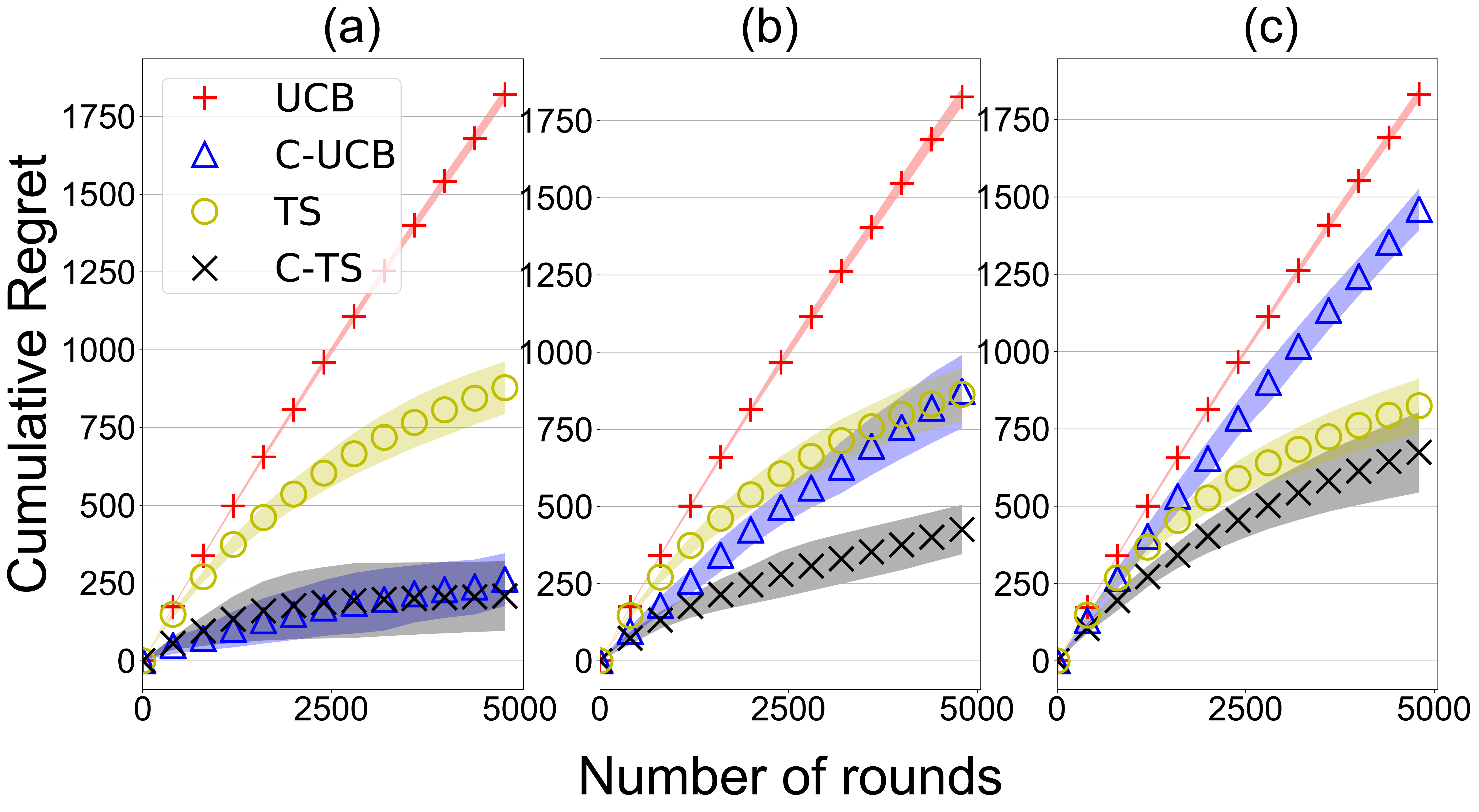}
    \caption{Cumulative regret for UCB, C-UCB, TS and C-TS for the application of recommending the best genre in the Movielens dataset, where $p$ fraction of the pseudo-entries are replaced with maximum reward $i.e., 5$. In $(a), p = 0.25$, for $(b), p = 0.50$ and $p = 0.7$ in $(c)$. The value of $C$ is $4,11$ and $13$ in $(a), (b)$ and $(c)$ respectively. As $C$ is smaller than $K$ (i.e., $18$) in each case, we see that C-UCB and C-TS outperform UCB and TS significantly.}
    \label{fig:genre_plot}
    \vspace{-0.3cm}
\end{figure}
\noindent
\textbf{Recommending the Best Genre.} In our first experiment, the goal is to provide the best genre recommendations to a population with unknown demographic. We use the training dataset to learn the pseudo-reward entries. The pseudo-reward entry $s_{\ell,k}(r)$ is evaluated by taking the empirical average of the ratings of genre $\ell$ that are rated by the users who rated genre $k$ as $r$. To capture the fact that  it might not be possible in practice to fill all pseudo-reward entries, we randomly remove $p$-fraction of the pseudo-reward entries. The removed pseudo-reward entries are replaced by the maximum possible rating, i.e., $5$ (as that gives a natural upper bound on the conditional mean reward). Using these pseudo-rewards,  we evaluate our proposed algorithms on the test data. Upon recommending a particular genre (arm), the rating (reward) is obtained by sampling one of the ratings for the chosen arm in the test data. Our experimental results for this setting are shown in \Cref{fig:genre_plot}, with $p = 0.25, 0.50$ and $0.70$ (i.e., fraction of pseudo-reward entries that are removed). We see that the proposed C-UCB and C-TS algorithms significantly outperform UCB and TS in all  three settings. For each of the three cases we also evaluate the value of $C$ (which is unknown to the algorithm), by always pulling the optimal arm and finding the size of empirically competitive set at $T = 5000$. The value of $C$ turned out to be $4,11$ and $13$ for $p = 0.25, 0.50$ and $0.70$. As $C < 18$ in each case, some of the 18 arms are stopped being pulled after some time and due to this, C-UCB and C-TS significantly outperform UCB and TS respectively. This shows that even when only a subset of the correlations are known, it is possible to exploit them to improve the performance of classical bandit algorithms.

\begin{figure}[t]
    \centering
    \includegraphics[width = 0.6\textwidth]{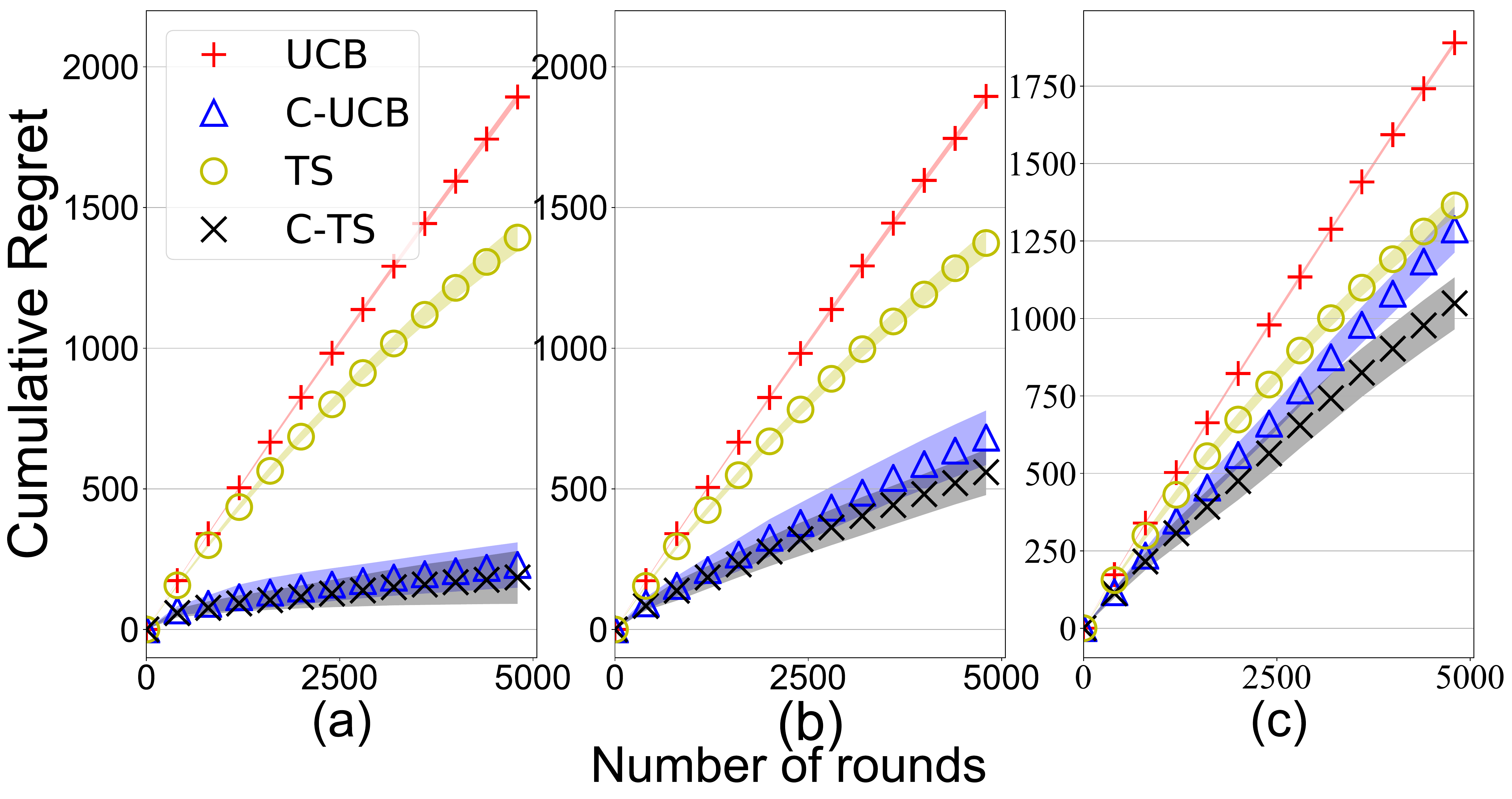}
    \caption{Cumulative regret of UCB, C-UCB, TS and C-TS for providing the best movie recommendations in the Movielens dataset. Each pseudo-reward entry is added by $0.1$ in (a), $0.4$ in (b) and $0.6$ in (c). The value of $C$ is $6,24$ and $39$ in $(a), (b)$ and $(c)$ respectively. As $C$ is smaller than $K$ (i.e., $50$) in each case, we see the superior performance of C-UCB, C-TS over UCB and TS.}
    \label{fig:movie_plot}
    \vspace{-0.3cm}
\end{figure}

\noindent
\textbf{Recommending the Best Movie.} We now consider the goal of providing the best movie recommendations to the population. To do so, we consider the 50 most rated movies in the dataset. containing 109,804 user-ratings given by 6,025 users.
In the testing phase, the goal is to recommend one of these 50 movies to each user. 
As was the case in previous experiment, we learn the pseudo-reward entries from the training data. Instead of using the learned pseudo-reward directly, we add a \textit{safety buffer} to each of the pseudo-reward entry; i.e., we set the pseudo-reward as the empirical conditional mean {\em plus} the {\sc safety buffer}. Adding a buffer will be needed in practice, as the conditional expectations learned from the training data are likely to have some noise and adding a safety buffer allows us to make sure that pseudo-rewards constitute an upper bound on the conditional expectations. Our experimental result in \Cref{fig:movie_plot} shows the performance of C-UCB and C-TS relative to UCB for this setting with safety buffer set to $0.1$ in \Cref{fig:movie_plot}(a), $0.4$ in \Cref{fig:movie_plot}(b) and $0.6$ in \Cref{fig:movie_plot}(c). In all three cases, even after addition of safety buffers, our proposed C-UCB and C-TS algorithms outperform the UCB algorithm.

\subsection{Experiments on the {\sc Goodreads} dataset}
The {\sc Goodreads} dataset \cite{wan2018item} contains the ratings for 1,561,465 books by a total of 808,749 users. Each rating is on a scale of 1-5. For our experiments, we only consider the poetry section and focus on the goal of providing best poetry recommendations to the whole population whose demographics is unknown. The poetry dataset has 36,182 different poems rated by 267,821 different users. We do the pre-processing of goodreads dataset in the same manner as that of the MovieLens dataset, by splitting the dataset into two halves, train and test. The train dataset contains the ratings of the users with most number of recommendations. 

\begin{figure}[t]
    \centering
    \includegraphics[width = 0.6\textwidth]{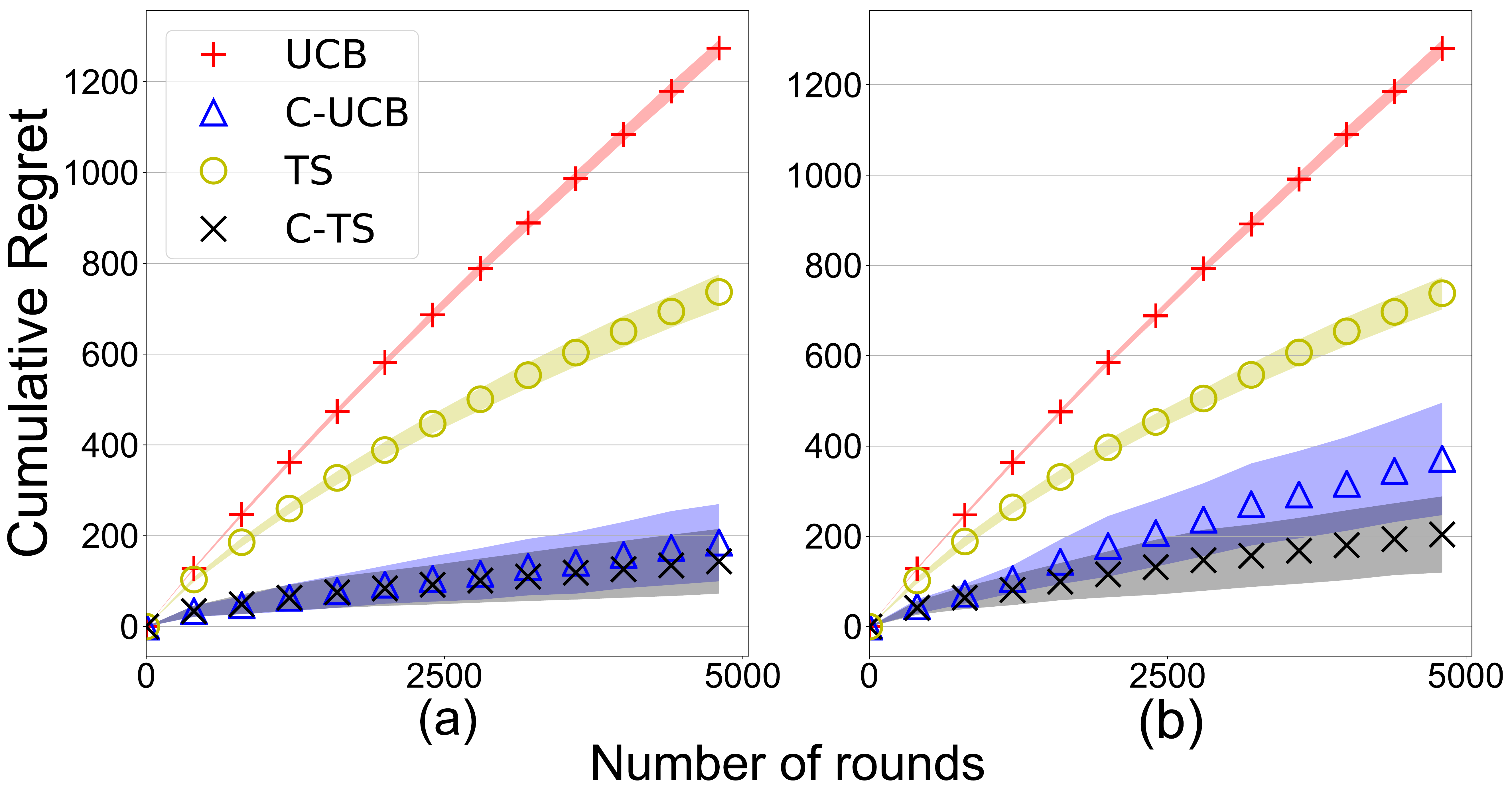}
    \caption{Cumulative regret of UCB, C-UCB, TS and C-TS for providing best poetry book recommendation in the Goodreads dataset. Every pseudo-reward entry is added by $q$ and $p$ fraction of the pseudo-reward entries are removed, with (a) $p = 0.1, q = 0.1$ and (b) $p = 0.3, q = 0.1$. The value of $C$ is $8$ and $11$ in $(a)$ and $(b)$ respectively. As $C$ is much smaller than $K$ (i.e., $25$) in each case, we see that C-UCB and  C-TS outperform UCB and TS significantly.}
    \label{fig:book_plot}
    \vspace{-0.3cm}
\end{figure}
\noindent
\textbf{Recommending the best poetry book.} We consider the 25 most rated books in the dataset and use these as the set of arms to recommend in the testing phase. These 25 poems have 349,523 user-ratings given by 171,433 users. As with the {\sc MovieLens} dataset, the pseudo-reward entries are learned on the training data. In practical situations it might not be possible to obtain all pseudo-reward entries. Therefore, we randomly select $p$ fraction of the pseudo-reward entries and replace them with maximum possible reward (i.e. $5$). Among the remaining pseudo-reward entries we add a safety buffer of $q$ to each entry. Our result in \Cref{fig:book_plot} shows the performance of C-UCB and C-TS relative to UCB and TS in two scenarios. In scenario (a), $10\%$ of the pseudo-reward entries are replaced by $5$ and remaining are padded with a safety buffer of $0.1$. For case (b), $30\%$ entries are replaced by $5$ and safety buffer is $0.1$. Under both  cases, our proposed C-UCB and C-TS algorithms are able to outperform UCB and TS significantly.

\newadd{
\subsection{Pseudo-rewards learned on a smaller dataset}
In our previous set of experiments, half of the dataset was used to learn the pseudo-reward entries. We did additional experiments in a setup where only $10\%$ of the data was used for learning the pseudo-reward entries and tested our algorithms on the remaining dataset. On doing so, we observed that C-UCB and C-TS were still able to outperform UCB and TS in most of our experimental setups. One setting in which the performance of C-UCB was similar to that of UCB is in a scenario where each pseudo-reward entry was padded by $0.6$. As the padding was large, the C-UCB algorithm was not able to identify many arms as non-competitive, leading to a performance that is similar to that of UCB. In all other scenarios, we noted that C-UCB and C-TS significantly outperformed UCB and TS, suggesting that even when smaller dataset is used for learning pseudo-rewards, the C-UCB and C-TS can be quite effective. The results are presented in \Cref{fig:genre_plot_less}, \Cref{fig:movie_plot_less} and \Cref{fig:book_plot_less}.

\begin{figure}[t]
    \centering
    \includegraphics[width = 0.6\textwidth]{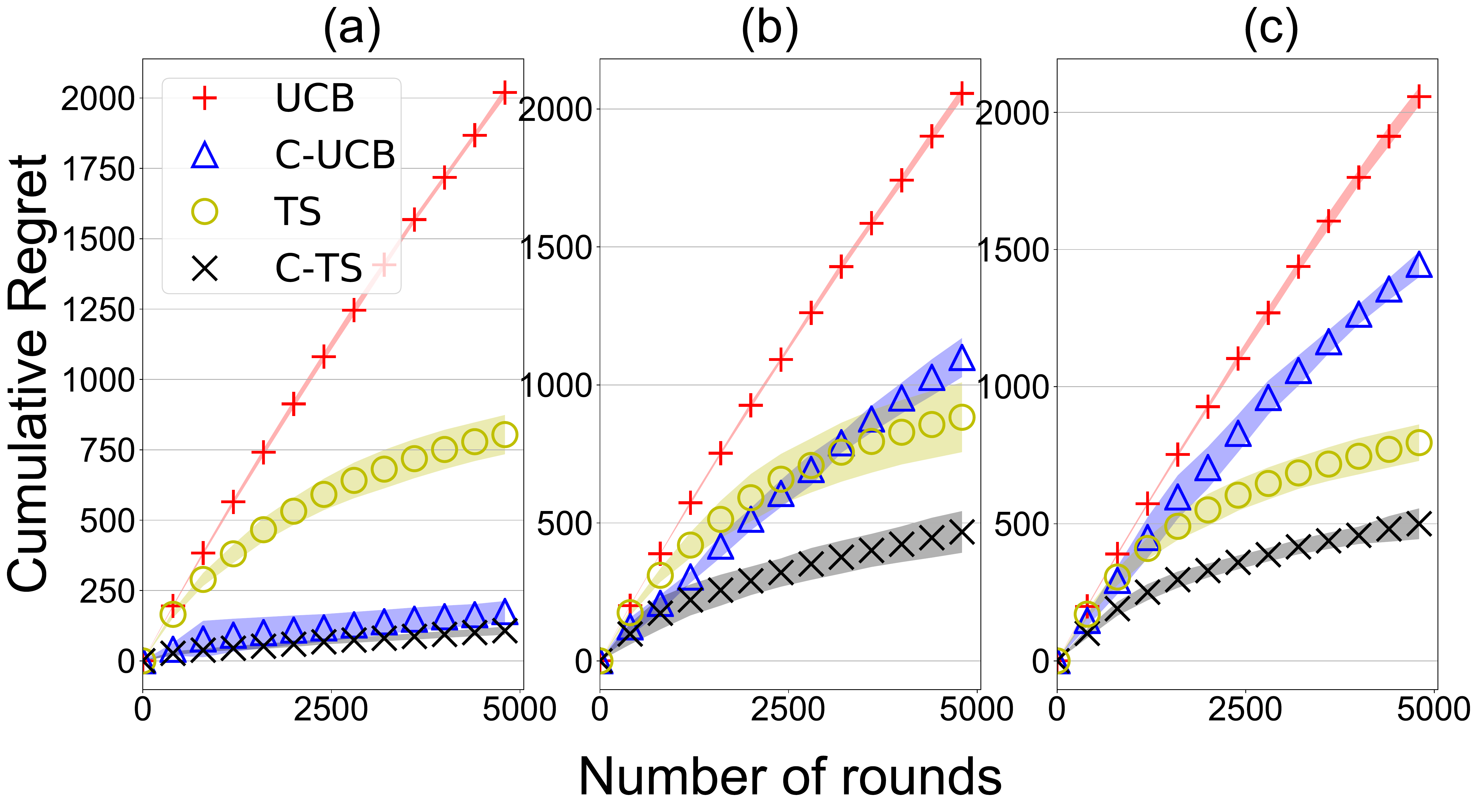}
    \caption{Cumulative regret for UCB, C-UCB, TS and C-TS for the application of recommending the best genre in the Movielens dataset, where $p$ fraction of the pseudo-entries are replaced with maximum reward $i.e., 5$. In $(a), p = 0.25$, for $(b), p = 0.50$ and $p = 0.7$ in $(c)$. We used $10\%$ of the dataset to learn the pseudo-reward entry and the algorithms are tested on the remaining dataset. The value of $C$ is $5,11$ and $15$ in $(a), (b)$ and $(c)$ respectively. As $C$ is smaller than $K$ (i.e., $18$) in each case, we see that C-UCB and C-TS outperform UCB and TS significantly. Note that the value of $C$ is larger in the case where only $10\%$ data is used for learning the pseudo-reward.}
    \label{fig:genre_plot_less}
    \vspace{-0.3cm}
\end{figure}

\begin{figure}[t]
    \centering
    \includegraphics[width = 0.6\textwidth]{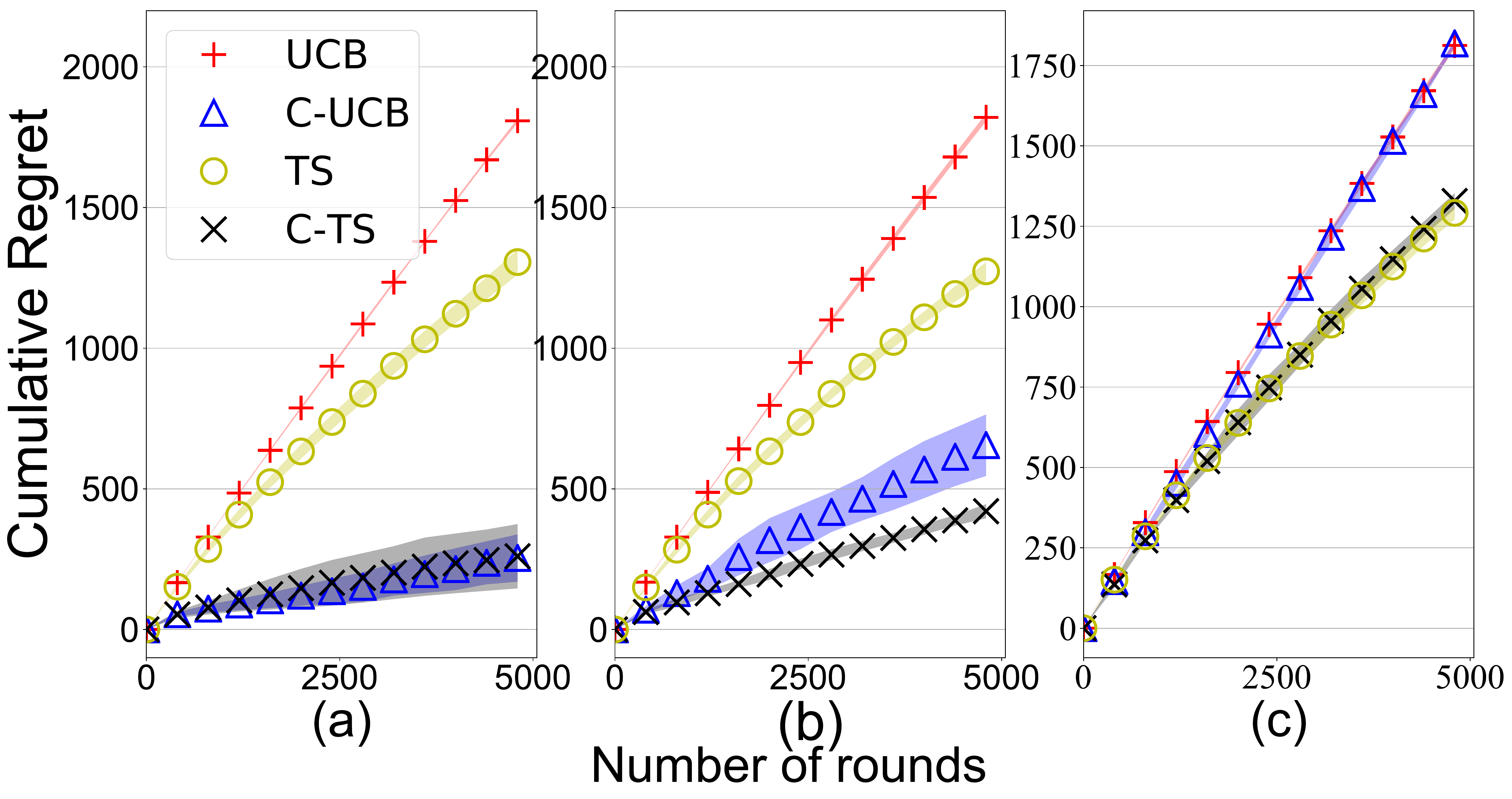}
    \caption{Cumulative regret of UCB, C-UCB, TS and C-TS for providing the best movie recommendations in the Movielens dataset. In this experiment $10\%$ of the dataset is used for learning the pseudo-reward entry and the algorithms are tested on the remaining dataset. Each pseudo-reward entry is added by $0.1$ in (a), $0.4$ in (b) and $0.6$ in (c). The value of $C$ is $14,29$ and $42$ in $(a), (b)$ and $(c)$ respectively. Note that the value of $C$ is larger in the case where only $10\%$ data is used for learning the pseudo-reward. As $C$ is still smaller than $K$ (i.e., $50$) in each case, we see the superior performance of C-UCB, C-TS over UCB and TS.}
    \label{fig:movie_plot_less}
    \vspace{-0.3cm}
\end{figure}

\begin{figure}[t]
    \centering
    \includegraphics[width = 0.6\textwidth]{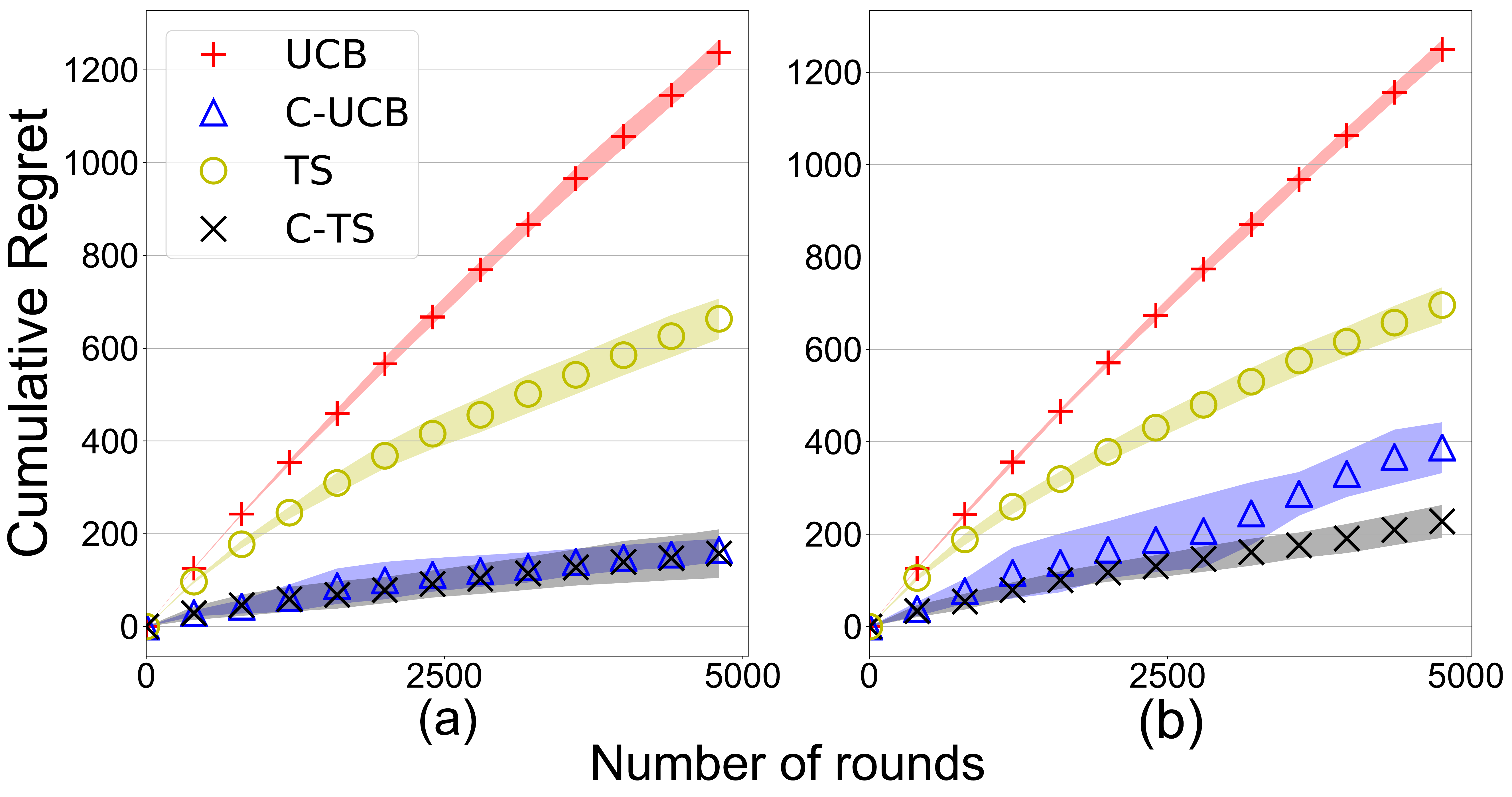}
    \caption{Cumulative regret of UCB, C-UCB, TS and C-TS for providing best poetry book recommendation in the Goodreads dataset. We used $10\%$ of the dataset to learn the pseudo-reward entry and the algorithms are tested on the remaining dataset. Every pseudo-reward entry is added by $q$ and $p$ fraction of the pseudo-reward entries are removed, with (a) $p = 0.1, q = 0.1$ and (b) $p = 0.3, q = 0.1$. The value of $C$ is $7$ and $12$ in $(a)$ and $(b)$ respectively. As $C$ is much smaller than $K$ (i.e., $25$) in each case, we see that C-UCB and  C-TS outperform UCB and TS significantly.}
    \label{fig:book_plot_less}
    \vspace{-0.3cm}
\end{figure}
}

\section{Conclusion}
\label{sec:conclusion}
This work presents a new correlated Multi-Armed bandit problem in which rewards obtained from different arms are correlated. We capture this correlation through the knowledge of \textit{pseudo-rewards}. These pseudo-rewards, which represent upper bound on conditional mean rewards, could be known in practice from either domain knowledge or learned from prior data. Using the knowledge of these pseudo-rewards, we the propose \textit{C-Bandit} algorithm which fundamentally generalizes any classical bandit algorithm to the correlated multi-armed bandit setting. A key strength of our paper is that it allows pseudo-rewards to be loose (in case there is not much prior information) and even then our \textit{C-Bandit} algorithms adapt and provide performance at least as good as that of classical bandit algorithms.

We provide a unified method to analyze the regret of C-Bandit algorithms. In particular, the analysis shows that C-UCB ends up pulling \emph{non-competitive} arms only $\OO(1)$ times; i.e., they stop pulling certain arms after a finite time $t$. Due to this, C-UCB pulls only $C-1 \leq K-1$ of the $K-1$ sub-optimal arms $\OO(\log T)$ times, as opposed to UCB that pulls {\em all} $K-1$ sub-optimal arms $\OO(\log T)$ times. In this sense, our C-Bandit algorithms reduce a $K$-armed bandit to a $C$-armed bandit problem. We present several cases where $C = 1$ for which C-UCB achieves bounded regret. For the special case when rewards are correlated through a latent random variable $X$, we provide a lower bound showing that bounded regret is possible only when $C = 1$; if $C > 1$, then $\OO(\log T)$ regret is not possible to avoid. Thus, our C-UCB algorithm achieves bounded regret whenever possible. Simulation results validate the theoretical findings and we perform experiments on {\sc Movielens} and {\sc Goodreads} datasets to demonstrate the applicability of our framework in the context of recommendation systems. The experiments on real-world datasets show that our C-UCB and C-TS algorithms significantly outperform the UCB and TS algorithms. 

There are several interesting open problems and extensions of this work, some of which we describe below.

\noindent
\textbf{Extension to light tailed and heavy tailed rewards} In this work, we assume that the rewards have a bounded support. The algorithm and analysis can be extended to settings with sub-gaussian rewards as well. In particular, in step 3 of the algorithm, one would play UCB/TS for sub-gaussian rewards. For instance, the UCB index in the scenario of sub-gaussian rewards can be redefined as $\hat{\mu}_k + \sqrt{\frac{2\sigma^2\log t}{n_k(t)}}$, where $\sigma$ is the sub-Gaussianity parameter of the reward distribution. Similar regret bounds will hold in this setting as well because the Hoeffding's inequality used in our regret analysis is valid for sub-Gaussian rewards as well. For heavy-tailed rewards, the Hoeffding's inequality is not valid. Due to which, one would need to construct confidence bounds for UCB in a different manner as done in \cite{bubeck2013bandits}. On doing so, the C-Bandit algorithm can be employed in heavy-tailed reward settings. However, the regret analysis may not extend directly as one would need to use modified concentration inequalities to obtain bounds on mean reward of arm $k$ as done in Lemma 1 of \cite{bubeck2013bandits}.

\noindent
\textbf{Designing better algorithms.}  While our proposed algorithms are order-optimal for the model in Section 2.3, they do not match the pre-log constants in the lower bound of the regret. It may be possible to design algorithms that have smaller pre-log constants in their regret upper bound. Further discussion along these lines is presented in Appendix F. A key advantage of our approach is that our algorithms are easy to implement and they incorporate the classical bandit algorithms nicely for the problem of correlated multi-armed bandits.

\noindent
\textbf{Best-Arm Identification.} We plan to study the problem of best-arm identification in the correlated multi-armed bandit setting, i.e., to identify the best arm with a confidence $1 - \delta$ in as few samples as possible. Since rewards are correlated with each other, we believe the sample complexity can be significantly improved relative to state of the art algorithms, such as LIL-UCB \cite{jamieson2014best, jamieson2014lil}, which are designed for classical multi-armed bandits. Another open direction is to improve the C-Bandit algorithm to make sure that it achieves bounded regret whenever possible in the general framework studied in this paper.

\section{Acknowledgments}
This work was partially supported by the NSF through grants CCF-1840860 and CCF-2007834, the Siebel Energy Institute, the Carnegie Bosch Institute, the Manufacturing Futures Initiative, and the CyLab IoT Initiative. In addition, Samarth Gupta was supported by the CyLab Presidential Fellowship and the David H. Barakat and LaVerne Owen-Barakat CIT Dean's Fellowship.

\bibliographystyle{ieeetr}
\bibliography{multi_armed_bandit}

\newpage
\onecolumn
\appendix
\section{Standard Results from Previous Works}
\begin{fact}[Hoeffding's inequality]
Let $Z_1, Z_2 \ldots Z_n$ be i.i.d random variables bounded between $[a, b]: a \leq Z_i \leq b$, then for any $\delta>0$, we have
$$\Pr\left(\left| \frac{\sum_{i = 1}^{n} Z_i}{n} - \E{Z_i} \right| \geq \delta\right) \leq \exp \left( \frac{-2 n \delta^2}{(b - a)^2}\right).$$ 
\end{fact}

\begin{lem}[Standard result used in bandit literature]
If $\hat{\mu}_{k,n_k(t)}$ denotes the empirical mean of arm $k$ by pulling arm $k$ $n_k(t)$ times through any algorithm and $\mu_k$ denotes the mean reward of arm $k$, then we have 
$$\Pr\left(\hat{\mu}_{k,n_k(t)} - \mu_k \geq \epsilon, \tau_2 \geq n_k(t) \geq \tau_1 \right) \leq \sum_{s = \tau_1}^{\tau_2}\exp \left(- 2 s \epsilon^2\right).$$ 
\label{lem:UnionBoundTrickInt}
\end{lem}

\begin{proof}
Let $Z_1, Z_2, ... Z_t$ be the reward samples of arm $k$ drawn separately. If the algorithm chooses to play arm $k$ for $m^{th}$ time, then it observes reward $Z_m$. Then the probability of observing the event $\hat{\mu}_{k,n_k(t)} - \mu_k \geq \epsilon, \tau_2 \geq n_k(t) \geq \tau_1$ can be upper bounded as follows,
\begin{align}
    \Pr\left(\hat{\mu}_{k,n_k(t)} - \mu_k \geq \epsilon, \tau_2 \geq n_k(t) \geq \tau_1 \right) &= \Pr\left( \left( \frac{\sum_{i=1}^{n_k(t)}Z_i}{n_k(t)} - \mu_k \geq \epsilon \right), \tau_2 \geq n_k(t) \geq \tau_1 \right) \\
    &\leq \Pr\left( \left(\bigcup_{m = \tau_1}^{\tau_2} \frac{\sum_{i=1}^{m}Z_i}{m} - \mu_k \geq \epsilon \right), \tau_2 \geq n_k(t) \geq \tau_1 \right) \label{upperBoundTrick}\\
    &\leq \Pr \left(\bigcup_{m = \tau_1}^{\tau_2} \frac{\sum_{i=1}^{m}Z_i}{m} - \mu_k \geq \epsilon \right) \\
    &\leq \sum_{s = \tau_1}^{\tau_2}\exp \left( - 2 s \epsilon^2\right).
\end{align}
\end{proof}

\begin{lem}[From Proof of Theorem 1 in \cite{auer2002finite}]
\label{lem:ucbindexmore}
Let $\Index_\arm(\slot)$ denote the UCB index of arm $\arm$ at round $\slot$, and $\meanReward_\arm = \E{g_{\arm}(X)}$ denote the mean reward of that arm.  Then, we have
$$\Pr(\meanReward_\arm > \Index_\arm(\slot)) \leq \slot^{-3}.$$ 
\end{lem}
Observe that this bound does not depend on the number $\pulls_\arm(\slot)$ of times arm $\arm$ is pulled. UCB index is defined in equation (6) of the main paper.
\begin{proof}
This proof follows directly from \cite{auer2002finite}. We present the proof here for completeness as we use this frequently in the paper.
\begin{align}
    \Pr(\meanReward_\arm > \Index_\arm(\slot)) &= \Pr\left(\meanReward_\arm > \hat{\meanReward}_{\arm,\pulls_\arm(\slot)} + \sqrt{\frac{2 \log \slot}{\pulls_\arm(\slot)}}\right) \\
    &\leq \sum_{m = 1}^{\slot} \Pr \left(\meanReward_\arm > \hat{\meanReward}_{\arm,m} + \sqrt{\frac{2 \log \slot}{m}} \right) \label{unionTrick}\\
    &= \sum_{m =1}^{\slot} \Pr \left(\hat{\meanReward}_{\arm,m} - \meanReward_\arm < - \sqrt{\frac{2 \log \slot}{m}}\right) \\ 
    &\leq \sum_{m = 1}^{\slot} \exp\left(- 2 m \frac{2 \log \slot}{m}\right) \label{eqn:ucbindex}\\
    &= \sum_{m = 1}^{\slot} \slot^{-4} \\
    &= \slot^{-3}.
\end{align}
where \eqref{unionTrick} follows from the union bound and is a standard trick (\Cref{lem:UnionBoundTrickInt}) to deal with random variable $\pulls_\arm(\slot)$. We use this trick repeatedly in the proofs. We have \eqref{eqn:ucbindex} from the Hoeffding's inequality. 
\end{proof}

\begin{lem} Let $\E{\indicator_{\Index_\arm > \Index_{\arm^*}}}$  be the expected number of times $\Index_\arm (t)> \Index_{\arm^*}(t)$ in $\totalPulls$ rounds. Then, we have 
$$\E{\indicator_{\Index_\arm > \Index_{\arm^*}}} = \sum_{\slot = 1}^{\totalPulls} Pr(\Index_\arm > \Index_{\arm^*}) \leq \frac{8 \log (\totalPulls)}{\gap_\arm^2} + \left(1 + \frac{\pi^2}{3} \right).$$
\label{lem:AuerResult}
\end{lem}

The proof follows the analysis in Theorem 1 of \cite{auer2002finite}. The analysis of  $Pr(\Index_\arm > \Index_{\arm^*})$ is done by by evaluating the joint probability $\Pr\left(\Index_\arm(\slot) > \Index_{\arm^*}(\slot),  \pulls_\arm(\slot) \geq \frac{8 \log T}{\Delta_k^2}\right)$. Authors in \cite{auer2002finite} show that the probability of pulling arm $k$ jointly with the event that it has had at-least $\frac{8 \log T}{\Delta_k^2}$ pulls decays down with $t$, i.e., $\Pr\left(\Index_\arm(\slot) > \Index_{\arm^*}(\slot), \pulls_\arm(\slot) \geq \frac{8 \log T}{\Delta_k^2}\right) \leq \slot^{-2}$.

\begin{lem}[Theorem 2 \cite{lai1985asymptotically}]

Consider a two armed bandit problem with reward distributions $\Theta = \{f_{R_1}(\reward), f_{R_2}(\reward)\}$, where the reward distribution of the optimal arm is $f_{R_1}(\reward)$ and for the sub-optimal arm is $f_{R_2}(\reward)$, and $\E{f_{R_1}(\reward)} > \E{f_{R_2}(\reward)}$; i.e., arm 1 is optimal. If it is possible to create an alternate problem with distributions $\Theta' = \{f_{R_1}(\reward), \tilde{f}_{R_2}(\reward)\}$ such that $\E{\tilde{f}_{R_2}(\reward)} > \E{f_{R_1}(\reward)}$ and $0< D(f_{R_2}(r)||\tilde{f}_{R_2}(r)) < \infty$ (equivalent to assumption 1.6 in \cite{lai1985asymptotically}),  then for any policy that achieves sub-polynomial regret, we have $$\liminf\limits_{\totalPulls \rightarrow \infty}  \frac{\E{\pulls_2(\totalPulls)}}{\log \totalPulls} \geq \frac{1}{D(f_{R_2}(r) || \tilde{f}_{R_2}(r))}.$$

\label{lem:LaiRobbins2Arms}
\end{lem}

\begin{proof}
Proof of this is derived from the analysis done in \cite{banditalgs}. We show the analysis here for completeness. A bandit instance $v$ is defined by the reward distribution of arm 1 and arm 2. Since policy $\pi$ achieves sub-polynomial regret, for any instance $v$, $\mathbb{E}_{v,\pi}\left[(\regret(\totalPulls))\right] = \OO(\totalPulls^p)$ as $\totalPulls \rightarrow \infty$, for all $p > 0$. 

Consider the bandit instances $\Theta = \{f_{R_1}(r), f_{R_2}(r)\}$, $\Theta' = \{f_{R_1}(r), \tilde{f}_{R_2}(r)\}$, where $\E{f_{R_2}(r)} < \E{f_{R_1}(r)} < \E{\tilde{f}_{R_2}(r)}$. The bandit instance $\Theta'$ is constructed by changing the reward distribution of arm 2 in the original instance, in such a way that arm 2 becomes optimal in instance $\Theta'$ without changing the reward distribution of arm 1 from the original instance.

From divergence decomposition lemma (derived in \cite{banditalgs}), it follows that $$D(\mathbb{P}_{\Theta,\Pi} || \mathbb{P}_{\Theta',\Pi}) = \mathbb{E}_{\Theta,\pi}\left[\pulls_2(\totalPulls)\right] D(f_{R_2}(r) || \tilde{f}_{R_2}(r)).$$

The high probability Pinsker's inequality (Lemma 2.6 from \cite{Tsybakov:2008:INE:1522486}, originally in \cite{highProbPinsker}) gives that for any event $A$, $$\mathbb{P}_{\Theta, \pi}(A) + \mathbb{P}_{\Theta',\pi}(A^c) \geq \frac{1}{2}\exp\left(-D(\mathbb{P}_{\Theta,\pi} || \mathbb{P}_{\Theta',\pi})\right),$$
or equivalently, 
$$D(\mathbb{P}_{\Theta,\pi} || \mathbb{P}_{\Theta',\pi}) \geq \log \frac{1}{2(\mathbb{P}_{\Theta,\pi}(A) + \mathbb{P}_{\Theta',\pi}(A^c))}.$$

If arm 2 is suboptimal in a 2-armed bandit problem, then $\E{\regret(\totalPulls)} = \gap_2 \E{\pulls_2(\totalPulls)}.$ Expected regret in $\Theta$ is $$\mathbb{E}_{\Theta,\pi}\left[\regret(\totalPulls)\right] \geq \frac{\totalPulls \gap_2}{2} \mathbb{P}_{\Theta,\pi}\left(\pulls_2(\totalPulls) \geq \frac{\totalPulls}{2}\right),$$
Similarly regret in bandit instance $\Theta'$ is $$\mathbb{E}_{\Theta',\pi}\left[\regret(\totalPulls)\right] \geq \frac{\totalPulls \delta}{2} \mathbb{P}_{\Theta',\pi}\left(\pulls_2(\totalPulls) < \frac{\totalPulls}{2}\right),$$
since suboptimality gap of arm $1$ in $\Theta'$ is  $\delta$. Define $\kappa(\gap_2, \delta) = \frac{\min(\gap_2, \delta)}{2}$. Then we have, $$\mathbb{P}_{\Theta,\pi}\left(\pulls_2(\totalPulls) \geq \frac{\totalPulls}{2}\right) + \mathbb{P}_{\Theta',\pi}\left(\pulls_2(\totalPulls) < \frac{\totalPulls}{2}\right) \leq \frac{\mathbb{E}_{\Theta,\pi}\left[\regret(\totalPulls)\right] + \mathbb{E}_{\Theta',\pi}\left[\regret(\totalPulls)\right]}{\kappa(\gap_2, \delta) \totalPulls}.$$

On applying the high probability Pinsker's inequality and divergence decomposition lemma stated earlier, we get 
\begin{align}
    D(f_{R_2}(r) || \tilde{f}_{R_2}(r)) \mathbb{E}_{\Theta,\pi}\left[\pulls_2(\totalPulls)\right] &\geq \log\left(\frac{\kappa(\gap_2, \delta) \totalPulls}{2 (\mathbb{E}_{\Theta,\pi}\left[\regret(\totalPulls)\right] + \mathbb{E}_{\Theta',\pi}\left[\regret(\totalPulls)\right]) }\right) \\
    &= \log\left(\frac{\kappa(\gap_2,\delta)}{2}\right) + \log(\totalPulls)  \nonumber \\ 
    &\qquad   - \log(\mathbb{E}_{\Theta,\pi}\left[\regret(\totalPulls)\right] + \mathbb{E}_{\Theta',\pi}\left[\regret(\totalPulls)\right]).
\end{align}

Since policy $\pi$ achieves sub-polynomial regret for any bandit instance, $\mathbb{E}_{\Theta,\pi}\left[\regret(\totalPulls)\right] + \mathbb{E}_{\Theta',\pi}\left[\regret(\totalPulls)\right] \leq \gamma \totalPulls^p$ for all $\totalPulls$ and any $p > 0$, 
hence, 
\begin{align} 
\liminf\limits_{\totalPulls \rightarrow \infty}  D(f_{R_2}(r) || \tilde{f}_{R_2}(r)) \frac{\mathbb{E}_{\Theta,\pi}\left[\pulls_2(\totalPulls)\right]}{\log \totalPulls} &\geq 1 - \limsup\limits_{\totalPulls \rightarrow \infty}  \frac{\mathbb{E}_{\Theta,\pi}\left[\regret(\totalPulls)\right] + \mathbb{E}_{\Theta',\pi}\left[\regret(\totalPulls)\right]}{\log \totalPulls} + \nonumber \\
&\quad \liminf\limits_{\totalPulls \rightarrow \infty}  \frac{\log\left(\frac{\kappa(\Delta_2,\delta)}{2}\right)}{\log \totalPulls} \\
&= 1.
\end{align}

Hence, $\liminf\limits_{\totalPulls \rightarrow \infty} \frac{\mathbb{E}_{\Theta,\pi}\left[\pulls_2(\totalPulls)\right]}{\log \totalPulls} \geq \frac{1}{D(f_{R_2}(r) || \tilde{f}_{R_2}(r))}.$

\end{proof}

\section{Results for any \textsc{C-Bandit} Algorithm}
\begin{lem}
Define $E_1(\slot)$ to be the event that arm $\arm^*$ is empirically \textit{non-competitive} in round $\slot+1$, then, 
$$\Pr(E_1(\slot)) \leq 2K\slot \exp \left(\frac{-\slot \gap_{\text{min}}^2}{2 \numArms}\right),$$
where $\gap_{\text{min}} = \min_\arm \gap_\arm$, the gap between the best and second-best arms.
\label{lem:eliminatedOptimal}
\end{lem}

\begin{proof}
The arm $k^*$ is empirically non-competitive at round $t$ if $k^* \neq k^{\text{emp}}$ and the empirical psuedo-reward of arm $k^*$ with respect to arms $\ell \in \mathcal{S}_t$ is smaller than $\hat{\mu}_{k^{\text{emp}}}(t)$. This event can only occur if at-least one of the two following conditions is satisfied, i) the empirical mean of $k^{\text{emp}} \neq k^*$ is greater than $\mu_k^* - \frac{\Delta_{\text{min}}}{2}$ or ii) the empirical pseudo-reward of arm $k^*$ with respect to arms in $\mathcal{S}_t$ is smaller than $\mu_{k^*} - \frac{\Delta_{\text{min}}}{2}$. We use this observation to analyse the $\Pr(E_1(t))$.
\begin{align}
    &\Pr(E_1(t)) \leq \Pr\left( \left(\max_{\{\ell: n_\ell(t) > t/K,  \ell \neq k^*\}} \hat{\mu}_{\ell}(t) > \mu_{k^*} - \frac{\Delta_{\text{min}}}{2} \right) \bigcup \left(\min_{\{\ell:n_\ell(t) > t/K\}} \hat{\phi}_{k^*, \ell}(t) < \mu_{k^*} - \frac{\Delta_{\text{min}}}{2} \right) \right) \\
    &\leq \Pr\left(\max_{\{\ell:n_\ell(t) > t/K, \ell \neq k^*\}} \hat{\mu}_{\ell}(t) > \mu_{k^*} - \frac{\Delta_{\text{min}}}{2} \right) + \Pr\left(\min_{\{\ell:n_\ell(t) > t/K\}} \hat{\phi}_{k^*, \ell}(t) < \mu_{k^*} - \frac{\Delta_{\text{min}}}{2} \right)
    \label{eq:simpleUnbd}\\
    &\leq \sum_{\ell \neq k^*} \Pr\left(\hat{\mu}_{\ell}(t) > \mu_{k^*} - \frac{\Delta_{\text{min}}}{2}, n_{\ell}(t) > \frac{t}{K} \right) + \sum_{\ell = 1}^{K} \Pr\left(\hat{\phi}_{k^*, \ell}(t) < \mu_{k^*} - \frac{\Delta_{\text{min}}}{2}, n_\ell(t) > \frac{t}{K} \right) \label{eq:atleastOne} \\
    &= \sum_{\ell \neq k^*} \Pr\left(\hat{\mu}_{\ell}(t) - \mu_{\ell} > \mu_{k^*} - \mu_{\ell} - \frac{\Delta_{\text{min}}}{2}, n_{\ell}(t) > \frac{t}{K} \right) \nonumber \\
    &+ \sum_{\ell = 1}^{K} \Pr\left(\hat{\phi}_{k^*, \ell}(t) - \phi_{k^*, \ell} < \mu_{k^*} - \phi_{k^*, \ell} - \frac{\Delta_{\text{min}}}{2}, n_\ell(t) > \frac{t}{K} \right) \\
    &\leq \sum_{\ell \neq k^*} \Pr\left(\frac{\sum_{\tau = 1}^{\slot} \indicator_{\{\arm_\tau = \ell\}} \reward_\tau}{\pulls_\ell(\slot)} - \mu_{\ell} > \frac{\Delta_{\text{min}}}{2}, n_{\ell}(t) > \frac{t}{K} \right) \nonumber \\
    &+ \sum_{\ell = 1}^{K} \Pr\left(\frac{\sum_{\tau = 1}^{\slot} \indicator_{\{\arm_\tau = \ell\}} \estimateReward_{\arm^*,\ell}(\reward_\tau)}{\pulls_\ell(\slot)} - \phi_{k^*, \ell} <  - \frac{\Delta_{\text{min}}}{2}, n_\ell(t) > \frac{t}{K} \right) \\    
    &\leq 2K\slot \exp\left(\frac{- \slot \gap_{\text{min}}^2}{2 \numArms}\right), \label{eq:appliedChernoff}
   \end{align}
Here \eqref{eq:simpleUnbd} follows from union bound. We have \eqref{eq:appliedChernoff} from the Hoeffding's inequality, as we note that rewards $\{\reward_\tau : \tau=1,\ldots, t, ~ k_{\tau}=k\}$ and pseudo-rewards $\{s_{k^*,l}: \tau_1, \ldots, t, ~ k_{\tau} = l\} $ form a collection of i.i.d. random variables each of which is bounded between $[-1,1]$ with mean $\meanReward_\arm$ and $\phi_{k^*,l}$. The term $\slot$ before the exponent in \eqref{eq:appliedChernoff} arises as the random variable $\pulls_\arm(\slot)$ can take values from $\slot/\numArms$ to $\slot$ (\Cref{lem:UnionBoundTrickInt}).

\end{proof}

\begin{lem} For a sub-optimal arm $k \neq k^*$ with sub-optimality gap $\Delta_k$,  
$$\Pr\left(k = k^{\text{emp}}(t), n_{k^*}(t) \geq \frac{t}{K}\right) \leq t\exp\left(\frac{-t\Delta_k^2}{2K}\right).$$
\label{lem:kempLem}
\end{lem}

\begin{proof}
We bound this probability as, 
\begin{align}
    &\Pr\left(k = k^{\text{emp}}(t), n_{k^*}(t) \geq \frac{t}{K}\right) \nonumber\\
    &= \Pr\left(k = k^{\text{emp}}(t), n_{k^*}(t) \geq \frac{t}{K}, n_k(t) \geq \frac{t}{K}\right) \label{eqn:kempMeanstk}\\
    &\leq \Pr\left(\hat{\mu}_k(t) \geq \hat{\mu}_{k^*}(t), n_k(t) \geq \frac{t}{K}, n_{k^*}(t) \geq \frac{t}{K}\right) \\
    &\leq \Pr\left( \left(\hat{\mu}_{k^*}(t) < \mu_{k^*} - \frac{\Delta_k}{2} \bigcup \hat{\mu}_k(t) > \mu_{k^*} - \frac{\Delta_k}{2} \right), n_k(t) \geq \frac{t}{K}, n_{k^*}(t) \geq \frac{t}{K}\right) \label{eqn:necessaryThings}\\ 
    &= \Pr\left( \left(\hat{\mu}_{k^*}(t) < \mu_{k^*} - \frac{\Delta_k}{2} \bigcup \hat{\mu}_k(t) > \mu_k + \frac{\Delta_k}{2} \right), n_k(t) \geq \frac{t}{K}, n_{k^*}(t) \geq \frac{t}{K}\right) \\
    &\leq \Pr\left( \hat{\mu}_{k^*}(t) - \mu_{k^*} < - \frac{\Delta_k}{2}, n_{k^*}(t) \geq \frac{t}{K}\right) + \Pr\left( \hat{\mu}_k(t) - \mu_k > \frac{\Delta_k}{2} , n_k(t) \geq \frac{t}{K}\right) \\
    &\leq 2t\exp\left(\frac{-t \Delta_k^2}{2K}\right) \label{eqn:chernoffplusunion}
\end{align}
We have \eqref{eqn:kempMeanstk} as arm $k$ needs to be pulled at least $\frac{t}{K}$ in order to be arm $k^{\text{emp}}(t)$ at round $t$. The selection of $k^{\text{emp}}$ is only done from the set of arms that have been pulled atleast $\frac{t}{K}$ times. Here, \eqref{eqn:chernoffplusunion} follows from the Hoeffding's inequality. The term $\slot$ before the exponent in \eqref{eqn:chernoffplusunion} arises as the random variable $\pulls_\arm(\slot)$ can take values from $\slot/\numArms$ to $\slot$ (\Cref{lem:UnionBoundTrickInt}).

\end{proof}

\begin{lem}
If for a suboptimal arm $\arm \neq \arm^*$, $\optimistGap_{\arm,\arm^*} > 0$, then,
$$\Pr(\arm_{\slot+1} = \arm, \pulls_{\arm^*}(\slot) = \max_\arm \pulls_\arm(t)) \leq \slot \exp\left(\frac{-2t\optimistGap_{\arm,\arm^*}^2}{ \numArms}\right).$$

Moreover, if $\optimistGap_{\arm,\arm^*} \geq \sqrt{\frac{2 \numArms \log \slot_0}{\slot_0}}$ for some constant $\slot_0 > 0$. Then, 
$$\Pr(\arm_{\slot+1} = \arm , \pulls_{\arm^*}(\slot) = \max_\arm \pulls_\arm(t)) \leq \slot^{-3} \quad \forall \slot > \slot_0.$$
\label{lem:suboptimalNotCompetitive}
\end{lem}

\begin{proof}
We now bound this probability as,
\begin{align}
    &\Pr(\arm_{\slot+1} = \arm , \pulls_{\arm^*}(t) = \max_\arm \pulls_\arm(t)) \nonumber \\ &\leq\Pr\left(k_{t+1} = k, n_{k^*}(t) \geq \frac{t}{K}\right) \label{eq:needAtleast}\\  
    &= \Pr\left(k \in \{\mathcal{A}_t \cup \{k_{\text{emp}}(t)\} \}, k_{t+1} = k, n_{k^*}(t) \geq \frac{t}{K}\right)  \label{eq:hastohappen} \\
    &\leq \Pr\left(k \in \mathcal{A}_t, k_{t+1} = k, n_{k^*}(t) \geq \frac{t}{K}\right) + \Pr\left(k = k^{\text{emp}}(t), n_{k^*}(t) \geq \frac{t}{K}\right) \\
    &\leq \Pr\left(k \in \mathcal{A}_t, k_{t+1} = k, n_{k^*}(t) \geq \frac{t}{K}\right) + 2t\exp\left(\frac{-t\Delta_k^2}{2K}\right) \label{eq:fromkempLem}\\
    &\leq \Pr\left(\hat{\meanReward}_{\arm^*}(\slot) < \estimateMean_{\arm,\arm^*}(\slot), k_{t+1} = k, \pulls_{\arm^*}(\slot) \geq \frac{t}{K} \right) + 2t\exp\left(\frac{-t\Delta_k^2}{2K}\right) \label{eq:necessaryCondn}\\
    &\leq \Pr\left(\hat{\meanReward}_{\arm^*}(\slot) < \estimateMean_{\arm,\arm^*}(\slot) , \pulls_{\arm^*}(\slot) \geq \frac{t}{K} \right) + 2t\exp\left(\frac{-t\Delta_k^2}{2K}\right)\\
    &\leq \Pr\left(\frac{\sum_{\tau = 1}^{\slot}\indicator_{\{\arm_\tau = \arm^*\}}\reward_\tau}{\pulls_{\arm^*}(\slot)} < \frac{\sum_{\tau = 1}^{\slot}\indicator_{\{\arm_\tau = \arm^*\}}\estimateReward_{\arm,\arm^*}(\reward_\tau)}{\pulls_{\arm^*(\slot)}} , \pulls_{\arm^*}(\slot) \geq \frac{\slot}{\numArms}\right) + 2t\exp\left(\frac{-t\Delta_k^2}{2K}\right)\\
    &= \Pr\left(\frac{\sum_{\tau = 1}^{\slot} \indicator_{\{\arm_\tau = \arm^*\}}(\reward_{\tau} - \estimateReward_{\arm,\arm^*})}{\pulls_{\arm^*}(\slot)} - (\meanReward_{\arm^*} - \expectedPseudoReward_{\arm,\arm^*}) < - \optimistGap_{\arm,\arm^*} , \pulls_{\arm^*} \geq \frac{\slot}{\numArms} \right) + 2t\exp\left(\frac{-t\Delta_k^2}{2K}\right) \label{eqn:lemchernoff}\\
    &\leq \slot \exp \left( \frac{- \slot \optimistGap_{\arm,\arm^*}^2}{2 \numArms} \right) + 2t\exp\left(\frac{-t\Delta_k^2}{2K}\right) \\
    &\leq 3\slot^{-3} \quad \forall \slot > \slot_0.\label{lastStepHere}
\end{align}

We have \eqref{eq:needAtleast} as $n_{k^*}(t)$ needs to be at-least $\frac{t}{K}$ for $n_{k^*}(t)$ to be $\max_{k} n_k(t)$. Equation \eqref{eq:hastohappen} holds as arm $k$ needs to be in the set $\{\mathcal{A}_t \cup \{k^{\text{emp}}(t)\} \}$ to be selected by C-BANDIT at round $t$. Inequality \eqref{eq:fromkempLem} arises from the result of \Cref{lem:kempLem}. The inequality \eqref{eq:necessaryCondn} follows as $\phi_{k,k^*} > \hat{\mu}_{k^*}$ is a necessary condition for arm $k$ to be in the competitive set $\mathcal{A}_t$ at round $t$. Here, \eqref{eqn:lemchernoff} follows from the Hoeffding's inequality as we note that rewards $\{\reward_\tau - \estimateReward_{\arm,\arm^*}(\reward_\tau): \tau=1,\ldots, t, ~ k_{\tau}=k^*\}$ form a collection of i.i.d. random variables each of which is bounded between $[-1,1]$ with mean $(\meanReward_\arm - \expectedPseudoReward_{\arm, \arm^*})$. The term $\slot$ before the exponent in \eqref{eqn:lemchernoff} arises as the random variable $\pulls_\arm(\slot)$ can take values from $\slot/\numArms$ to $\slot$ (\Cref{lem:UnionBoundTrickInt}). Step \eqref{lastStepHere} follows from the fact that $\optimistGap_{\arm,\arm^*} \geq 2\sqrt{\frac{2 \numArms \log \slot_0}{\slot_0}}$ for some constant $\slot_0 > 0$. 
\end{proof}

\section{Algorithm specific results for C-UCB}

\begin{lem}
If $\gap_{\text{min}} \geq 4\sqrt{\frac{2\numArms \log \slot_0}{\slot_0}}$ for some constant $\slot_0 > 0$, then,
$$\Pr(\arm_{\slot+1} = \arm , \pulls_\arm(\slot) \geq s) \leq 3 \slot^{-3} \quad \text{for } s > \frac{\slot}{2 \numArms}, \forall \slot > \slot_0.$$
\label{lem:noMorePulls}
\end{lem}

\begin{proof}
By noting that $\arm_{\slot + 1} = \arm$ corresponds to arm $\arm$ having the highest index among the set of arms that are not empirically \textit{non-competitive} (denoted by $\mathcal{A}$), we have,  
\begin{align}
    \Pr(\arm_{\slot + 1} = \arm , \pulls_\arm(\slot) \geq s) &= \Pr(\Index_\arm(\slot) = \arg \max_{\arm' \in \mathcal{A}} \Index_{\arm'}(\slot) , \pulls_\arm(\slot) \geq s) \\
    &\leq \Pr(E_1(\slot) \cup \left(E_1^c(\slot), \Index_\arm(\slot) > \Index_{\arm^*}(\slot)\right) , \pulls_\arm(\slot) \geq s) \label{eliminating1}\\ 
    &\leq \Pr(E_1(\slot) , \pulls_\arm(\slot) \geq s) + \Pr(E_1^c(\slot), \Index_\arm(\slot) > \Index_{\arm^*}(\slot) , \pulls_\arm(\slot) \geq s ) \label{unionBound}\\
    &\leq 2K\slot \exp\left(\frac{-\slot \gap_{\text{min}}^2}{2 \numArms}\right) + \Pr\left(\Index_\arm(\slot) > \Index_{\arm^*}(\slot) , \pulls_\arm(\slot) \geq s\right). \label{usedHoeffding}
\end{align}
Here $E_1(t)$ is the event described in \Cref{lem:eliminatedOptimal}. If arm $\arm^*$ is not empirically non-competitive at round $\slot$, then arm $\arm$ can only be pulled in round $\slot + 1$ if $\Index_\arm(\slot) > \Index_{\arm^*}(\slot)$, due to which we have \eqref{eliminating1}. Inequalities \eqref{unionBound} and \eqref{usedHoeffding} follow from union bound and \Cref{lem:eliminatedOptimal} respectively.

We now bound the second term in \eqref{usedHoeffding}.
\begin{align}
    &\Pr(\Index_\arm(\slot) > \Index_{\arm^*}(\slot) , \pulls_\arm(\slot) \geq s) = \nonumber \\
    &\Pr\left(\Index_\arm(\slot) > \Index_{\arm^*}(\slot) ,  \pulls_\arm(\slot) \geq s, \meanReward_{\arm^*} \leq \Index_{\arm^*}(\slot)\right)  + \nonumber \\
    &\quad \Pr\left(\Index_\arm(\slot) > \Index_{\arm^*}(\slot), \pulls_\arm(\slot) \geq s | \meanReward_{\arm^*} > \Index_{\arm^*}(\slot) \right) \times \Pr\left(\meanReward_{\arm^*} > \Index_{\arm^*}(\slot) \right) \label{conditionTerm} \\
    &\leq  \Pr\left(\Index_\arm(\slot) > \Index_{\arm^*}(\slot), \pulls_\arm(\slot) \geq s, \meanReward_{\arm^*} \leq \Index_{\arm^*}(\slot)\right) + \Pr\left(\meanReward_{\arm^*} > \Index_{\arm^*}(\slot)\right) \label{droppingTerms}\\
    &\leq \Pr\left(\Index_\arm(\slot) > \Index_{\arm^*}(\slot), \pulls_\arm(\slot) \geq s, \meanReward_{\arm^*} \leq \Index_{\arm^*}(\slot)\right) + \slot^{-3} \label{usingHoeffdingAgain}\\
    &= \Pr\left(\Index_\arm(\slot) > \meanReward_{\arm^*} ,  \pulls_\arm(\slot) \geq s\right) + \slot^{-4} \label{usingConditioning} \\
    &= \Pr\left(\hat{\meanReward}_\arm(\slot) + \sqrt{\frac{2 \log \slot}{\pulls_\arm(\slot)}} > \meanReward_{\arm^*} , \pulls_\arm(\slot) \geq s \right) + \slot^{-3} \label{expandingIndex}\\
    &= \Pr\left(\hat{\meanReward}_\arm(\slot) - \meanReward_\arm > \meanReward_{\arm^*} - \meanReward_\arm - \sqrt{\frac{2 \log \slot}{\pulls_\arm(\slot)}} , \pulls_\arm(\slot) \geq s \right) + \slot^{-3} \\ 
    &= \Pr\left( \frac{\sum_{\tau = 1}^{\slot} \indicator_{\{\arm_\tau = \arm\}}\reward_\tau}{\pulls_\arm(\slot)} - \meanReward_\arm > \gap_\arm - \sqrt{\frac{2 \log \slot}{\pulls_\arm(\slot)}} , \pulls_\arm(\slot) \geq s\right) + \slot^{-3} \\
    &\leq \slot \exp\left(-2 s \left(\gap_\arm - \sqrt{\frac{2 \log \slot}{s}}\right)^2\right) + \slot^{-3} \label{eqn:chernoffagain}\\
    &\leq \slot^{-3}\exp\left(-2 s \left(\gap_\arm^2 - 2 \gap_\arm \sqrt{\frac{2 \log \slot}{s}}\right)\right) + \slot^{-3} \\
    &\leq 2 \slot^{-3} \quad \text{ for all  } \slot > \slot_0. \label{finalCondn}
\end{align}
We have \eqref{conditionTerm} holds because of the fact that $P(A) = P(A|B)P(B) + P(A|B^c)P(B^c)$, Inequality \eqref{usingHoeffdingAgain} follows from \Cref{lem:ucbindexmore}. From the definition of $\Index_\arm(\slot)$ we have \eqref{expandingIndex}. Inequality \eqref{eqn:chernoffagain} follows from Hoeffding's inequality and the term $\slot$ before the exponent in \eqref{eqn:chernoffagain} arises as the random variable $\pulls_\arm(\slot)$ can take values from $s$ to $\slot$ (\Cref{lem:UnionBoundTrickInt}). Inequality \eqref{finalCondn} follows from the fact that $s > \frac{\slot}{2 \numArms}$ and $\gap_\arm \geq 4\sqrt{\frac{2\numArms \log \slot_0}{\slot_0}}$ for some constant $\slot_0 > 0.$

Plugging this in the expression of $\Pr(\arm_\slot = \arm, \pulls_\arm (\slot) \geq s)$ \eqref{usedHoeffding} gives us, 
\begin{align}
    \Pr(\arm_{\slot+1} = \arm , \pulls_\arm (\slot) \geq s) &\leq 2 K \slot \exp\left(\frac{-\slot \gap_{\text{min}}^2}{2 \numArms}\right) + \Pr(\Index_\arm(\slot) > \Index_{\arm^*}(\slot) , \pulls_\arm(\slot) \geq s) \\
    &\leq 2K\slot\exp\left(\frac{-\slot \gap_{\text{min}}^2}{2 \numArms}\right) + 2\slot^{-3} \\
    &\leq 2(K+1) \slot^{-3}. \label{usingConditiont0}
\end{align}
Here, \eqref{usingConditiont0} follows from the fact that $\gap_{\text{min}} \geq 4\sqrt{\frac{2\numArms \log \slot_0}{\slot_0}}$ for some constant $\slot_0 > 0$.  
\end{proof}

\begin{lem}
If $\gap_{\text{min}} \geq 4\sqrt{\frac{2 \numArms \log \slot_0}{\slot_0}}$ for some constant $\slot_0 > 0$, then, $$\Pr\left(\pulls_\arm(\slot) > \frac{\slot}{ \numArms}\right) \leq  (2K + 2)K \left(\frac{\slot}{\numArms}\right)^{-2} \quad \forall \slot > \numArms \slot_0.$$
\label{lem:suboptimalNotPulled}
\end{lem}

\begin{proof}
We expand $\Pr\left(\pulls_\arm(\slot) > \frac{t}{\numArms}\right)$ as,

\begin{align}
    \Pr\left(\pulls_\arm(\slot) \geq \frac{\slot}{\numArms}\right) &= \Pr\left( \pulls_{\arm}(\slot) \geq \frac{\slot}{\numArms} \mid \pulls_\arm(\slot - 1) \geq \frac{\slot}{\numArms} \right) \Pr\left( \pulls_\arm(\slot - 1) \geq \frac{\slot}{\numArms} \right) + \nonumber \\
    &\quad \Pr\left(\arm_\slot = \arm , \pulls_\arm(\slot - 1) = \frac{\slot}{\numArms} - 1\right)  \\
    &\leq \Pr\left(\pulls_\arm(\slot - 1) \geq \frac{\slot}{\numArms}\right) + \Pr\left(\arm_\slot = \arm , \pulls_\arm(\slot - 1) = \frac{\slot}{\numArms} - 1\right) \\
    &\leq \Pr\left(\pulls_\arm(\slot - 1) \geq \frac{\slot}{\numArms}\right) + (2K + 2) (\slot - 1)^{-3} \quad \forall (\slot - 1) > \slot_0. \label{fromPrevLemma}
\end{align}
Here, \eqref{fromPrevLemma} follows from \Cref{lem:noMorePulls}.\\

This gives us $$\Pr\left(\pulls_\arm(\slot) \geq \frac{\slot}{\numArms}\right) - \Pr\left(\pulls_\arm(\slot - 1) \geq \frac{\slot}{\numArms}\right) \leq (2K + 2)(\slot - 1)^{-3}, \quad \forall (\slot - 1) > \slot_0.$$
Now consider the summation $$ \sum_{\tau = \frac{\slot}{\numArms}}^{\slot} \Pr\left(\pulls_\arm(\tau) \geq \frac{\slot}{\numArms}\right) - \Pr\left(\pulls_\arm(\tau - 1) \geq \frac{\slot}{\numArms}\right) \leq \sum_{\tau = \frac{\slot}{\numArms}}^{\slot}(2K + 2)(\tau - 1)^{-3}.$$ This gives us, $$\Pr\left(\pulls_\arm(\slot) \geq \frac{\slot}{\numArms}\right) - \Pr\left(\pulls_\arm\left(\frac{\slot}{\numArms} - 1\right) \geq \frac{\slot}{\numArms}\right) \leq \sum_{\tau = \frac{\slot}{\numArms}}^{\slot}(2K + 2)(\tau - 1)^{-3}.$$
Since $\Pr\left(\pulls_\arm\left(\frac{\slot}{\numArms} - 1\right)\geq \frac{\slot}{\numArms}\right)  = 0$, we have, 
\begin{align}
    \Pr\left(\pulls_\arm(\slot) \geq \frac{\slot}{\numArms}\right) &\leq \sum_{\tau = \frac{\slot}{\numArms}}^{\slot}(2K + 2)(\tau - 1)^{-3} \\
    &\leq (2K + 2)K \left(\frac{\slot}{\numArms}\right)^{-2} \quad \forall \slot > \numArms \slot_0.
\end{align}

\end{proof}

\section{Regret Bounds for C-UCB}
\label{proof:UCB}

\textbf{Proof of Theorem 1}
We bound $\E{\pulls_\arm(\totalPulls)}$ as,
\begin{align}
\E{\pulls_\arm(\totalPulls)} &= \E{\sum_{\slot = 1}^{\totalPulls}\indicator_{\{\arm_\slot = \arm\}}}\\
&= \sum_{\slot = 0}^{\totalPulls-1} \Pr(\arm_{\slot+1} = \arm) \\
&= \sum_{\slot = 1}^{\numArms \slot_0} \Pr(\arm_\slot = \arm) + \sum_{\slot = \numArms \slot_0}^{\totalPulls-1} \Pr(\arm_{\slot+1} = \arm) \\
&\leq \numArms \slot_0 + \sum_{\slot = \numArms \slot_0}^{\totalPulls-1}\Pr(\arm_{\slot+1} = \arm, \pulls_{\arm^*}(\slot) = \max_{\arm'} \pulls_{\arm'}(\slot))  \nonumber \\
&+ \sum_{\slot = \numArms \slot_0}^{\totalPulls-1} \sum_{\arm' \neq \arm^*} \Pr(\pulls_{\arm'}(\slot) = \max_{\arm''} \pulls_{\arm''}(\slot))\Pr(\arm_{\slot+1} = \arm |  \pulls_{\arm'}(\slot) = \max_{\arm''} \pulls_{\arm''}(\slot)) \\
&\leq \numArms \slot_0 + \sum_{\slot = \numArms \slot_0}^{\totalPulls-1} \Pr(\arm_{\slot+1} = \arm, \pulls_{\arm^*}(\slot) = \max_{\arm'} \pulls_{\arm'}(\slot)) \nonumber \\
&+ \sum_{\slot = \numArms \slot_0}^{\totalPulls-1} \sum_{\arm' \neq \arm^*} \Pr(\pulls_{\arm'}(\slot) = \max_{\arm''} \pulls_{\arm''}(\slot)) \\
&\leq \numArms \slot_0 + \sum_{\slot = \numArms \slot_0}^{\totalPulls - 1} 3\slot^{-3} + \sum_{\slot = \numArms \slot_0}^{\totalPulls} \sum_{\arm' \neq \arm^*} \Pr\left(\pulls_{\arm'}(\slot) \geq \frac{\slot}{\numArms}\right) \label{usingSomeLemma1}\\
&\leq  \numArms \slot_0  + \sum_{\slot = 1}^{\totalPulls} 3\slot^{-3} + (\numArms + 1)K(\numArms - 1) \sum_{\slot = \numArms \slot_0}^{\totalPulls}  2\left(\frac{\slot}{\numArms}\right)^{-2}. \label{usingSomeOtherLemma}
\end{align}
Here, \eqref{usingSomeLemma1} follows from \Cref{lem:suboptimalNotCompetitive} and \eqref{usingSomeOtherLemma} follows from \Cref{lem:suboptimalNotPulled}.

\textbf{Proof of Theorem 2}

For any suboptimal arm $\arm \neq \arm^*$,
\begin{align}
    \E{\pulls_\arm(\totalPulls)} &\leq \sum_{\slot = 1}^{\totalPulls} \Pr(\arm_\slot = \arm) \\
    &= \sum_{\slot = 1}^{\totalPulls} \Pr(E_1(\slot), k_t = k \cup (E_1^c(\slot), \Index_\arm > \Index_{\arm^*}), k_t = k) \label{beatOptimal} \\
    &\leq \sum_{\slot = 1}^{\totalPulls} \Pr(E_1(\slot)) + \Pr(E_1^c(\slot), \Index_\arm(\slot - 1) > \Index_{\arm^*}(\slot - 1), k_t = k) \\
    &\leq \sum_{\slot = 1}^{\totalPulls} \Pr(E_1(\slot)) + \Pr(E_1^c(\slot), \Index_\arm(\slot - 1) > \Index_{\arm^*}(\slot - 1)) \nonumber\\
    &\leq \sum_{\slot = 1}^{\totalPulls} \Pr(E_1(\slot)) + \Pr(\Index_\arm(\slot - 1)> \Index_{\arm^*}(\slot - 1), k_t = k) \\
    &= \sum_{\slot = 1}^{\totalPulls} 2K\slot \exp\left(- \frac{\slot \gap_{\text{min}}^2}{2 \numArms}\right) + \sum_{\slot = 0}^{\totalPulls-1} \Pr\left(\Index_\arm(\slot) > \Index_{\arm^*}(\slot), k_t = k\right) \label{eliminatedArmProb1} \\
    &= \sum_{\slot = 1}^{\totalPulls} 2K\slot \exp\left(- \frac{\slot \gap_{\text{min}}^2}{2 \numArms}\right) + \E{\indicator_{\Index_\arm > \Index_{\arm^*}}(\totalPulls)} \label{followFromDefinition} \\
    &\leq 8 \frac{\log (\totalPulls)}{\gap_\arm^2} + \left(1 + \frac{\pi^2}{3}\right) + \sum_{\slot = 1}^{\totalPulls} 2K\slot \exp\left(- \frac{ \slot \gap_{\text{min}}^2}{2 \numArms}\right). \label{fromAuer1}
\end{align}
Here, \eqref{eliminatedArmProb1} follows from \Cref{lem:eliminatedOptimal}. We have \eqref{followFromDefinition} from the definition of $\E{n_{\Index_\arm > \Index_{\arm^*}}(\totalPulls)}$ in \Cref{lem:AuerResult}, and \eqref{fromAuer1} follows from \Cref{lem:AuerResult}.

\textbf{Proof of Theorem 3:} Follows directly by combining the results on Theorem 1 and Theorem 2.

\section{Regret analysis for the C-TS Algorithm}
\newadd{
We now present results for C-TS in the scenario where $K = 2$ and Thompson sampling is employed with Beta priors \cite{agrawal2013further}. In order to prove results for C-TS, we assume that rewards are either $0$ or $1$. The Thompson sampling algorithm with beta prior, maintains a posterior distribution on mean of arm $k$ as $Beta\left(n_k(t) \times \hat{\mu}_k(t) + 1, n_k(t) \times (1 - \hat{\mu}_k(t)) + 1 \right)$. Subsequently, it generates a sample $S_{k}(t) \sim Beta\left(n_k(t) \times \hat{\mu}_k(t) + 1, n_k(t) \times (1 - \hat{\mu}_k(t)) + 1 \right)$ for each arm $k$ and selects the arm $k_{t+1} = \argmax_{k \in \mathcal{K}} S_{k}(t)$. The C-TS algorithm with Beta prior uses this Thompson sampling procedure in its last step, i.e., $k_{t+1} = \argmax_{k \in \mathcal{C}_t} S_{k}(t)$, where $\mathcal{C}_t$ is the set of empirically competitive arms at round $t$. We show that in a 2-armed bandit problem, the regret is $\OO(1)$ if the sub-optimal arm $k$ is non-competitive and is $\OO(\log T)$ otherwise.

For the purpose of regret analysis of C-TS, we define two thresholds, a lower threshold $L_k$, and an upper threshold $U_k$ for arm $k\neq k^*$,
\begin{align}
U_k = \mu_k + \frac{\Delta_k}{3}, \hspace*{3em} L_k = \mu_{k^*} - \frac{\Delta_k}{3}. \label{eq:threshold}
\end{align}

Let $E^{\mu}_{i}(t)$ and $E^{S}_{i}(t)$ be the events that,
\begin{align}
E^{\mu}_k(t) &= \{\hat{\mu}_k(t) \leq U_k \} \nonumber\\
E^{S}_k(t) &= \{S_k(t) \leq L_k \}  \label{eq:events}.
\end{align}

To analyse the regret of C-TS, we first show that the number of times arm $k$ is pulled jointly with the event that $n_k(t-1) \geq \frac{t}{2}$ is bounded above by an $\OO(1)$ constant, which is independent of the total number of rounds $T$.

\begin{lem}
\label{lem:notPullsuboptimalifEnough}
If $\Delta_{k} \geq 4 \sqrt{\frac{ 2K \log t_{0}}{t_{0}}}$ for some constant $t_{0}>0$, then,
\begin{align*}
    \sum_{t = 2t_0}^{T} \Pr\left(k_{t}=k, n_{k}(t-1) \geq \frac{t}{2}\right) = \OO(1)
\end{align*}
where $k \neq k^{*}$ is a sub-optimal arm.
\end{lem}

\begin{proof}
We start by bounding the probability of the pull of $k$-th arm at round $t$ as follows,
\begin{align}
\Pr\left(k_{t}=k, n_{k}(t-1) \geq \frac{t}{2}\right) \overset{(a)}{\leq} & \Pr\left(E_{1}(t), k_{t}=k, n_{k}(t-1) \geq \frac{t}{2}\right) +  \nonumber \\
& \Pr\left(\overline{E_{1}(t)}, k_{t}=k, n_{k}(t-1) \geq \frac{t}{2}\right) \nonumber\\
\overset{(b)}{\leq} & 2K\slot \exp\left(\frac{-\slot \gap_{\text{min}}^2}{2 \numArms}\right) +  \Pr\left(\overline{E_{1}(t)}, k_{t}=k, n_{k}(t-1) \geq \frac{t}{2}\right)\nonumber\\
\overset{(c)}{\leq} & 2Kt^{-3} +  \underbrace{\Pr\left(k_{t} = k, E^{\mu}_k(t), E^{S}_k(t),n_{k}(t-1) \geq \frac{t}{2}\right)}_{\textbf{term A}} + \nonumber\\
 &\underbrace{\Pr\left(k_{t} = k, E^{\mu}_k(t), \overline{E^{S}_k(t)},n_{k}(t-1) \geq \frac{t}{2}\right)}_{\textbf{term B}}+ \nonumber \\
 &\underbrace{\Pr\left(k_{t} = k, \overline{E^{\mu}_k(t)},n_{k}(t-1) \geq \frac{t}{2}\right)}_{\textbf{term C}} 
\label{eq:cric-s} 
\end{align}
where $(b)$, comes from \Cref{lem:eliminatedOptimal}. Here, \eqref{eq:cric-s} follows from the fact that $\gap_{\text{min}} \geq 4\sqrt{\frac{2\numArms \log \slot_0}{\slot_0}}$ for some constant $\slot_0 > 0$. Now we treat each term in \eqref{eq:cric-s} individually. To bound term A, we note that $\Pr\left(k_{t} = k, E^{\mu}_k(t), E^{S}_k(t),n_{k}(t-1) \geq \frac{t}{2}\right) \leq \Pr\left(k_{t} = k, E^{\mu}_k(t), E^{S}_k(t)\right)$. From the analysis in \cite{agrawal2013further} (equation 6), we see that 
$\sum_{t = 1}^{T}\Pr\left(k_{t} = k, E^{\mu}_k(t), E^{S}_k(t)\right) = \OO(1)$ as it is shown through Lemma 2 in \cite{agrawal2013further} that, \\
$\sum_{t = 1}^{T}\Pr\left(k_{t} = k, E^{\mu}_k(t), E^{S}_k(t)\right) \leq \frac{216}{\Delta_k^2} + \sum_{j = 0}^{T} \Theta\left(e^{-\frac{\Delta_k^2j}{18}} + \frac{1}{e^{\frac{\Delta_k^2 j}{36}} - 1} + \frac{9}{(j + 1)\Delta_k^2}e^{-D_k j} \right)$. \\
Here, $D_k = L_k \log \frac{L_k}{\mu_{k^*}} + (1 - L_k) \log \frac{1 - L_k}{1 - \mu_{k^*}}$. 
Due to this, \\ $\sum_{t = 2t_0}^{T} \Pr\left(k_{t} = k, E^{\mu}_k(t), E^{S}_k(t),n_{k}(t-1) \geq \frac{t}{2}\right) = \OO(1)$.
\vspace{2mm}

\noindent
We now bound the sum of term B from $t = 1$ to $T$ by noting that \\
$\Pr\left(k_{t} = k, E^{\mu}_k(t), \overline{E^{S}_k(t)},n_{k}(t-1) \geq \frac{t}{2}\right) \leq \Pr\left(k_{t} = k, \overline{E^{S}_k(t)}\right) $. Additionally, from Lemma 3 in \cite{agrawal2013further}, we get that
$\sum_{t = 1}^{T} \Pr\left(k_{t} = k, \overline{E^{S}_k(t)}\right) \leq \frac{1}{d(U_k,\mu_k)} + 1$, where $d(x,y) = x \log \frac{x}{y} + (1 - x) \log \frac{1 - x}{1 - y}$. As a result, we see that 
$\sum_{t = 1}^{T} \Pr\left(k_{t} = k, E^{\mu}_k(t), \overline{E^{S}_k(t)},n_{k}(t-1) \geq \frac{t}{2}\right) = \OO(1)$.

\vspace{2mm}

\noindent
Finally, for the last term C we can show that,
\begin{align}
 (C) &= \Pr\left(k_{t} = k, \overline{E^{\mu}_k(t)},n_{k}(t-1) \geq \frac{t}{2}\right) \nonumber \\
 & \leq \Pr\left(\overline{E^{\mu}_k(t)},n_{k}(t-1) \geq \frac{t}{2}\right) \nonumber \\
 &= \Pr\left(\hat{\mu}_k - \mu_k > \frac{\Delta_k}{3}, n_k(t-1) \geq \frac{t}{2}\right) \nonumber \\
 &\leq t \exp \left(-2 \frac{t}{2} \frac{\Delta_k^2}{9} \right) \label{eq:unbd_and_hoeffding} \\
& \leq t^{-3} \nonumber
\end{align}
Here \Cref{eq:unbd_and_hoeffding} follows from hoeffding's inequality and the union bound trick to handle random variable $n_k(t-1)$. After plugging these results in \eqref{eq:cric-s}, we get that 

\begin{align}
    \sum_{t = 2t_0}^{T} \Pr\left(k_{t}=k, n_{k}(t-1) \geq \frac{t}{2}\right) &\leq \sum_{t = 2t_0}^{T} 2Kt^{-3} + \sum_{t = 2t_0}^{T} \Pr\left(k_{t} = k, E^{\mu}_k(t), E^{S}_k(t),n_{k}(t-1) \geq \frac{t}{2}\right) + \nonumber\\
 & \sum_{t = 2t_0}^{T} \Pr\left(k_{t} = k, E^{\mu}_k(t), \overline{E^{S}_k(t)},n_{k}(t-1) \geq \frac{t}{2}\right)+ \nonumber \\
 & \sum_{t = 2t_0}^{T} \Pr\left(k_{t} = k, \overline{E^{\mu}_k(t)},n_{k}(t-1) \geq \frac{t}{2}\right) \\
 &\leq \sum_{t = 2t_0}^{T} 2Kt^{-3} + \OO(1) + \OO(1) + \sum_{t = 2t_0}^{T} t^{-3} \\
 &= \OO(1)
\end{align}

\end{proof}

We now show that the expected number of pulls by C-TS for a non-competitive arm is bounded above by an $\OO(1)$ constant.\\
\noindent
\textbf{Expected number of pulls by C-TS for a non-competitive arm.}
We bound $\E{\pulls_\arm(\slot)}$ as 
\begin{align}
\E{\pulls_\arm(\totalPulls)} &= \E{\sum_{\slot = 1}^{\totalPulls}\indicator_{\{\arm_\slot = \arm\}}} \nonumber \\
&= \sum_{\slot = 0}^{\totalPulls-1} \Pr(\arm_{\slot+1} = \arm) \nonumber \\
&= \sum_{\slot = 1}^{2\slot_0} \Pr(\arm_\slot = \arm) + \sum_{\slot = 2 \slot_0}^{\totalPulls-1} \Pr(\arm_{\slot+1} = \arm) \nonumber \\
&\leq 2 \slot_0 +  \sum_{\slot = 2 \slot_0}^{\totalPulls-1} \Pr\left(\arm_{\slot+1} = \arm , \pulls_{\arm^*}(\slot) \geq \frac{t}{2}\right)  + \sum_{\slot = 2 \slot_0}^{\totalPulls-1} \Pr\left(\arm_{\slot+1} = \arm , \pulls_{\arm}(\slot) \geq \frac{t}{2}\right) \\
&\leq 2 \slot_0 + \sum_{\slot = 2 \slot_0}^{\totalPulls - 1} 3\slot^{-3} + \sum_{\slot = 2 \slot_0}^{\totalPulls-1} \Pr\left(\arm_{\slot+1} = \arm , \pulls_{\arm}(\slot) \geq \frac{t}{2}\right) \label{usingSomeLemma}\\
&= \OO(1) \label{eq:lastStepTS}
\end{align}
Here, \eqref{usingSomeLemma} follows from \Cref{lem:suboptimalNotCompetitive} and \eqref{eq:lastStepTS} follows from \Cref{lem:notPullsuboptimalifEnough} and the fact that the sum of $3t^{-3}$ is bounded and $\slot_0  = \inf \bigg\{\tau > 0: \Delta_{\text{min}},\epsilon_k \geq 4 \sqrt{\frac{ 2K \log \tau}{\tau}} \bigg\}.$

We now show that when the sub-optimal arm $k$ is competitive, the expected pulls of arm $k$ is $\OO(\log T)$.\\

\noindent
\textbf{Expected number of pulls by C-TS for a competitive arm $k \neq k^*$.}:
For any sub-optimal arm $\arm \neq \arm^*$,
\begin{align}
    \E{\pulls_\arm(\totalPulls)} &\leq \sum_{\slot = 1}^{\totalPulls} \Pr(\arm_\slot = \arm) \nonumber \\
    &= \sum_{\slot = 1}^{\totalPulls} \Pr((k_t = k, E_1(\slot)) \cup (E_1^c(\slot), k_t = k)) \label{stepE12} \\
    &\leq \sum_{\slot = 1}^{\totalPulls} \Pr(E_1(\slot)) + \sum_{\slot = 1}^{\totalPulls}  \Pr(E_1^c(\slot), k_t = k) \nonumber \\
    & \leq \sum_{\slot = 1}^{\totalPulls} \Pr(E_1(\slot)) + \sum_{\slot = 1}^{\totalPulls}  \Pr(E_1^c(\slot), k_{t} = k, S_k(t-1) > S_{k^*}(t-1)) \nonumber 
\end{align}
\begin{align}
    &\leq \sum_{\slot = 1}^{\totalPulls} \Pr(E_1(\slot)) + \sum_{\slot = 0}^{\totalPulls-1} \Pr(S_\arm(\slot)> S_{\arm^*}(\slot), k_{t+1} = k) \nonumber \\
    &= \sum_{\slot = 1}^{\totalPulls} 2\numArms\slot^{-3} + \sum_{\slot = 0}^{\totalPulls-1} \Pr\left(S_\arm(\slot) > S_{\arm^*}(\slot), k_{t+1} = k \right) \label{eliminatedArmProb} \\
    &\leq \frac{9\log(T)}{\Delta_k^2} + \OO(1) + \sum_{\slot = 1}^{\totalPulls} 2 \numArms \slot^{-3}. \label{fromAuer} \\
    &= \OO(\log T).
\end{align}
Here, \eqref{eliminatedArmProb} follows from \Cref{lem:eliminatedOptimal}. We have \eqref{fromAuer} from the analysis of Thompson Sampling for the classical bandit problem in \cite{agrawal2013further}. This arises as the term $\Pr\left(S_\arm(\slot) > S_{\arm^*}(\slot), k_{t+1} = k \right)$ counts the number of times $S_k(t) > S_{k^*}(t)$ and $k_{t+1} = k$. This is precisely the term analysed in Theorem 3 of \cite{agrawal2013further} to bound the expected pulls of sub-optimal arms by TS. 
In particular, \cite{agrawal2013further} analyzes the expected number of pull of sub-optimal arm (termed as $\E{k_i(T)}$ in their paper) by evaluating $\sum_{t = 0}^{T-1} \Pr(S_k(t) > S_{k^*}(t), k_{t+1} = k)$ and it is shown in their Section 2.1 (proof of Theorem 1 of \cite{agrawal2013further}) that $\sum_{t = 0}^{T-1} \Pr(S_k(t) > S_{k^*}(t), k_{t+1} = k) \leq \OO(1) + \frac{\log(T)}{d(x_i, y_i)}$. The term $x_i$ is equivalent to $U_k$ and $y_i$ is equal to $L_k$ in our notations. Moreover $d(U_k, L_k) \leq \frac{\Delta_k^2}{9}$, giving us the desired result of \eqref{fromAuer}.}

\section{Lower Bounds}
For the proof we define $R_\arm = Y_\arm(X)$ and $\tilde{R}_\arm = g_\arm(\tilde{X})$, where $f_X(x)$ is the probability density function of random variable $X$ and $f_{\tilde{X}}(x)$ is the probability density function of random variable $\tilde{X}$. Similarly, we define $f_{R_k}(r)$ to be the reward distribution of arm $k$.

\vspace{0.1cm}
\noindent
\textbf{Proof of Theorem 4}

Let arm $\arm$ be a $\textit{Competitive}$ sub-optimal arm, i.e $\optimistGap_{\arm,\arm^*} < 0$. To prove that regret is $\Omega(\log T)$ in this setting, we need to create a new bandit instance, in which reward distribution of optimal arm is unaffected, but a previously competitive sub-optimal arm $k$ becomes optimal in the new environment. We do so by constructing a bandit instance with latent randomness $\tilde{X}$ and random rewards $\tilde{Y}_k(X)$. Let's denote to $\tilde{Y}_k(\tilde{X})$ to be the random reward obtained on pulling arm $k$ given the realization of $\tilde{X}$. To make arm $k$ optimal in the new bandit instance, we construct $\tilde{Y}_k(X)$ and $\tilde{X}$ in the following manner. Let $\mathcal{Y}_k$ denote the support of $Y_k(X)$. 

Define $$\tilde{Y}_k(X) = 
\begin{cases}
\bar{g}_k(X) \quad \text{w.p. } 1-\epsilon_1 \\
\tilde{Y}_k(X) \sim \text{Uniform}(\mathcal{Y}_k) \quad \text{w.p. } \epsilon_1
\end{cases} 
$$This changes the conditional reward of arm $k$ in the new bandit instance (with increased mean). 

Furthermore, Define $$\tilde{X} = 
\begin{cases}
S(R_{k^*}) \quad w.p. 1 - \epsilon_2  \\
\text{Uniform} \sim \mathcal{X} \quad w.p. \epsilon_2.
\end{cases}, 
$$
with $S(R_{k^*}) = \arg \max_{\underline{g}_{k^*}(x) < R_{k^*} < \bar{g}_{k^*}(x)} \bar{g}_k(x)$.

\noindent
Here $R_{k^*}$ represents the random reward of arm $k^*$ in the original bandit instance.

This construction of $\tilde{X}$ is possible for some $\epsilon_1, \epsilon_2 > 0$, whenever arm $k$ is competitive by definition. Moreover, under such a construction one can change reward distribution of $\tilde{Y}_{k^*}(\tilde{X})$ such that reward $\tilde{R}_{k^*}$ has the same distribution as $R_{k^*}$. This is done by changing the conditional reward distribution, $f_{\tilde{Y}_{k^*} | X}(r) = \frac{f_{Y_{k^*} | X}(r) f_X(x)}{f_{\tilde{X}}(x)}$. 

Due to this, if an arm is competitive, there exists a new bandit instance with latent randomness $\tilde{X}$ and conditional rewards $\tilde{Y}_{k^*}|X$ and $\tilde{Y}_k | X$ such that $f_{R_{k^*}} = f_{\tilde{R}_k^*}$ and $\E{\tilde{R}_k} > \mu_{k^*}$, with $f_{R_k}$ denoting the probability distribution function of the reward from arm $k$ and $\tilde{R}_k$ representing the reward from arm $k$ in the new bandit instance.

Therefore, if these are the only two arms in our problem, then from \Cref{lem:LaiRobbins2Arms}, $$\lim_{\totalPulls \rightarrow \infty}\inf \frac{\E{\pulls_\arm(\totalPulls)}}{\log \totalPulls} \geq \frac{1}{D(f_{R_\arm}(r) || f_{\tilde{R}_{\arm}}(r))},$$
where $f_{\tilde{R}_k}(r)$ represents the reward distribution of arm $k$ in the new bandit instance.

Moreover, if we have more $\numArms - 1$ sub-optimal arms, instead of just 1, then $$\lim_{\totalPulls \rightarrow \infty}\inf \frac{\E{\sum_{\ell \neq \arm^*} \pulls_{\ell}(\totalPulls)}}{\log \totalPulls} \geq \frac{1}{D(f_{R_{\arm}}(r)|| f_{\tilde{R}_{\arm}}(r))}.$$

Consequently, since $\E{\regret(\totalPulls)} = \sum_{ell = 1}^{\numArms} \gap_\ell \E{\pulls_{\ell}(\totalPulls)}$, we have 
\begin{align}
\lim_{\totalPulls \rightarrow \infty}\inf \frac{\E{\regret (\totalPulls)}}{\log (\totalPulls)} \geq \max_{\arm \in \setofArmsC}\frac{\Delta_k}{D(f_{R_k} || f_{\tilde{R}_k})}.
\end{align}

\vspace{0.1cm}
\noindent
\textbf{A stronger lower bound valid for the general case}
A stronger lower bound for the general case can be shown by using the result in Proposition 1 of \cite{van2020optimal}. If $\mathcal{P}$ denotes the set of all possible joint probability distribution under which all pseudo-reward constraints are satisfied and $P$ denotes the underlying unknown joint probability distribution which has $k^*$ as the optimal arm. Then, the expected cumulative regret for any algorithm that achieves a sub-polynomial regret is lower bounded as 
    $$\lim \inf_{\totalPulls \rightarrow \infty} \frac{Reg(\totalPulls)}{\log \totalPulls} \geq L(P),$$
where $L(P)$ is the solution of the optimization problem: 
\begin{align}
    &\min_{\eta(\arm) \geq 0, \arm \in \mathcal{K}} \sum_{\arm \in \mathcal{K}} \eta(\arm)\left(\max_{\ell \in \mathcal{K}}\mu_\ell- \mu_\arm\right) \nonumber\\
    &\text{subject to} \sum_{\arm \in \mathcal{K}} \eta(\arm) D(P,Q,\arm) \geq 1, \quad \forall Q \in \mathcal{Q}, \label{optProblem}\\
    where \quad & \mathcal{Q} = \{Q \in \mathcal{P} : f_R(R_{k^*} | Q, k^*) = f_R(R_{k^*} | P, k^*) ~~ \text{and} ~~ \bestarm \neq \arg \max_{\arm \in \mathcal{K}} \mu_\arm(Q) \}. \nonumber
\end{align}
Here, $D(P,Q,\arm)$ is the KL-Divergence between reward distributions of arm $k$ under joint probability distributions $P$ and $Q$, i.e., $f_R(R_{\arm}|\theta,\arm)$ and $f_R(R_{\arm}|\lambda,\arm)$. The term $\mu_k(Q)$ represents the mean reward of arm $k$ under the joint probability distribution $Q$. 

To interpret the lower bound, one can think of $\mathcal{Q}$ as the set of all joint probability distributions, under which the reward distribution of arm $k^*$ remains the same, but some other arm $k' \neq k^*$ is optimal under the joint probability distribution. The optimization problem reflects the amount of samples needed to distinguish these two joint probability distributions. This result is based on the original result of \cite{graves1997asymptotically}, which has been used recently in \cite{combes2017minimal, van2020optimal} for studying other bandit problems.

\vspace{0.1cm}
\noindent
\textbf{Lower bound discussion in general framework}

\begin{table}[t]
\centering
\begin{tabular}{|l|l|l|l|l|}
\cline{1-2} \cline{4-5}
\textbf{r} & \textbf{$s_{2,1}(r)$} &  & \textbf{r} & \textbf{$s_{1,2}(r)$} \\ \cline{1-2} \cline{4-5} 
\textbf{0} & $\frac{2}{3}$                   &  & \textbf{0} & $\frac{3}{4}$                     \\ \cline{1-2} \cline{4-5} 
\textbf{1} & $\frac{6}{7}$                   &  & \textbf{1} & $\frac{2}{3}$                     \\ \cline{1-2} \cline{4-5} 
\end{tabular}
\\ \vspace{2mm}
\parbox{.45\linewidth}{
\centering
\begin{tabular}{|l|l|l|}
\hline
    \textbf{(a)}    & $R_2 = 0$ & $R_2 = 1$ \\ \hline
$R_1 = 0$ & 0.1       & 0.2       \\ \hline
$R_1 = 1$ & 0.3       & 0.4       \\ \hline
\end{tabular}
}
\hfill
\parbox{.45\linewidth}{
\centering
\begin{tabular}{|l|l|l|}
\hline
    \textbf{(b)}    & $R_2 = 0$ & $R_2 = 1$ \\ \hline
$R_1 = 0$ & a       & b       \\ \hline
$R_1 = 1$ & c       & d      \\ \hline
\end{tabular}
}
\caption{The top row shows the pseudo-rewards of arms 1 and 2, i.e., upper bounds on the conditional expected rewards (which are known to the player). The bottom row depicts two possible joint probability distribution (unknown to the player). Under distribution (a), Arm 1 is optimal and all pseudo-reward except $s_{2,1}(1)$ are tight.}
\label{tab:pseudoBinappendix}
\end{table}

Consider the example shown in \Cref{tab:pseudoBinappendix}, for the joint probability distribution $(a)$, Arm 1 is optimal. Moreover, all pseudo-rewards except $s_{2,1}(1)$ are tight, i.e.,$s_{\ell,k}(r) = \E{R_\ell | R_k = r}$. For the joint probability distribution shown in $(a)$, expected pseudo-reward of Arm 2 is $0.8$ and hence it is competitive. Due to this, our C-UCB and C-TS algorithms pull Arm 2 $\OO(\log T)$ times.

However, it is not possible to construct an alternate bandit environment with joint probability distribution shown in \Cref{tab:pseudoBinappendix}(b), such that Arm 2 becomes optimal while maintaining the same marginal distribution for Arm 1, and making sure that the pseudo-rewards still remain upper bound on conditional expected rewards. Formally, there does not exist $a,b,c,d$ such that $c + d = 0.7$, $\frac{c}{a+c} < 3/4$, $\frac{b}{a+b} < 2/3$, $\frac{d}{b+d} < 2/3$, $\frac{d}{d+c} < 6/7$ and $a+b+c+d = 1$. This suggests that there should be a way to achieve $\OO(1)$ regret in this scenario. We believe this can be done by using all the constraints (imposed by the knowledge of pair-wise pseudo-rewards to shrink the space of possible joint probability distributions) when calculating empirical pseudo-reward. However, this becomes tough to implement as the ratings can have multiple possible values and the number of arms is more than 2. We leave the task of coming up with a practically feasible and easy to implement algorithm that achieves bounded regret whenever possible in a general setup as an interesting open problem.

\end{document}